\documentclass{article}
\pdfoutput=1
\usepackage{etoolbox}

\newtoggle{arxivupload}
\togglefalse{arxivupload}
\newtoggle{nips}
\togglefalse{nips}
\newtoggle{alt}
\togglefalse{alt}

\newcommand{\alttogtrue}[1]{\iftoggle{alt}{#1}{}}
\newcommand{\alttogfalse}[1]{\iftoggle{alt}{}{#1}}

\newcommand{\nipstogfalse}[1]{\iftoggle{nips}{}{#1}}
\newcommand{\nipstogtrue}[1]{\iftoggle{nips}{#1}{}}

\usepackage{multirow,nicefrac}
\usepackage{makecell}
\usepackage{footnote}
\usepackage{longtable}
\usepackage{tablefootnote}
\usepackage[T1]{fontenc}
\usepackage{verbatim} 
\usepackage[utf8]{inputenc}

\nipstogfalse{
\usepackage{fullpage}
}

\usepackage[breaklinks=true]{hyperref}

\usepackage{amsthm}

\usepackage{amsfonts,amssymb,mathrsfs}
\usepackage{mathtools}
\DeclareMathSymbol{\shortminus}{\mathbin}{AMSa}{"39}

\iftoggle{nips}
{
	\usepackage[numbers]{natbib}

}
{
	\usepackage{natbib}
}

\usepackage{prettyref}

\alttogfalse{
\usepackage[vlined,boxed]{algorithm2e}
}

\usepackage[capitalise,nameinlink]{cleveref}
\Crefname{equation}{Eq.}{Eqs.}
\Crefname{assumption}{Assumption}{Assumptions}
\Crefname{condition}{Condition}{Conditions}

\usepackage{breakcites,bm}

\hypersetup{colorlinks,citecolor=blue,linkcolor = blue}
\usepackage{latexsym}
\usepackage{relsize}
\usepackage{thm-restate}
\usepackage{appendix}

\usepackage{xcolor}
\usepackage{dsfont}

\numberwithin{equation}{section}

\newcommand{\nipsminpt}{\iftoggle{nips}{\vspace{-.2em}}{}}

\newcommand{\dmax}{d_{\max}}

\newcommand{\calM}{\mathcal{M}}

\newcommand{\rank}{\mathrm{rank}}

\DeclareFontFamily{U}{mathx}{\hyphenchar\font45}
\DeclareFontShape{U}{mathx}{m}{n}{
      <5> <6> <7> <8> <9> <10>
      <10.95> <12> <14.4> <17.28> <20.74> <24.88>
      mathx10
      }{}
\DeclareSymbolFont{mathx}{U}{mathx}{m}{n}
\DeclareFontSubstitution{U}{mathx}{m}{n}
\DeclareMathAccent{\widecheck}{0}{mathx}{"71}
\DeclareMathAccent{\wideparen}{0}{mathx}{"75}

\newcommand{\ignore}[1]{}

\newcommand{\mscomment}[1]{\noindent{\textcolor{red}{\{{\bf MS:} \em #1\}}}}

\DeclareMathOperator{\BigOm}{\mathcal{O}}

\newcommand{\BigOh}[1]{\BigOm\left({#1}\right)}
\DeclareMathOperator{\BigOmtil}{\widetilde{\mathcal{O}}}
\newcommand{\BigOhTil}[1]{\BigOmtil\left({#1}\right)}

\newcommand{\matM}{\mathbf{M}}

\newcommand{\Gst}{G_{\star}}

\newcommand{\R}{\mathbb{R}}

\newcommand{\Mclass}{\mathcal{M}}

\newcommand{\calA}{\mathcal{A}}

\newcommand{\opnorm}[1]{\|#1\|_{\op}}

\newcommand{\fpred}{f^{\mathrm{pred}}}

\newcommand{\N}{\mathbb{N}}
\newcommand{\mats}{\mathbf{s}}
\newcommand{\closedloop}{\mathrm{cl}}

\newcommand{\transferclass}[2]{\mathcal{G}^{#1 \times #2}}

\newcommand{\natu}{\matu^{\mathrm{nat}}}

\newcommand{\yalg}{\maty^{\Alg}}
\newcommand{\ualg}{\matu^{\Alg}}

\newcommand{\matY}{\mathbf{Y}}

\newcommand{\Restu}{R_{\matu,\mathrm{est}}}

\newcommand{\Rubar}{\overline{R}_{\matu}}

\newcommand{\fhat}{\widehat{f}}
\newcommand{\Fhat}{\widehat{F}}

\newcommand{\Alg}{\mathsf{alg}}
\newcommand{\calN}{\mathcal{N}}

\newcommand{\Ghat}{\widehat{G}}

\newcommand{\truenaty}{\maty^{\mathrm{nat}}}
\newcommand{\matx}{\mathbf{x}}

\newcommand{\matu}{\mathbf{u}}
\newcommand{\matw}{\mathbf{w}}
\newcommand{\maty}{\mathbf{y}}

\newcommand{\iidsim}{\overset{\mathrm{i.i.d}}{\sim}}
\newcommand{\Regret}{\mathrm{Regret}}

\newcommand{\dimy}{d_{y}}
\newcommand{\dimx}{d_{x}}
\newcommand{\dimu}{d_{u}}

\newcommand{\Cpicl}{C_{\pi,\closedloop}}

\newcommand{\Apicl}{A_{\pi,\closedloop}}
\newcommand{\Bpicl}{B_{\pi,\closedloop}}

\newcommand{\Gpicl}{G_{\pi,\closedloop}}
\newcommand{\Dpicl}{D_{\pi,\closedloop}}

\newcommand{\dimpi}{d_{\pi}}
\newcommand{\Bst}{B_{\star}}
\newcommand{\Ast}{A_{\star}}
\newcommand{\Cst}{C_{\star}}
\newcommand{\Cpi}{C_{\pi}}
\newcommand{\Api}{A_{\pi}}
\newcommand{\Bpi}{B_{\pi}}
\newcommand{\Dpi}{D_{\pi}}

\newcommand{\matupi}{\matu^{\pi}}
\newcommand{\matypi}{\maty^{\pi}}
\newcommand{\matxpi}{\matx^{\pi}}

\newcommand{\naty}{\truenaty}

\newcommand{\psiGst}{\psi_{\Gst}}
\newcommand{\psicomp}{\psi_{\star}}
\newcommand{\Rcomp}{R_{\Pi}}

 \newcommand{\epsG}{\epsilon_G}

\newcommand{\mcomp}{m_{\star}}

\newcommand{\radM}{R_{\Mclass}}
\newcommand{\rady}{R_{\maty}}

\newcommand{\radnat}{R_{\mathrm{nat}}}

\newcommand{\I}{\mathbb{I}}
\newcommand{\mate}{\mathbf{e}}

\newcommand{\matz}{\mathbf{z}}

\newcommand{\op}{\mathrm{op}}
\newcommand{\loneop}{\mathrm{\ell_1,op}}

\newcommand{\fro}{\mathrm{F}}

\newcommand{\calD}{\mathcal{D}}

\newcommand{\algcomment}[1]{\texttt{\color{blue} #1}}

\alttogfalse{
	
	\theoremstyle{plain}
	\newtheorem{theorem}{Theorem}[section]

	\newtheorem{lemma}{Lemma}[section]
	\newtheorem{claim}[lemma]{Claim}
	\newtheorem{fact}[lemma]{Fact}
	\newtheorem{corollary}{Corollary}[section]
	\newtheorem{proposition}[lemma]{Proposition}

	\theoremstyle{definition}

	\newtheorem{definition}{Definition}[section]
	\newtheorem{example}{Example}[section]

	\newtheorem{remark}{Remark}[section]
    
}

  \alttogtrue{

	  \newtheorem{example}{Example}
	  \newtheorem{theorem}{Theorem}
	  \newtheorem{lemma}{Lemma}
	  \newtheorem{proposition}{Proposition}
	  \newtheorem{remark}{Remark}
	  \newtheorem{corollary}[section]{Corollary}
	  \newtheorem{definition}{Definition}

    \newtheorem{claim}{Claim}[section]
    \newtheorem{fact}{Fact}[section]
    \newcommand\qed{\BlackBox}

  }
  \newtheorem{assumption}{Assumption}
  \newtheorem{condition}{Condition}[section]

\makeatletter
\newcommand{\neutralize}[1]{\expandafter\let\csname c@#1\endcsname\count@}
\makeatother

\newenvironment{thmmod}[2]
  {\renewcommand{\thetheorem}{\ref*{#1}#2}%
   \neutralize{theorem}\phantomsection
   \begin{theorem}}
  {\end{theorem}}

     \newenvironment{defnmod}[2]
  {\renewcommand{\thedefinition}{\ref*{#1}#2}%
   \neutralize{definition}\phantomsection
   \begin{definition}}
  {\end{definition}}

   \newenvironment{asmmod}[2]
  {\renewcommand{\theassumption}{\ref*{#1}#2}%
   \neutralize{assumption}\phantomsection
   \begin{assumption}}
  {\end{assumption}}

\newtheorem*{theorem*}{Theorem}
\newtheorem*{lemma*}{Lemma}
\newtheorem*{corollary*}{Corollary}
\newtheorem*{proposition*}{Proposition}
\newtheorem*{claim*}{Claim}
\newtheorem*{fact*}{Fact}
\newtheorem*{observation*}{Observation}

\newtheorem*{definition*}{Definition}
\newtheorem*{remark*}{Remark}
\newtheorem*{example*}{Example}

\alttogfalse{
\newtheoremstyle{named}{}{}{\itshape}{}{\bfseries}{}{.5em}{\Cref{#3} {\normalfont (informal)} }
{}
\theoremstyle{named}
\newtheorem*{informaltheorem}{Theorem}
\theoremstyle{plain}
}

\newcommand{\nablatwo}{\nabla^{\,2}}

\newcommand{\zst}{z_{\star}}

\DeclareMathAlphabet{\mathbfsf}{\encodingdefault}{\sfdefault}{bx}{n}

\DeclareMathOperator*{\argmin}{arg\,min}

\newcommand{\ceil}[1]{\lceil #1 \rceil}
\newcommand{\floor}[1]{\lfloor #1 \rfloor}

\newcommand{\wt}[1]{\smash{\widetilde{#1}}}

\newcommand{\trace}{\mathrm{tr}}

\newcommand{\poly}{\mathrm{poly}}

\let\oldtfrac\tfrac
\renewcommand{\tfrac}[2]{\smash{\oldtfrac{#1}{#2}}}

\let\nablaold\nabla
\renewcommand{\nabla}{\nablaold\mkern-2.5mu}
\newcommand{\loss}{\ell}

\newcommand{\Exp}{\mathbb{E}}
\newcommand{\matdel}{\boldsymbol{\delta}}

\newcommand{\Torus}{\mathbb{T}}

\newcommand{\din}{d_{\mathrm{in}}}

\newcommand{\C}{\mathbb{C}}
\newcommand{\rmd}{\mathrm{d}}

\newcommand{\calC}{\mathcal{C}}

\newcommand{\Apinot}{A_{\pi_0}}
\newcommand{\Bpinot}{B_{\pi_0}}
\newcommand{\Cpinot}{C_{\pi_0}}
\newcommand{\Dpinot}{D_{\pi_0}}

\newcommand{\Gpipinot}{G_{\pinot \to \pi}}

\newcommand{\bigohst}[1]{\mathcal{O}(#1)}
\newcommand{\bigomst}[1]{\Omega(#1)}
\newcommand{\sftext}[1]{{\normalfont \textsf{#1}}}
\newcommand{\normtext}[1]{{\normalfont \text{#1}}}
\newcommand{\nicehalf}{\nicefrac{1}{2}}
\newcommand{\cC}{\mathcal{C}}

\newcommand{\mb}[1]{\mathbf{#1}}
\newcommand{\matzch}{\check{\matz}}
\newcommand{\cost}{\ell}
\newcommand{\clam}{c_{\lambda}}
\newcommand{\minone}{\shortminus 1}
\newcommand{\psipinot}{\psi_{\pinot}}

\newcommand{\radyc}{R_{Y,\calC}}
\newcommand{\psiG}{\psi_{G}}
\newcommand{\cv}{c_{v}}
\newcommand{\radH}{R_{H}}

\newcommand{\MemReg}{\mathrm{MemoryReg}}
\newcommand{\UnaReg}{\textsc{Oco}\mathrm{Reg}}

\newcommand{\MoveDiff}{\mathrm{MoveDiff}}
\newcommand{\EucMoveCost}{\mathrm{EucCost}}
\newcommand{\ocoam}{\textsc{OcoAM}}
\newcommand{\AdapMoveCost}{\mathrm{AdapCost}}

\newcommand{\constk}{c_{K}}
\newcommand{\rhok}{\rho_{K}}

\newcommand{\oco}{\textsc{Oco}}

\newcommand{\nustar}{\nu_{\star}}

\newcommand{\cpsi}[1][t]{c_{\psi;#1}}
\newcommand{\jmax}{j_{\max}}

\newcommand{\kapnot}{\kappa}
\newcommand{\munot}{\mu_0}

\newcommand{\Rcancel}{\mathrm{Reg}_{\mathrm{cancel}}}
\newcommand{\Regblock}{\mathrm{Reg}_{\mathrm{block}}}

 \newcommand{\Regmovyi}{\mathrm{Reg}_{Y,\mathrm{move},i}}
 \newcommand{\Regmovh}{\mathrm{Reg}_{H,\mathrm{move}}}

\newcommand{\Clow}{C_{\mathrm{low}}}
\newcommand{\Cmid}{C_{\mathrm{mid}}}
\newcommand{\Chigh}{C_{\mathrm{hi}}}

\newcommand{\Lfactor}{\mathfrak{L}}

\newcommand{\ogd}{\textsc{Ogd}}
\newcommand{\ons}{\textsc{Ons}}
\newcommand{\semions}{\text{\normalfont Semi-}\textsc{Ons}}

\newcommand{\Lamtil}{\wt{\Lambda}}

\newcommand{\matv}{\mb{v}}
\newcommand{\dimv}{\sftext{d}_{\mb{v}}}

 \newcommand{\Leff}{L_{\mathrm{eff}}}
\newcommand{\matztil}{\tilde{\matz}}
\newcommand{\radv}{R_{v}}

\newcommand{\radY}{R_{Y}}
\newcommand{\matX}{\mb{X}}
\newcommand{\diamz}{D}
\newcommand{\radG}{R_{G}}

\newcommand{\yk}{\maty^K}
\newcommand{\uk}{\matu^K}
\newcommand{\vk}{\matv^K}

\newcommand{\ykhat}{\widehat{\maty}^K}
\newcommand{\ukhat}{\widehat{\matu}^K}

\newcommand{\midykhat}[1]{\mid \ykhat_{1:#1}}

\newcommand{\salg}{\mats^{\Alg}}

\newcommand{\minhalf}{\shortminus\frac{1}{2}}

\newcommand{\ccond}{c_{\Lambda}}

\newcommand{\calI}{\mathcal{I}}

\newcommand{\Err}{{\textsc{Err}}}
	\newcommand{\Reghat}{\widehat{\mathrm{Reg}}}

	\newcommand{\Rbar}{\overline{R}}
\newcommand{\matYhat}{\widehat{\matY}}

\newcommand{\nabhat}{\hat{\nabla}}

\newcommand{\Lamhat}{\widehat{{\Lambda}}}

\newcommand{\err}{\mathbf{err}}
\newcommand{\matvhat}{\widehat{\matv}}
\newcommand{\matvst}{\matv^{\star}}

\newcommand{\uex}{\matu^{\mathrm{ex}}}

\newcommand{\GK}{G_{K}}
\newcommand{\Gk}{G_{K}}

\newcommand{\uexalg}{\matu^{\mathrm{ex},\Alg}}

\newcommand{\matH}{\mathbf{H}}
\newcommand{\matHhat}{\widehat{\matH}}

\newcommand{\matspi}{\mats^{\pi}}

\newcommand{\drc}{\textsc{Drc}}

\newcommand{\embed}{\mathfrak{e}}
\newcommand{\embedM}{\embed}
\newcommand{\embedy}{\embed_{y}}

\newcommand{\pinot}{\pi_0}
\newcommand{\Pildc}{\Pi_{\mathrm{ldc}}}

\newcommand{\Rpinot}{R_{\pinot}}

\newcommand{\psinot}{\psi_{\pi_0}}
\newcommand{\constcomp}{c_{\star}}
\newcommand{\rhocomp}{\rho_{\star}}
\newcommand{\radcomp}{R_{\star}}

\newcommand{\uring}{\mathring{\matu}}
\newcommand{\yring}{\mathring{\maty}}
\newcommand{\sring}{\mathring{\mats}}
\newcommand{\uopenpi}{\uring^{\pi}}
\newcommand{\yopenpi}{\yring^{\pi}}
\newcommand{\sopenpi}{\sring^{\pi}}
\newcommand{\Pistar}{\Pi_{\star}}

\newcommand{\Mdfc}{\calM_{\mathrm{drc}}}

\newcommand{\valg}{\matv^{\Alg}}
\newcommand{\etasymbol}{\omega}
\newcommand{\mateta}{\boldsymbol{\etasymbol}}
\newcommand{\etanat}{\mateta^{\mathrm{nat}}}
\newcommand{\etanathat}{\widehat{\mateta}^{\,\mathrm{nat}}}
\newcommand{\dimeta}{d_{\etasymbol}}
\newcommand{\embedeta}{\embed_{\etasymbol}}

\newcommand{\ynathat}{\widehat{\maty}^{\mathrm{nat}}}
\newcommand{\unathat}{\widehat{\matu}^{\mathrm{nat}}}
\newcommand{\vnathat}{\widehat{\matu}^{\mathrm{nat}}}
\newcommand{\midhateta}[1][t]{\mid \etanathat_{1:#1}}
\newcommand{\midnateta}[1][t]{\mid \etanat_{1:#1}}

\newcommand{\Bpinotu}{B_{\pinot,u}}
	
\newcommand{\etalg}{\mateta^{\Alg}}
\newcommand{\Gexyu}{G_{\mathrm{ex}\to v}}

\newcommand{\Ghatexyu}{\Ghat_{\mathrm{ex}\to (y,u)}}
\newcommand{\Ghatexeta}{\Ghat_{\mathrm{ex}\to \etasymbol}}
\newcommand{\Gexeta}{G_{\mathrm{ex} \to \etasymbol}}
\newcommand{\Cpinoteta}{C_{\pinot,\etasymbol}}
\newcommand{\Dpinoteta}{D_{\pinot,\etasymbol}}

\newcommand{\vnat}{\matv^{\mathrm{nat}}}
\newcommand{\ynat}{\maty^{\mathrm{nat}}}
\newcommand{\unat}{\matu^{\mathrm{nat}}}

\newcommand{\drcdyn}{\textsc{Drc-Dyn}}
\newcommand{\drconsdyn}{\textsc{Drc-Ons-Dyn}}

\title{Making Non-Stochastic Control (Almost) as Easy as Stochastic}
\author{{Max Simchowitz \thanks{UC Berkeley. \url{msimchow@berkeley.edu}}}}

\begin{document}
\maketitle

\begin{abstract}

Recent literature has made much progress in understanding  \emph{online LQR}\alttogfalse{: a modern learning-theoretic take on the classical control problem} where a learner attempts to optimally control an unknown linear dynamical system with fully observed state, perturbed by  i.i.d. Gaussian noise. \iftoggle{nips}{The}{It is now understood that the} optimal regret over time horizon $T$ against the optimal control law scales as $\widetilde{\Theta}(\sqrt{T})$. In this paper, we show that the same regret rate (against a suitable benchmark) is attainable even in the considerably more general non-stochastic control model, where the system is driven by \emph{arbitrary adversarial} noise \citep{agarwal2019online}. \nipstogfalse{In other words, \emph{stochasticity does not improve regret rates for online LQR}. }
\nipstogfalse{
	
}
 We attain the optimal $\widetilde{\mathcal{O}}(\sqrt{T})$ regret when the dynamics are unknown to the learner, and $\mathrm{poly}(\log T)$ regret when known, provided that the cost functions are strongly convex (as in LQR).  Our algorithm is based on a novel variant of online Newton step \citep{hazan2007logarithmic}, which adapts to the geometry induced by adversarial disturbances, and our analysis hinges on generic regret bounds for certain structured losses in the OCO-with-memory framework \citep{anava2015online}.
 \nipstogfalse{Moreover, our results accomodate the full generality of the non-stochastic control setting: adversarially chosen (possibly non-quadratic) costs, partial state observation,  and fully adversarial process and observation noise. }



\end{abstract}
\newcommand{\PolicyRegret}{\mathrm{PolicyRegret}}
\newcommand{\ControlReg}{\mathrm{ControlReg}}

\newcommand{\calO}{\mathcal{O}}
\newcommand{\ocom}{\textsc{OcoM}}
\newcommand{\convoco}{\textsf{ConvOCO}}
\newcommand{\convreg}{\mathrm{ConvReg}}
\newcommand{\drcons}{\textsc{Drc-Ons}}
\newcommand{\drcgd}{\textsc{Drc-Gd}}
\nipsminpt
\nipsminpt
\nipsminpt
\section{Introduction}
\nipsminpt
\nipsminpt
\nipsminpt
In control tasks, a learning agent seeks to minimize cumulative loss in a dynamic environment which responds to its actions. While dynamics make control problems immensely expressive, they also pose a significant challenge: the learner's past decisions affect future losses incurred. 

This paper focuses on the widely-studied setting of linear control, where the the learner's environment is described by a continuous state, and evolves according to a linear system of equations, perturbed by \emph{process noise}, and guided by inputs chosen by the learner. Many of the first learning-theoretic results for linear control focused on \emph{online LQR} \citep{abbasi2011regret,dean2018regret,cohen2019learning,mania2019certainty}, an online variant of the classical \emph{Linear Quadratic Regulator (LQR)} \citep{kalman1960new}. In online LQR, the agent aims to control an unknown linear dynamical system driven by independent, identically distributed Gaussian process noise.  Performance is measured by regret against the optimal LQR control law on a time horizon $T$, for which the optimal regret rate is $\widetilde{\Theta}(\sqrt{T})$ \citep{cohen2019learning,mania2019certainty,simchowitz2020naive,cassel2020logarithmic}.
\nipstogfalse{
	
}
Theoretical guarantees for LQR rely heavily on the strong stochastic modeling assumptions for the noise, and may be far-from-optimal if these assumptions break. A complementary line of work considers \emph{non-stochastic control},   replacing stochastic process noise with adversarial disturbances to the dynamics \citep{agarwal2019online,simchowitz2020improper}. Here, performance is measured by \emph{regret}: performance relative to the best (dynamic) linear control policy in hindsight, given full knowledge of the adversarial perturbations. 

Though many works have proposed efficient algorithms which attain sublinear regret for non-stochastic control, they either lag behind optimal guarantees for the stochastic LQR problem, or require partial stochasticity assumptions to ensure their regret. And while there is a host of literature demonstrating that, in many online learning problems without dynamics, the worst-case rates of regret for the adversarial and stochastic settings are the same ~\citep{auer2002nonstochastic,zinkevich2003online,hazan2007logarithmic}, whether this is true in control is far from clear. Past decisions affect future losses in control settings, and this may be fundamentally more challenging when perturbations are adversarial and unpredictable. 
\nipstogfalse{
	
}
Despite this challenge, we propose an efficient algorithm that matches the optimal $\sqrt{T}$ regret bound attainable  the stochastic LQR problem, but under arbitrary, non-stochastic disturbance sequences and arbitrary strongly convex costs. Thus, from the perspective of regret with respect to a benchmark of linear controllers, we show that the optimal rate for non-stochastic control matches the stochastic setting.



\paragraph{Our Setting}Generalizing LQR, we consider partially-observed linear dynamics :
\begin{align}
\matx_{t+1} = \Ast \matx_t + \Bst \matu_t + \matw_t, \quad 
 \maty_t = \Cst \matx_t + \mate_t \label{eq:LDS_system}
\end{align}
Here, the state $\matx_t$ and process noise $\matw_t$ lie in $\R^{\dimx}$, the observation $\maty_t$ and observation noise $\mate_t$ lie in $\R^{\dimy}$, and the input $\matu_t \in \R^{\dimu}$ is elected by the learner, and  $\Ast,\Bst,\Cst$ are matrices of appropriate dimensions. We call the $(\matw_t,\mate_t)$ the \emph{disturbances}, and let $(\matw,\mate)$ denote the entire disturbance sequence.  Unlike LQR, we assume that the disturbances are selected by an oblivious\footnote{The oblivious assumption is only necessary if the dynamics are unknown to the learner; if the dynamics are known, our guarantees hold against adaptive adversaries as well.} adversary, rather than from a mean zero stochastic process, and the learner observes the outputs $\maty_t$, but not the full state $\matx_t$.  \Cref{ssec:app_classical} describes how our setting strictly generalizes the online LQR problem, and relates to its partially observed analogoue LQG.  
\nipstogfalse{
	
}
A \emph{policy} $\pi$ is a (possibly randomized) sequence of mappings $\matu_t := \pi_t(\maty_{1:t},\matu_{1:t-1})$. We denote by $\maty_t^{\pi}$ and $\matu_t^\pi$ sequence the realized sequence of outputs and inputs produced by policy $\pi$ and the noise sequence $(\matw,\mate)$. At each time $t$,  a convex cost $\cost_t : \R^{\dimy \times \dimu} \to \R$ is revealed, and the learner observes the current $\maty_t$, and  suffers loss $\cost_t(\maty_t,\matu_t)$. The \emph{cost functional} of a policy $\pi$ is 
%
\begin{align*}
	J_T(\pi) := \nipstogtrue{\textstyle}\sum_{t=1}^T \cost_t(\maty^{\pi}_t,\maty^{\pi}_t),
\end{align*}
measuring the cumulative losses evaluated on the outputs and inputs induced by the realization of the disturbances $(\matw,\mate)$. The learner's policy $\Alg$, is chosen to attain low \emph{control regret} with respect to a pre-specified benchmark class $\Pi$ of reference policies,
\begin{align}
\ControlReg_T(\Alg;\Pi) := J_T(\Alg) - \inf_{\pi \in \Pi} J_T(\pi), \label{eq:reg_def}
\end{align}
which measures the performance of $\Alg$ (on the realized losses/disturbances) compared to the best policy $\pi \in \Pi$ in hindsight (chosen with knowledge of losses and disturbances). We consider a restricted a benchmark class $\Pi$ consisting of linear, dynamic controllers, formalized in \Cref{defn:benchmark}. While this class encompasses optimal control laws for many classical settings \citep{simchowitz2020improper}, in general it \emph{does not} include the optimal control law for a given realization of noise. This \nipstogfalse{restriction} is unavoidable: even in \iftoggle{nips}{the simplest settings,}{simple, noiseless settings with adversarial costs,} it is impossible to attain sublinear regret with respect to the optimal control law \citep{li2019online}.
\nipstogfalse{
	
}
We assume that the losses $\ell_t(\cdot)$ are $\alpha$-strongly convex, and grow at most quadratically:
\begin{assumption}\label{asm:loss_reg} We suppose that all $\cost_t:\R^{\dimy + \dimu} \to \R$ are \emph{$L$-subquadratic:} $0 \le \cost(v) \le L\max\{1,\|v\|_2^2\}$, and $\|\nabla \cost(v)\|_2 \le L\max\{1,\|v\|\}$.  We also assume that $\cost_t$ are twice-continuously differentiable, and $\alpha$-strong convex ($\nablatwo \cost_t \succeq \alpha I$).  For simplicity, we assume $L \ge \max\{1,\alpha\}$. 
\end{assumption}
This assumption is motivated by classical LQR/LQG, where the loss is a strongly convex quadratic of the form $\ell(y,u) = y^\top R y + u^\top Q u$ for $R,Q \succ 0$. The central technical challenge of this work is that, unlike  standard online learning settings, \emph{the strong convexity of the losses does not directly yield fast rates}~\citep{agarwal2019logarithmic,foster2020logarithmic}. 

\subsection{Our Contributions}
For the above setting, we propose \emph{Disturbance Reponse Control via Online Newton Step}, or \drcons{} - an adaptive control policy which attains fast rates previously only known for settings with stochastic or semi-stochastic noise \citep{mania2019certainty,simchowitz2020improper,cohen2019learning,agarwal2019logarithmic}. Our algorithm combines the \drc{} controller parametrization \citep{simchowitz2020improper} with \semions{}, a novel second-order online learning algorithm tailored to our setting.  We show that \drcons{} achieves logarithmic regret when the learner knows the dynamical matrices:

\begin{informaltheorem}[thm:known_control]
When the agent knows the dynamics \eqref{eq:LDS_system} (but does not have foreknowledge of disturbances nor the costs $\ell_t$),  \drcons{} has $\ControlReg_T = \bigohst{\frac{L^2}{\alpha} \cdot \poly(\log T )}$.
\end{informaltheorem}

This is the first bound to guarantee logarithmic regret with general strongly convex losses and non-stochastic noise. Past work required stochastic or semi-stochastic noise \citep{agarwal2019logarithmic,simchowitz2020improper}, or was limited to fixed quadratic costs \citep{foster2020logarithmic}. For unknown dynamics, we find:

\begin{informaltheorem}[thm:unknown_control]
When the dyamics are unknown, \drcons{} with an initial estimation phase attains $\ControlReg_T = \widetilde{\mathcal{O}}(\frac{L^2}{\alpha}\sqrt{T})$.
\end{informaltheorem}

This bound matches the optimal $\sqrt{T}$-scaling for \emph{stochastic} online LQR \citep{simchowitz2020naive}. Thus, from the perspective of regret minimization with respect to the benchmark $\Pi$, non-stochastic control is \emph{almost} as easy as stochastic. This is not without many caveats, which are left to the discussion in \Cref{app:comparison_discussion}.

\paragraph{Technical Contributions} While our main results are control theoretic, our major technical insights pertain to online convex optimization (\oco). Our control algorithm  leverages a known reduction \citep{agarwal2019online} to the online convex optimization with memory (\ocom) framework \citep{anava2015online}, which modifies \oco{} by allowing losses to depend on past iterates. Past \ocom{} analyses required bounds on both the standard \oco{} regret and  total Euclidean variation of the iterates produced (\Cref{sec:semions_sketch}). 
\nipstogfalse{
	
}
But for the the losses that arise in our setting, \Cref{thm:Regmu_lb} shows that there is a significant tradeoff between the two, obviating sharp upper bounds. 
\nipstogfalse{
Specifically, with non-stochastic noise, the loss functions induced by this reduction are no longer strongly convex. They do satisfy a weaker property called exp-concavity, an assumption which suffices for fast rates in standard \oco{}. However, we show that this condition is not sufficient to naively extend the standard \ocom{} approach:
	\begin{informaltheorem}[thm:Regmu_lb] For exp-conave (but not necessarily strongly convex) losses, \ocom{} analyses based on  Euclidean movement cannot guarantee better than $\Omega(T^{1/3})$ regret, which translates into $\Omega(T^{1/3})$ regret for the setting of controlling an \emph{known} system. This lower bound is matched by the Online Newton Step algorithm \citep{hazan2007logarithmic} with an appropriately selected regularization parameter.  
\end{informaltheorem}
In addition, \Cref{thm:Regmu_lb} characterizes the tradeoff between \oco-regret and Euclidean movement cost for exp-concave losses; for example, it shows any algorithm which attain logarithmic regret in the (without memory) \oco-setting suffers Euclidean movement cost $\tilde{\Omega}(\sqrt{T})$.

}
To overcome this \nipstogfalse{ lower bound}, we show that online control enjoys additional structure  we call \emph{\oco{} with affine memory}, or \ocoam{}. 
We propose a novel second order method, \semions{},  based on online Newton step (\ons,~\cite{hazan2007logarithmic}), tailored to this structure.
Under a key technical condition satisfied by online control, we establish logarithmic regret.

\begin{informaltheorem}[thm:semions_memory]
 Under the aforementioned assumption (\Cref{defn:input_recov}), the \semions{} algorithm attains $\BigOh{\frac{1}{\alpha}\log T}$ regret in the \ocoam{} setting.
\end{informaltheorem}

The above bound directly translates to logarithmic control regret for known systems, via the control-to-\ocoam{} reduction spelled out in \Cref{sec:prelim}. For control of unknown systems, the undergirding \ocoam{} bound is quadratic sensitivity to $\epsilon$-approximate losses: 
\begin{informaltheorem}[thm:semions_unknown]
 Consider the \ocoam{} setting with $\epsilon$-approximate losses (in the sense of \Cref{asm:approx_ocoam}). Then, \nipstogfalse{under appropriate assumptions,}  \semions{} has regret $\BigOh{\frac{1}{\alpha}\log T \cdot T \epsilon^2}$. 
 \end{informaltheorem}

Quadratic sensitivity to errors in the gradients  was previously demonstrated for strongly convex stochastic optimization \citep{devolder2014first},  and subsequently for strongly convex \oco{} \citep{simchowitz2020improper}. Extending this guarantee to \semions{}  is the most intricate technical undertaking of this paper. 
\nipstogfalse{The full full formulation of \ocoam{}, description of the \semions{} algorithm, and rigorous statements of the above two bounds are given in \Cref{sec:ocoam_setting,sec:semions_alg,sec:semions_bounds}, respectively. \Cref{sec:semions_sketch} motivates  the novelty of our proof strategy and algorithm, and \Cref{sec:lb} formalizes the limitations of past approaches via a regret-movem{}ent tradeoff (\Cref{thm:Regmu_lb}).} 



\subsection{Prior Work\label{ssec:prior_work}}
In the interest of brevity, we restrict our attention to previous works regarding online control with a regret benchmark; for a survey of the decades old field of adaptive control, see e.g. \cite{stengel1994optimal}.
Much work has focused on obtaining low regret in online LQR with unknown dynamics \citep{abbasi2011regret,dean2018regret,mania2019certainty,cohen2019learning}, a setting we formally detail in \Cref{ssec:app_lqr}.  Recent algorithms \citep{mania2019certainty,cohen2019learning} attain $\sqrt{T}$ regret for this setting, with polynomial runtime and polynomial regret dependence on relevant problem parameters.  This was recently demonstrated to be optimal \citep{simchowitz2020naive,cassel2020logarithmic}, with \citet{cassel2020logarithmic}  showing that logarithmic regret is possible the partial system knowledge. In the related LQG setting (partial-observation, stochastic process and observation noise, \Cref{ssec:app_lqg}), \citet{mania2019certainty} present perturbation bounds which suggest $T^{2/3}$ regret, improve to $\sqrt{T}$ by \citet{lale2020regret}, matching the optimal rate for LQG. For LQG with both non-denegerate process and observation noise,  \citet{lale2020logarithmic} attain $\poly(\log T)$ regret, demonstrating that in the presence of observation, LQG is in fact \emph{easier} than LQR (with no observation noise) in terms of regret; see \Cref{app:comparison_discussion} for further discussion.

Recent work first departed from online LQR by considered adversarially chosen costs under known stochastic or noiseless dynamics \citep{abbasi2014tracking,cohen2018online}. \citet{agarwal2019logarithmic} obtain logarithmic regret for fully observed systems, stochastic noise and adversarially chosen, strongly convex costs.  The non-stochastic control setting we consider in this paper was established in 
\citet{agarwal2019online}, who obtain $\sqrt{T}$-regret for convex, Lipschitz (not strongly convex) cost functions and known dynamics. \citet{hazan2019nonstochastic} attains $T^{2/3}$ regret for the same setting with unknown dynamics. \citet{simchowitz2020improper} generalizes both guarantees to partial observation, and generalize the optimal rate of logarithmic and $\sqrt{T}$ for known and unknown systems, respectively to  strongly convex losses and a `semi-stochastic'' noise model. This assumption requires the noise to have a well-conditioned, stochastic component; in contrast, our methods allow \emph{truly adversarial} noise sequences. Lastly, for the known system setting, \citet{foster2020logarithmic} propose a different paradigm \nipstogfalse{- online learning with advantages - }which yields logarithmic regret with truly adversarial noise, but fixed quadratic cost functions and with full observation. In contrast, our algorithm accomodates both partial observation and arbitrary, changing costs, and its analysis and presentation are considerably simpler. Our work also pertains to the broader literature of online optimization with policy regret and loss functions with memory \citep{arora2012online,anava2015online}, and our lower bound (\Cref{thm:Regmu_lb}) draws on the learning-with-switching-costs literature \citep{altschuler2018online,chen2019minimax,dekel2014bandits}.

\subsection{Organization and Notation}
\Cref{sec:ocoam} formulates the general \ocoam{} setting, describes our \semions{} algorithm, and states its guarantees (\Cref{thm:semions_memory,thm:semions_unknown}), and the regret-movement tradeoff that hindered past approaches (\Cref{thm:Regmu_lb}). \Cref{sec:prelim} turns to the control setting, describing the reduction to \ocoam{}, the \drcons{} algorithm, and stating our main results (\Cref{thm:known_control,thm:unknown_control}). 
\iftoggle{nips}{
Discussion of our results is deferred to \Cref{app:comparison_discussion}. All proofs are deferred to our appendix, whose  organization of the appendix is detailed in \Cref{app:organization}. 
}
{\Cref{sec:known} proves \Cref{thm:semions_memory}, and \Cref{sec:unknown} proves \Cref{thm:semions_unknown} and an important refinements (\Cref{thm:semions_unknown_clam}) required  for the control setting. Finally, \Cref{app:comparison_discussion} discusses the consequence of our results.
The organization of the appendix is detailed in \Cref{app:organization}. 
} 
\nipstogfalse{
	
}
Throughout, let $a \lesssim b$ denote that $a \le C b$, where $C$ is a universal constant independent of problem parameters. We use  $\bigomst{\cdot},\bigohst{\cdot}$ as informal asymptotic notation. 
	 We let  $a \vee b$ denote $\max\{a,b\}$, and $a \wedge b$ to denote $\min\{a,b\}$. For vectors $x$ and $\Lambda \succeq 0$,  we denote $\|x\|_{\Lambda}:= \sqrt{x^\top \Lambda x}$, and use $\|x\|$ and $\|x\|_2$ interchangeably for Euclidean norm. We let $\|A\|_{\op}$ denote the operator norm, and given a sequence of matrices $G = (G^{[i]})_{i \ge 0}$, we define $\|G\|_{\loneop} := \sum_{i \ge 0}\|G^{[i]}\|_{\op}$. We use $[(\cdot);(\cdot)]$  to denote vertical concatenation of vectors and matrices. Finally, non-bold arguments (e.g. $z$) denote function arguments, and bold (e.g. $\matz_t$) denote online iterates.


\newcommand{\nablahat}{\widehat{\nabla}}
\section{Fast Rates for OCO with Affine Memory \label{sec:ocoam}}

Building on past work \citep{simchowitz2020improper,agarwal2019online}, our results for control proceed via a reduction to {o}nline convex optimization (\oco) with memory, proposed by \citet{anava2015online}, and denoted by \ocom{} in this work. Our lower bound in \Cref{sec:lb} explains why this past strategy is insufficient. Thus, we consider a structured special case, \ocoam{} \nipstogfalse{(\oco{} with affine memory)}, which arises in control, present a second-order algorithm for this setting, \semions{}, and state its main guarantees. 
\nipsminpt
\nipsminpt
\nipsminpt
\paragraph{\ocom{} preliminaries} Let $\calC \subset \R^d$ be a convex constraint set. \ocom{} is an online learning game where, at each time $t$, the learner plays an input $\matz_t \in \calC$, nature reveals an $h+1$-argument loss $F_t:\calC^{h+1} \to \R$, and the learner suffers loss $F_t(\matz_{t},\matz_{t-1},\dots,\matz_{t-h})$, abbreviated as $F_t(\matz_{t:t-h})$. For each $F_t$, we define its \emph{unary specialization} $f_t(z) := F_t(z,\dots,z)$. 
The learner's performance is measured by what we term \emph{memory-regret}:\footnote{Throughout, the initial iterates $(\matz_{s})_{s \le 0}$ are arbitrary elements of $\calC$. We note that \citet{anava2015online} referred to $\MemReg_T$ as ``policy regret'', but this differs slightly from the policy regret proposed by \citet{arora2012online}. To avoid confusion, we use ``memory regret''.} 
  \begin{align}
\MemReg_T := \nipstogtrue{\textstyle}\sum_{t=1}^T F_t(\matz_{t:t-h}) - \inf_{z \in \calC} \nipstogtrue{\textstyle}\sum_{t=1}^T f_t(z).\label{eq:mem_reg_def}
\end{align}
Because the learner's loss is evaluated on a history of past actions, \ocom{} encodes learning problems with dynamics, such as our control setting. This is in contrast to the standard \oco{} setting, which measures regret evaluated on the unary $f_t$:
\iftoggle{nips}{
  $\UnaReg_T := $ $\sum_{t=1}^T f_t(\matz_t) - \inf_{z \in \calC}\sum_{t=1}^T f_t(z)$.
}
{
  \begin{align}
\UnaReg_T := \sum_{t=1}^T f_t(\matz_t) - \inf_{z \in \calC}\sum_{t=1}^T f_t(z). \label{eq:std_reg_def}
\end{align} 
}
Our goal is to attain logarithmic memory-regret, and quadratic sensitivity to structured errors (in a sense formalized below). 



\subsection{OCO with Affine Memory \label{sec:ocoam_setting}}
While we desire logarithmic memory regret, \Cref{thm:Regmu_lb} shows  that  existing analyses  cannot yield better rates than $\Omega(T^{1/3})$. Luckily, the control setting gives us more structure. Let us sketch this with a toy setting, and defer the full reduction to \Cref{sec:prelim}. Consider a nilpotent, fully observed system: $\maty_t \equiv \matx_t$, and $\Ast^h = 0$.  Defining $G^{[i]} := [ \Ast^{i-1}\Bst; I\cdot\I_{i = 0}]$, the linear dynamics give $[\matx_t;\matu_t] := \sum_{i=0}^{h} G^{[i]}\matu_{t-i} + [\matx_{t,0};0]$, where $\matx_{t,0} = \sum_{i=0}^{h}\Ast^{i}\matw_{t-i}$ \nipstogfalse{is the noise-response}. For simple policies parametrized by $\matu_t^z = z \cdot \matw_t, z \in \R$, the loss incured under iterates $z_{t:t-h}$, $\ell_t([\matx_{t,0};0] + \sum_{i=0}^{h} G^{[i]}\matw_{t-i} z_{t-i}) =: F_t(z_{t:t-h})$, exhibits \emph{affine} dependence on the past.
\nipstogfalse{
  
}
Generalizing the above, the \oco{} with affine memory (\ocoam{}) setting is as follows. Fix $G = (G^{[i]})_{i \ge 0} \in (\R^{p \times \din})^\N$  across rounds. At each $t \ge 1$, the learner selects $\matz_t \in \calC \subset \R^d$, and the adversary reveals a convex cost $\ell_t: \R^p \to \R$, an offset vector $\matv_t \in \R^p$, and a matrix $\matY_t \in \R^{\din \times d}$.
 The learner suffers loss with-memory loss $F_t(\matz_{t:t-h})$, given by \iftoggle{nips}
{$F_t(z_{t:t-h}) := \loss_t(\matv_t + \sum_{i=0}^h G^{[i]}\,\matY_{t-i}z_{t-i} ).$}
{
\begin{align*}
F_t(z_{t:t-h}) := \loss_t(\matv_t +\sum_{i=0}^h G^{[i]}\,\matY_{t-i}z_{t-i} ).
\end{align*}
}
The induced unary losses are 
\begin{align}
f_t(z) := \loss_t(\matv_t + \matH_t z ), \quad \text{where } \matH_t := \nipstogtrue{\textstyle}\sum_{i=0}^h  G^{[i]}\,\matY_{t-i}. \label{eq:unary_known}
\end{align}
We consider two settings for \ocoam. In the \emph{exact} setting,  $G$ is known to the learner, and $\ell_t,\matv_t,\matY_t$ are revealed at each $t$. Thus $f_t$ and $\matH_t$ can be computed after each round. The \emph{approximate} setting, the learner knows only an approximation $\Ghat$ of $G$, and recieves an estimate $\matvhat_t$ of $\matv_t$ ($\matY_t$ and $\ell_t$ remain exact). Our algorithm uses approximate unary losses:
\begin{align}
\fhat_t(z) := \loss_t(\matvhat_t + \matHhat_t z ), \quad \text{where } \matHhat_t := \nipstogtrue{\textstyle}\sum_{i=0}^h  \Ghat^{[i]}\,\matY_{t-i}. \label{eq:unary_unknown}
\end{align}
We desire low sensitivity to the approximation errors of $\Ghat$ and $\matvhat$, translating to low estimation error sensitivity for control of an unknown system. 
\nipstogfalse{Like exact setting, the approximate \ocoam{} learner can construct $\fhat_t$ and $\matHhat_t$ at the end of round $t$.} For  both exact and approximate losses, memory regret is evaluated on the \emph{exact} losses $F_t,f_t$, consistent with  \ocom{}. 

\newcommand{\altquad}{\iftoggle{nips}{~}{\quad}}
\newcommand{\alttextstyle}{\nipstogtrue{\textstyle}}

\subsection{The \semions{} algorithm \label{sec:semions_alg}}
The standard algorithmic template for \ocom{} is to run an online optimization procedure on the unary losses $f_t$, otherwise disregarding $F_t$ (but accounting for the discrepancy between the two in the analysis) \citep{anava2015online}. We take this approach here, but with a tailored second order method. Let $\matz_{t-h+1},\dots,\matz_0 \in \calC$ be arbitrary initial parameters. For  step size and regularization parameters $\eta > 0$ and $\lambda > 0$, and setting $\nabla_t := \nabla f_t(\matz_t)$, the \semions{}(\Cref{alg:semions})  iterates are:
 \begin{align}
 \matztil_{t+1} \gets \matz_{t} - \eta \Lambda_{t}^{-1} \nabla_t, \altquad \matz_{t+1} \gets \argmin_{z \in \calC}\|\Lambda^{1/2}(\matztil_{t+1} - z)\|, \altquad 
\Lambda_t := \lambda I + \nipstogtrue{\textstyle}\sum_{s=1}^t \matH^\top \matH_t, \label{eq:known_semions}
\end{align}
The updates are nearly identical to online Newton step (\ons) \citep{hazan2007logarithmic}, but whereas the \ons{} uses preconditioner $\Lambda_{t,\ons} := \lambda I + \sum_{s=1}^t \nabla f_t(\matz_t)\nabla f_t(\matz_t)^\top$, \semions{} uses outer products of $\matH_t$. This decision is explained in the \hyperlink{rationaletarget}{paragraph} concluding \Cref{sec:semions_sketch}.  In the approximate setting \semions{} proceeds using the following approximations, with $\nablahat_t := \nabla \fhat_t(\matz_{t})$
\begin{align}
 \matztil_{t+1} \gets \matz_{t} - \eta \Lamhat_{t}^{-1} \nablahat_t , \altquad \matz_{t+1} \gets \argmin_{z \in \calC}\|\Lamhat^{1/2}(\matztil_{t+1} - z)\|, \altquad 
\Lamhat_t := \lambda I + \nipstogtrue{\textstyle}\sum_{s=1}^t \matHhat^\top \matHhat_t, \label{eq:unknown_semions}
\end{align}
defined using the quantities in \Cref{eq:unary_unknown}. In other words, approximate \semions{} is equivalent to exact \semions, treating $(\fhat_t,\matHhat_t)$ like the true $(f_t,\matH_t)$.


\begin{algorithm}[h]
\SetAlgoNoEnd
    \textbf{parameters}: Learning rate
    $\eta>0$, regularization parameter $\lambda>0$, convex domain $\calC \subset \R^d$. \\
    \textbf{initialize}:   $\Lambda_0=\lambda\cdot{}I_{d}$,  $\matz_1 \gets \mathbf{0}_{d}$\\
    \For{$t=1,2,\dots$:} 
    {
    \textbf{recieve} triple $(\ell_t,\matv_t, \matH_t)$. \algcomment{\% For approximate setting, replace $(\matv_t,\matH_t) \gets (\matvhat_t,\matHhat_t)$}\\
    $\nabla_{t}\gets \nabla f_t(\matz_t)$, \text{ where } $f_t(z) = \ell_t(\matv_t + \matH_t z).$\\
     $\Lambda_{t} \gets{}\Lambda_{t - 1} +  \matH_t^\top\matH_t$ \label{line:Lambda_def}.\\
    $\widetilde{\matz}_{t+1} \gets{} \matz_{t} - \eta \Lambda_{t}^{-1}\nabla_{t}$. \label{line:semions_step}\\
    $\matz_{t+1}\gets\argmin_{z\in\calC}\|\Lambda_{t}^{\nicehalf}(z-\widetilde{\matz}_{t+1})\|^{2}_{2}$.
    }
  \caption{Online Semi-Newton Step - $\semions(\lambda,\eta,\calC)$}
  \label{alg:semions}
\end{algorithm}

\subsection{Guarantees for \semions{} \label{sec:semions_bounds}}
To state our guarantees, we assume the $\alpha$-strong convexity and $L$-subquadratic assumption of \Cref{asm:loss_reg}. We assume various upper bounds on relevant quantities:
\begin{restatable}[Bounds on Relevant Parameters]{definition}{defnPolPars}\label{defn:pol_reg_pars} We assume $\calC$ contains the origin. Further, we define 
\iftoggle{nips}
{
  the diameter $\diamz := \max\{\|z-z'\|:z,z' \in \calC\}$, $Y$-radius $\radY := \max_{t } \|\matY_t\|_{\op}$, and  $\radyc := \max_{t}\max_{z \in \calC}\|\matY_t z\|$; In the \emph{exact} setting, we define the radii $\radv := \max_{t } \max\{ \|\matv_t\|_2\}$ and  $\radG := \max\{1,\|G\|_{\loneop}\}$. In the \emph{approximate} setting, $\radv := \max_{t} \max\{ \|\matv_t\|_2,\|\matvhat_t\|_2\}$, $\radG := \max\{1,\|G\|_{\loneop},\|\Ghat\|_{\loneop}\}$; For settings, we define the $H$-radius $\radH = \radG \radY$, and define the  effective Lipschitz constant  $\Leff := L\max\{1,\radv + \radG \radyc\}$. 
}
{
\begin{itemize}
\item The diameter $\diamz := \max\{\|z-z'\|:z,z' \in \calC\}$, $Y$-radius $\radY := \max_{t \in [T]} \|\matY_t\|_{\op}$, and  $\radyc := \max_{t}\max_{z \in \calC}\|\matY_t z\|$. 
\item 
In the \emph{exact} setting, we define the radii $\radv := \max_{t \in [T]} \max\{ \|\matv_t\|_2\}$ and  $\radG := \max\{1,\|G\|_{\loneop}\}$. In the \emph{approximate} setting, $\radv := \max_{t \in [T]} \max\{ \|\matv_t\|_2,\|\matvhat_t\|_2\}$, $\radG := \max\{1,\|G\|_{\loneop},\|\Ghat\|_{\loneop}\}$. 
\item For both exact and approximate settings, we define the $H$-radius $\radH = \radG \radY$, and define the  effective Lipschitz constant  $\Leff := L\max\{1,\radv + \radG \radyc\}$. 
\end{itemize}
}
\end{restatable}
Lastly, our analysis requires that the smallest singular value of $G$, viewed as  linear operator acting by convolution with sequences $(u_{1},u_{2},\dots) \in (\R^{\din})^\N$,  is bounded below: 
\begin{definition}\label{defn:input_recov} We define the \emph{convolution invertibility-modulus} as
\iftoggle{nips}{
$\kappa(G) := 1 \wedge \inf_{(u_{0},u_{1},\dots)} \\\{\sum_{n \ge 0}\|\sum_{i=0}^n G^{[i]}u_{n-i}\|_2^2 : \sum_{t}\|u_t\|_2^2 = 1 \}$, and the 
}
{
	\begin{align*}
\kappa(G) := 1 \wedge \inf_{(u_{0},u_{1},\dots)} \left\{\sum_{n \ge 0}\left\|\sum_{i=0}^n G^{[i]}u_{n-i}\right\|_2^2 : \sum_{t\ge 0}\|u_t\|_2^2 = 1 \right\}.
\end{align*}
We also define its
}
 \emph{decay-function} $\psiG(n) := \sum_{i \ge n} \|G^{[i]}\|_{\op}$.
\end{definition}
A Fourier-analytic argument (\Cref{lem:kap_bound_K}) demonstrates that $\kappa(G) > 0$ when expressing reducing our control setting  to \ocoam{} (\Cref{sec:prelim}), and stability of our control parametrization ensures $\psiG(n)$ decays exponentially; the reader should have in mind the scalings $\kappa(G) = \Omega(1)$ and $\psiG(n) = \exp(-\Omega(n))$. 
\nipstogfalse{
\paragraph{Exact \ocoam}
} For the exact setting, we have the following guarantee:
\begin{theorem}[\semions{} regret, exact case]\label{thm:semions_memory} Suppose $\kappa = \kappa(G) >0$, \Cref{asm:loss_reg} holds, and consider the update rule \Cref{eq:known_semions}  with  parameters $\eta = \frac{1}{\alpha}$, $\lambda := 6h\radY^2 \radG^2$. Suppose in addition that $h$ is large enough to satisfy  $\psiG(h+1)^2 \le \radG^2/T$. Then, 
\iftoggle{nips}
{we have 
  $\MemReg_T \le 
  3 \alpha h D^2 \radH^2 + \frac{3dh^2 \Leff^2  \radG}{\alpha \kappa^{\nicehalf}}  \log\left(1 + T\right)$.
}
{
  \begin{align*}
\MemReg_T \le 
  3 \alpha h D^2 \radH^2 + \frac{3dh^2 \Leff^2  \radG}{\alpha \kappa^{\nicehalf}}  \log\left(1 + T\right).
\end{align*}
}
\end{theorem}
The above regret mirrors fast rates for strongly convex rates \ocom{} and exp-concave standard \oco. Its proof departs significantly from those of existing \ocom{} bounds, and is sketched in \Cref{sec:semions_sketch}, and formalized in \Cref{sec:known}.
\nipstogfalse{
\paragraph{Approximate \ocoam}}
For the approximate setting, we assume \nipstogfalse{bounds on the error of the estimates $\matvhat_t$ and $\Ghat$.}
\begin{assumption}[Approximate \semions{} assumptions]\label{asm:approx_ocoam} We assume that $\|\Ghat -\Gst\|_{\loneop} \le \epsG$, $\max_{t \ge 1}\|\matv_t - \matvhat_t\|_2 \le \cv \epsG$ for some $\cv > 0$, and that $\Ghat^{[i]} = 0$ for all $i > h$.
\end{assumption}
For simplicity, the following theorem considers $\epsG^2 \ge 1/\sqrt{T}$, which arises in our estimation-exploitation tradeoff for control of unknown linear systems. It shows that \semions{} exhibits a quadratic sensitivity to the estimation error $\epsG$, with $\MemReg_T$ scaling as $\frac{1}{\alpha} \log T \cdot T\epsG^2$.

\begin{theorem}[\semions{} regret, approximate case]\label{thm:semions_unknown} Suppose \Cref{asm:loss_reg,asm:approx_ocoam} holds, and in addition $\nablatwo \ell_t \preceq L I$ uniformly, and  $\epsG^2 \ge 1/\sqrt{T}$. Consider  the update rule \Cref{eq:unknown_semions} with parameters $\eta =  \frac{3}{\alpha}$ and  $\lambda= (T\epsG^2 + h \radG^2)$.  Then
\iftoggle{nips}
{$\MemReg_T \lesssim  \log T \left(\frac{C_1}{\alpha \kappa^{\nicehalf}} + C_2\right) $ $\cdot\left(T\epsG^2 + h^2(\radG^2 + \radY)\right)$, 
}
{
  \begin{align*}
\MemReg_T \lesssim  \log T \cdot\left(\frac{C_1}{\alpha \kappa^{\nicehalf}} + C_2\right) \left(T\epsG^2 + h^2(\radG^2 + \radY)\right),
\end{align*}
}
where $C_1 := (1+\radY)\radG(h+d) \Leff^2$ and $C_2:= (L^2 \cv^2/\alpha +  \alpha D^2)$. 
\end{theorem}
The above mirrors the strongly convex setting, where online gradient descent with $\epsilon$-approximate gradients attains $\frac{1}{\alpha}T\epsilon^2$ regret \citep{simchowitz2020improper}. In \Cref{sec:unknown} we provide two stronger versions: The first (\Cref{thm:semions_unknown_clam}) includes a certain negative regret term which is indispensible for the control setting, and accomodates misspecified $\lambda$. The second (\Cref{thm:unknown_granular}) allows for $\epsG^2 \ll 1/\sqrt{T}$, establishing $(T\epsG)^{2/3}$ regret for small $\epsG$. \Cref{sec:unknown} also details the proof of \Cref{thm:semions_unknown}, which constitutes the main technical undertaking of the paper. The proof draws heavily on ideas from the proof of \Cref{thm:semions_memory}, which we presently sketch. 

\subsection{Proof Sketch for Exact \semions{} (\Cref{thm:semions_memory}) \label{sec:semions_sketch}}
Recall the with-memory and unary regret defined \iftoggle{nips}{at the start of \Cref{sec:ocoam}}{in \Cref{eq:mem_reg_def,eq:std_reg_def}}, and set $\nabla_t := \nabla f_t(\matz_t)$.  Following \cite{anava2015online}, our analysis begins with the following identity:
\iftoggle{nips}
{
  \begin{align*}
\MemReg_T = \UnaReg_T + \MoveDiff_T, \quad \text{where } \MoveDiff_T := \textstyle\sum_{t=1}^T F(\matz_{t:t-h}) - f(\matz_t).
\end{align*}
}
{
  \begin{align*}
\MemReg_T = \UnaReg_T + \MoveDiff_T, \quad \text{where } \MoveDiff_T := \sum_{t=1}^T F(\matz_{t:t-h}) - f(\matz_t).
\end{align*}
}

That is, $\MemReg_T$ equals the standard regret on the $f_t$ sequence, plus the cumulative difference between $F_t$ (with memory) and $f_t$ (unary). The bound on $\UnaReg_T$ for \semions{} mirros the analysis of standard \ons{}, using that $\nablatwo f_t(\matz_t) \succsim \matH_t^\top \matH_t \succsim \nabla_t \nabla_t^\top$ (\Cref{lem:ft_facts}). To bound $\MoveDiff_T$, past work on \ocom{} applies the triangle inequality and an $L$-Lipschitz condition on $F$ to bound the movement difference by \nipstogfalse{aggregate} movement in the Euclidean norm:
\iftoggle{nips}
{
  \begin{align}
\MoveDiff_T \le \poly(L,h)\cdot\EucMoveCost_T, \altquad \text{where }\EucMoveCost_T :=\textstyle\sum_{t=1}^T\|\matz_{t} - \matz_{t-1}\|.  \label{eq:eucMoveCost}
\end{align}
}
{
  \begin{align}
\MoveDiff_T \le \poly(L,h)\cdot\EucMoveCost_T, \quad \text{where }\EucMoveCost_T :=\sum_{t=1}^T\|\matz_{t} - \matz_{t-1}\|.  \label{eq:eucMoveCost}
\end{align}
}

The standard approach is to run \ogd{} on the unary losses \citep{anava2015online} When doing so, the differences $\|\matz_{t} - \matz_{t-1}\| $ scale with Lipschitz constant $L$ and step sizes $\eta_t$. In particular, for the standard $\eta_t \propto \frac{1}{\alpha t}$ step size for $\alpha$-strongly convex losses, $\sum_{t=1}^T\|\matz_{t} - \matz_{t-1}\| = \bigohst{\frac{1}{\alpha} \log T}$.  Since \ogd{} also has logarithmic unary regret, we obtain $\bigohst{\frac{\poly(L,h)}{a}\log T}$ memory regret. 
\nipstogfalse{
  
}
However, when $\ell_t$ are strongly convex, the induced \ocoam{} losses $f_t$ need not be \citep{foster2020logarithmic}, and \Cref{thm:Regmu_lb} shows that it is impossible to attain both logarithmic regret and logarithmic movement cost simultaneously. 
As a work around, we establish a refined movement  bound in terms of  $\matY_t$-sequence (see \Cref{lem:Ft_diff}):
  \begin{align}
 \MoveDiff_T \le \poly(L,h)\cdot\AdapMoveCost_T, \altquad \nipstogfalse{\text{where }} \AdapMoveCost_T := \sum_{i=1}^h\sum_{t=1}^T \|\matY_t(\matz_{t - i}-\matz_{t-i-1})\|_2, \iftoggle{nips}{\nonumber}{\label{eq:adap_move_cost}}
\end{align}
Via \Cref{lem:Y_bound}, the \nipstogfalse{definition of the} \semions{} updates and an application of Cauchy-Schwartz yields:
\iftoggle{nips}
{
  \begin{align}
\AdapMoveCost_T \le \BigOh{\poly(L,h)}\cdot \Big( \underbrace{\textstyle \sum_{t=1}^T \nabla_t^\top \Lambda_t^{-1} \nabla_t}_{\text{$\nabla$-movement}}\Big)^{\nicehalf} \cdot \Big( \underbrace{\textstyle \sum_{t=1}^T \matY_t^\top \Lambda_t^{-1} \matY_t}_{\text{$\matY$-movement}}\Big)^{\nicehalf}. \label{eq:adap_move_bound_ub}
\end{align}
}
{
  \begin{align}
\AdapMoveCost_T \le \BigOh{\poly(L,h)}\cdot \left( \underbrace{\sum_{t=1}^T \nabla_t^\top \Lambda_t^{-1} \nabla_t}_{\text{$\nabla$-movement}}\right)^{\nicehalf} \cdot \left( \underbrace{\sum_{t=1}^T \matY_t^\top \Lambda_t^{-1} \matY_t}_{\text{$\matY$-movement}}\right)^{\nicehalf}. \label{eq:adap_move_bound_ub}
\end{align}
}

 Readers familiar with the analysis of \ons{} will recognize the $\nabla$-movement as the dominant term in its regret bound\iftoggle{nips}{, and can be bounded in a similar fashion.}{. Recalling the \ons{} preconditioner  $\Lambda_{t,\ons}:= \lambda I + \sum_{s=1}^{t-1}\nabla_s \nabla_s^\top$, one obtains $\sum_{t=1}^T \nabla_t^\top (\Lambda_{t,\ons})^{-1} \nabla_t \le \bigohst{d \log T}$ by an application of the log-determinant potential lemma (\Cref{lem:log_potential}). 
A simple computation reveals that the $\semions$ preconditioner dominates the \ons{}  one: $\Lambda_{t} := \sum_{s=1}^{t-1} \matH_s^\top \matH_s + \lambda I  \succeq \bigomst{\Lambda_{t,\ons}}$ (consequence of \Cref{lem:ft_facts}). Thus, similar bound on $\nabla$-movement holds. }
\nipstogfalse{
  
}
To address the $\matY$-movement, we use the convolution-invertibility assumption (\Cref{defn:input_recov}). This assumption implies that convolution with $G = (G^{[i]})_{i\ge 0}$ is invertible, meaning that we can essentially invert the sequence $(\matH_1,\matH_2,\dots)$ defined by $\matH_t := \sum_{i=0}^h G^{[i]}\matY_{t-i}$ so as to back out $(\matY_1,\matY_2,\dots)$. Linear algebraically, this implies (see \Cref{prop:covariance_lb})
\iftoggle{nips}
{
  $\Lambda_t - \lambda I = \sum_{s=1}^t \matH_s^\top \matH_s \succeq \frac{\kappa(G)}{2} \sum_{s=1}^t\matY_t^\top \matY_t - \bigohst{1}.$
}
{:
  \begin{align}
\Lambda_t - \lambda I = \sum_{s=1}^t \matH_s^\top \matH_s \succeq \frac{\kappa(G)}{2} \sum_{s=1}^t\matY_t^\top \matY_t - \bigohst{1}.
\end{align}
}
In other words, up to an additive remainder term and multiplicative factor of $\kappa(G)$, the $\matH_s$-covariance dominates that $\matY_s$-covariance. Hence, $\Lambda_{t}$ roughly dominates $\sum_{s=1}^{t-1} \matY_s^\top \matY_s + \lambda I$. Hence, $\matY$-movement is also $\bigohst{d\log T}$ by an application of the log-determinant lemma \nipstogtrue{(\Cref{lem:log_potential})}. This yields a logarithmic upper bound on $\MoveDiff$, and thus logarithmic memory regret.

\hypertarget{rationaletarget}{\paragraph{\semions{} v.s. \ons{}}} Standard $\ons{}$ uses a preconditioner based on outer products of $\nabla_t$.
However, the movement difference depends on gradients of the with-memory loss $F_t(\cdot,\dots,\cdot)$, which may not be aligned with direction of $\nabla_t$. Indeed, $\nabla_t \in \mathrm{RowSpace}(\matY_t)$, but this is in general a strict inclusion; that is, $\matY_t$  accounts for more possible directions of movement that $\nabla_t$. Thus, \semions{} forms its preconditioner  to ensure slower movement in all $\matY_t$-directions, using $\matH_t$ as a proxy via the convolution-invertibility analysis.




\newcommand{\Regmu}[1][\mu]{#1\text{-}\mathrm{Reg}_{T}}
\subsection{The Regret-Movement Tradeoff \label{sec:lb}}
As described above, the standard analysis of \ocom{} bounds the sum of the unary regret and Euclidean total variation of the iterates. While this permits logarithmic regret when $f_t$ are strongly convex,
\nipstogfalse{
	
}
\ocoam{} losses $f_t$ are not strongly convex even if $\ell_t$ are (see e.g. below). We now show that for a simple class of quadratic \ocoam{} losses, there is a nontrivial trade-off between the two terms. We lower bound
\iftoggle{nips}{
	$\Regmu := \UnaReg_T + \mu \EucMoveCost_T = \sum_{t=1}^T f_t(\matz_t) + \mu \|\matz_t - \matz_{t-1}\| - \inf_{z \in \calC}  \sum_{t=1}^T f_t(\matz_t)$, 
}
{
	\begin{align}
\Regmu := \UnaReg_T + \mu \EucMoveCost_T = \sum_{t=1}^T f_t(\matz_t) + \mu \|\matz_t - \matz_{t-1}\| - \inf_{z \in \calC}  \sum_{t=1}^T f_t(\matz_t) \label{eq:Regmu_def},
\end{align}
}
which characterizes the Pareto curve between unary regret and Euclidean movement. 
\nipstogfalse{
	
}
{}
\iftoggle{nips}{We consider}{For our construction, we }{} $d = 1$, \nipstogfalse{constraint set} $\calC = [-1,1]$, $\ell(u) = u^2$, and the memory-$1$ \ocoam{} losses  $f_t = \ell(\matv_t - \epsilon z)$\nipstogfalse{$= (\matv_t - \epsilon z)^2$}, where  $\epsilon \in (0,1]$ is fixed and $\matv_t \in \{-1,1\}$ are chosen by an adversary . On $\calC$, $f_t$ are $\mathcal{O}(\epsilon)$-Lipschitz, and have Hessian $\epsilon^2$ (thus arbitrarily small strong convexity). Still,  $\ell$ satisfies \Cref{asm:loss_reg} with $\alpha = L = 1$. We prove the following in \Cref{sec:proof:regmu_lb}:\nipstogfalse{, basing our construction off \cite[Theorem 13]{altschuler2018online}:}
\begin{theorem}\label{thm:Regmu_lb} Let $c_1,\dots,c_4$ be constants. For  $T \ge 1$ and $\mu \le c_1 T$, there exists $\epsilon = \epsilon(\mu,T)$ and a joint distribution $\calD$ over $\matv_1,\dots,\matv_T \in \{-1,1\}^T$ such that any proper (i.e. $\matz_t \in \calC$ for all $t$)
 possibly randomized algorithm $\Alg$ suffers
 \iftoggle{nips}
 {$\Exp[\Regmu] \ge c_2(T \mu^2)^{1/3}$. 
 }
 {:
 \begin{align*}
\Exp[\Regmu] \ge c_2(T \mu^2)^{1/3}.
\end{align*} 
 }
In particular, $\Exp[\Regmu[1]] \ge c_2 T^{1/3}$, and if $\Exp[\UnaReg_T] \le R \le c_3 T$, then,  $\Exp[\EucMoveCost_T] \ge c_4 \sqrt{T/R}
$. 
\end{theorem}
Hence,  existing analyses based on Euclidean movement cannot ensure better than $T^{1/3}$ regret \nipstogfalse{in the \ocom setting, and thus no better than $\Omega(T^{1/3})$ regret for online control of a known system with strongly convex losses}. Moreover,  to ensure $\UnaReg_T = \bigohst{\log T}$, then one must suffer $\sqrt{T/\log T}$ movement. In \Cref{thm:ons_move} in \Cref{sec:ons_tradeoff}, we show that  standard \ons{} with an appropriately tuned regularization parameter attains this optimal tradeoff (up to logarithmic and dimension factors), even in the more general case of arbitrary exp-concave losses.


\newcommand{\Mdrc}{M_{\mathrm{drc}}}
\newcommand{\vkhat}{\widehat{\matv}^K}
\newcommand{\Pistab}{\Pi_{\mathrm{stab}}}

\section{From  \ocoam{}\label{sec:prelim} to Online Control }

This sections proposes and analyzes the \drcons{} algorithm  via \ocoam{}.
\nipstogfalse{\subsection{Preliminaries and Assumptions} }
Recall the control setting with dynamics described by \Cref{eq:LDS_system}, and regret defined by \Cref{eq:reg_def}. Throughout, we assume that the losses satisfy the strong convexity and quadratic growth assumption  of \Cref{asm:loss_reg}. Outputs $\maty$ lie in $\R^{\dimy}$, inputs $\matu$ lie in $\R^{\dimu}$. 
\nipstogfalse{\paragraph{Stabilizing Feedback Parametrization }}For the main text of this paper, we assume knowledge of a \emph{stabilizing, static feedback} policy: that is a matrix $K \in \R^{\dimu \times \dimy}$ such that the policy $\matu_t = K \maty_t$ which is stabilizing ($\rho(\Ast + \Bst K \Cst) < 1$, where $\rho$ denotes the spectral radius). \footnote{This  may be restrictive for partially observed systems \citep{halevi1994stable}, see  \Cref{app:general} for generalizations. }
For this stabilizing $K$, we select inputs $\ualg_t := K \yalg_t + \uexalg_t$, where $\uexalg_t$ is the \emph{exogenous output} dictated by an online learning procedure. We let the \emph{nominal iterates} $\yk_t,\uk_t$ denote the sequence of outputs and inputs that would occur by selecting $\ualg_t = K \yalg_t$, with no exogenous inputs. We exploit the superposition identity \nipstogtrue{(using $[\cdot;\cdot]$ to denote vertical concatenation)} 
\iftoggle{nips}
{
\begin{align}
\begin{bmatrix}\yalg_t;\ualg_t\end{bmatrix} = \begin{bmatrix}\yk_t;\uk_t \end{bmatrix} +  \textstyle\sum_{i = 0}^{t-1} \GK^{[i]}\uex_{t-1}
\label{eq:superposition_id},
\end{align}
}
{
	\begin{align}
\begin{bmatrix}\yalg_t\\\ualg_t\end{bmatrix} = \begin{bmatrix}\yk_t\\\uk_t \end{bmatrix} +  \sum_{i = 0}^{t-1} \GK^{[i]}\uex_{t-1}
\label{eq:superposition_id},
\end{align}
}
where 
\iftoggle{nips}
{
$ \GK^{[0]} = \begin{bmatrix} 0 ; I_{\dimu}\end{bmatrix}$ and $\GK^{[i]} =\begin{bmatrix} \Cst;
K\Cst \\ \end{bmatrix}  (\Ast + \Bst K \Cst)^{i-1} \Bst$
}
{
$ \GK^{[0]} = \begin{bmatrix} 0 \\ I_{\dimu}\end{bmatrix}$ and $\GK^{[i]} =\begin{bmatrix} \Cst \\
K\Cst \\ \end{bmatrix}  (\Ast + \Bst K \Cst)^{i-1} \Bst$
} for $i \ge 1$. We call $\GK$ the \emph{nominal Markov operator}. Since $K$ is {stabilizing}, we will assume that $\GK^{[i]}$ decays {geometrically}, and that the nominal iterates are bounded. For simplicity, we take $\matx_1 = 0$.
\begin{assumption}\label{asm:stab}For some $\constk > 0$ and $\rhok \in (0,1)$ and all $n \ge 0$, $\|\Gk^{[i]}\|_{\op} \le \constk\,\rhok^n$. 
\end{assumption}
\begin{assumption}\label{asm:bound}
We assume that $(\matw_t,\mate_t)$ are bounded such that, for all $t \ge 1$, $\|(\yk_t,\uk_t)\|_{2} \le \radnat$\nipstogfalse{, where we recall that $(\yk_t,\uk_t)$ are interates under $\yalg_t = K\ualg_t$. }
\end{assumption}
\Cref{asm:stab} is analogous to ``strong stability'' \citep{cohen2018online}, and holds for \nipstogfalse{some $\rhok \in (0,1),\constk > 0$ for} any stabilizing $K$.  \Cref{asm:bound} is analogous to the bounded assumption in \citet{simchowitz2020improper}: since $K$ is stabilizing, any bounded sequence of disturbances implies a uniform upper bound on $\|(\yk_t,\uk_t)\|_{2}$\footnote{The assumed bound can be stated in terms of $\max_{t}\|\matw_t,\mate_t\|_2$. One may allow $\radnat$ to grow logarithmically (e.g. $\radnat = \mathcal{O}(\log^{1/2} T)$ for subguassian noise), by inflating logarithmic factors in the final bounds.} 

\paragraph{Benchmark Class}

\iftoggle{nips}
{
	We compete with linear dynamical controllers (LDCs) $\pi \in \Pildc$ whose \emph{closed loop} iterates are denoted $(\matypi_t,\matupi_t,\matxpi_t)$ (see \Cref{defn:LDC} for further details). These policies include static fedback laws  $\uopenpi_t = K \yopenpi_t$, but \emph{are considerably more general due to the internal state}.  
}
{
We compete with linear dynamical controllers, or \emph{LDC}s, specified by a linear dynamical system $(\Api,\Bpi,\Cpi,\Dpi)$, with internal state $\sopenpi_t \in \R^{\dimpi}$,  equipped with the internal dynamical equations $\sopenpi_{t+1} = \Api \sopenpi_t + \Bpi \yopenpi_t \quad \text{and} \quad \uopenpi_t := \Cpi \sopenpi_t + \Dpi \yopenpi_t$. 
We let $\Pildc$ denote the set of all LDC's $\pi$.  These policies include static fedback laws  $\uopenpi_t = K \yopenpi_t$, but \emph{are considerably more general due to the internal state}. The \emph{closed loop} iterates $(\matypi_t,\matupi_t,\matxpi_t,\matspi_t)$ denotes the unique sequence consistent with  \Cref{eq:LDS_system}, the above internal dynamics, and the equalities $\uopenpi_t = \matu_t$, $\yopenpi_t = \maty_t$. The sequence $(\yk_t,\uk_t)$ is a special case with $\Dpi = K$ and $\Cpi =0$.}
We consider \emph{stabilizing} $\pi$: for all bounded disturbance sequences $\max_{t \ge 1} \|\matw_t\|,\|\mate_t\| <\infty$, it holds that $\max_{t \ge 1} \|\matypi_t\|,\|\matupi_t\| < \infty$. 
These policies enjoy geometric decay, motivating the following parametrization of our benchmark class.
\begin{definition}[Policy Benchmark]\label{defn:benchmark} Fix parameters $\rhocomp \in (0,1)$ and $\constcomp > 0$. Our regret benchmark competes LDC's $\pi \in \Pistar := \Pistab(\constcomp,\rhocomp)$, where we define
\iftoggle{nips}
{
	$\Pistab(c,\rho) := \{\pi \in \Pildc : (\|\Gpicl^{[i]}\|_{\op} \le c\rho^n, \forall n \ge 0\}$,
}
{
	\begin{align*}
\Pistab(c,\rho) := \{\pi \in \Pildc : \|\Gpicl^{[i]}\|_{\op} \le c\rho^n, \forall n \ge 0\},
\end{align*}
}
where the Markov operator $\Gpicl$ is \nipstogfalse{formally defined} in \Cref{defn:gpicl}. 
\end{definition}


\paragraph {Known v.s. Unknown Dynamics } We refer to the \emph{known} dynamics setting as the setting where the learner knows the matrices $\Ast,\Bst,\Cst$ defining the dynamics in \Cref{eq:LDS_system}. In the \emph{unknown dynamics} setting, the learner does not know these matrices (but knows \nipstogfalse{a stabilizing controller} $K$). 

\paragraph{The DRC parametrization}

 Given radius $\radM > 0$ and memory $m \in \N$, we adopt the \drc{} parametrization of memory-$m$ controllers $M \in \calM $ \citep{simchowitz2020improper} :
\begin{align}
\textstyle \calM = \Mdrc(m,\radM) :=  \{M  = (M^{[i]})_{i=0}^{m-1} \in (\R^{\dimy \dimu})^m : \sum_{i=0}^{m-1}\|M\|_{\op} \le \radM\}. \label{eq:Mdfc}
\end{align}

Controllers $M \in \calM$ are then applied to estimates of the nominal outputs $\yk_t$. When the dynamics are known,  $\yk_t$ and $\uk_t$ are recovered exactly via \Cref{eq:superposition_id}.
If $\Ast,\Bst,\Cst$ are not known, we use an estimate $\Ghat$ of $\Gk$ to construct estimates $\ykhat_{1:t},\ukhat_{1:t}$:
\begin{align}
\begin{bmatrix}{\ykhat_t}; \ukhat_t\end{bmatrix} =  \begin{bmatrix} \yalg_t ;
K \yalg_t
\end{bmatrix} - \textstyle\sum_{i=1}^{t-1} \Ghat^{[i]} \uexalg_{t-i}. \label{eq:recover_hat}
\end{align}
 Going forward, we use the  more general $\ykhat_{1:t}$ notation, noting that it specializes to $\yk_{1:t}$ for known systems (i.e. when $\Ghat = \Gk$). The \drc{} parametrization  selects exogenous inputs as linear combinations  of $\ykhat_{1:t}$ under $M \in \calM$:
 \iftoggle{nips}
 {via $\uex_t(M \mid \ykhat_{1:t} ) := \sum_{i=0}^{m-1}M^{[i]}\ykhat_{t-i}$.}
 {:
 \begin{align}
\uex_t(M \mid \ykhat_{1:t} ) := \sum_{i=0}^{m-1}M^{[i]}\ykhat_{t-i}. \label{eq:uex_def}
\end{align}
 }


\subsection{Reducing \drc{} to \ocoam{}}
\nipstogfalse{In this section, we explain how the \drc{} parametrization  leads to a natural \ocoam.} Fixing  the \drc{} length $m \ge 1$, let $d = \dimy \dimu m$, and $p = \dimy + \dimu$.  Further, let $(\ykhat_t,\ukhat_t)_{t \ge 1}$ and $\Ghat$ denote estimates of $(\yk_t,\uk_t)_{t\ge 1}$ and $G_K$, respectively\nipstogfalse{; for known systems, these estimates are exact. }
\begin{definition}[\ocoam{} quantities for control]
\iftoggle{nips}{
  Let  $\embedM[\cdot]$ denote the natural embedding of $M  \in \calM$  into $\R^d$, and let $\embedM^{\shortminus 1}[\cdot]$ denote its inverse; Define the \ocoam{} matrices $\matY_t := \embedy[\ykhat_{t:t\shortminus m+1}]$, where $\embedy$ is embedding satisfying  $\matY_t z =  \uex_t(M \mid \ykhat_{1:t})$ for all $z$ of the form $z = \embed[M]$; Define the offset $\vk_t = (\yk_t,\uk_t) \in \R^{p}$, and its approximation $\vkhat_t = (\ykhat_t,\ukhat_t) \in \R^{p}$; Define the constraint set $\calC := \embed(\calM) \subset \R^d$ (that is, embed the \drc{} set into $\R^d$). 
}
{For the control setting, we use the following correspondences:
  \begin{itemize}
\item Let  $\embedM[\cdot]$ denote the natural embedding of $M  \in \calM$  into $\R^d$, and let $\embedM^{\shortminus 1}[\cdot]$ denote its inverse.  
\item Define the \ocoam{} matrices $\matY_t := \embedy[\ykhat_{t:t\shortminus m+1}]$, where $\embedy$ is embedding satisfying  $\matY_t z =  \uex_t(M \mid \ykhat_{1:t})$ for all $z$ of the form $z = \embed[M]$. 
\item Define the offset $\vk_t = (\yk_t,\uk_t) \in \R^{p}$, and its approximation $\vkhat_t = (\ykhat_t,\ukhat_t) \in \R^{p}$. 
\item Define the constraint set $\calC := \embed(\calM) \subset \R^d$ (that is, embed the \drc{} set into $\R^d$).
\end{itemize}
}
\end{definition}
We now define the relevant \ocoam{} losses as those consistent with the above notation. 
\begin{definition}[\ocoam{} losses for control]\label{defn:ocoam_control_loss} Let $\matY_t,\vk_t,\vkhat_t$ be as above. For $h \in \N$, define the \emph{exact} losses
\iftoggle{nips}
{
  $F_t(z_{t:t-h}) := \loss_t(\vk_t + \sum_{i=0}^h \GK^{[i]}\,\matY_{t-i}z_{t-i} )$, and $ f_t(z) := \loss_t(\vk_t + \matH_t z )$, where $\matH_t := \sum_{i=0}^h \GK^{[i]}\,\matY_{t-i}$.
}
{
  \begin{align}
F_t(z_{t:t-h}) := \loss_t(\vk_t + \sum_{i=0}^h \GK^{[i]}\,\matY_{t-i}z_{t-i} ), \quad f_t(z) := \loss_t(\vk_t + \matH_t z ), \quad \matH_t := \sum_{i=0}^h \GK^{[i]}\,\matY_{t-i}.
\end{align}
}
Given an estimate $\Ghat$ of $\Gk$, the \emph{approximate} unary loss is $\fhat_t(z) := \loss_t(\vkhat_t + \matHhat_t z )$ with  $\matHhat_t := \sum_{i=0}^h \Ghat^{[i]}\,\matY_{t-i}$. 
\end{definition}
We take $h = \Theta(\log T)$, since the exponential decay assumption (\Cref{asm:stab}) ensures $\Gk^{[i]} = \exp(-\Omega(h)) \approx 0$ for $i > h$\nipstogfalse{, mimicking the nilpotent toy example in \Cref{sec:ocoam_setting}}. The resulting \ocoam{} problem is to produce a sequence of iterates $\matz_t$ minimizing $\MemReg_T$ \nipstogfalse{(\Cref{eq:mem_reg_def})} on the sequence $(F_t,f_t)$. Since $\matz_t$ are embeddings of controllers, this gives rise to a natural control algorithm: for each iterate $\matz_t$, back out a \drc{} controller $\matM_t = \embed^{-1}(\matz_t)$, and applies exogenous input $\uexalg_t := \uex_t(\matM_t \mid \yk_{1:t})$. 
 In \Cref{sec:control_proofs_ful}, we streamline past work \citep{simchowitz2020improper} by providing black-box reductions bounding the control regret (\Cref{eq:reg_def}) of such an algorithm by its memory regret. \Cref{prop:redux_known} addresses the known system case, and  \Cref{prop:redux_unknown} the unknown case.\nipstogfalse{\footnote{ We note that both reductions are stated for more general \drc{} parametrizations that do not require stabilizing static feedback; these are described in \Cref{app:dynam_stabilize_feedback}. }} Because the latter is more intricate, we conclude the present discussion with an informal statement of the known system reduction:

\iftoggle{nips}{\newtheorem*{thm:e}{Proposition \ref*{prop:redux_known} (informal)}\begin{thm:e}}
{\begin{informaltheorem}[prop:redux_known]} Let algorithm $\Alg$ which produces iterates $\matz_t \in \R^d$. Let $\Alg'$ denote the control algorithm which selects $\uexalg_t := \uex_t(\matM_t \mid \yk_{1:t})$, where $\matM_t = \embed^{-1}(\matz_t)$. Then, for  $m = \BigOhTil{1}$
 \iftoggle{nips}
 {, we have 
$\ControlReg_T(\Alg') \le \MemReg_T(\Alg) + \bigohst{1}$. 
 }
 {:
 \begin{align*}
\ControlReg_T(\Alg') \le \MemReg_T(\Alg) + \bigohst{1}.
\end{align*} 
 }

\iftoggle{nips}{\end{thm:e}}{\end{informaltheorem}}

\begin{remark}[Hat-accent notation] We use $\matY_t$ even when defined using the approximate $\ykhat_{1:t}$. However, $G$ and $\vk$ do recieve hat-accents when estimates are used. This is because, while \ocoam{} can account for the approximation error on $G$ and $\vk$ (\Cref{thm:semions_unknown}), the approximation error introduced by setting $\matY_t := \embedy[\ykhat_{t:t\shortminus m+1}]$ requires  control specific arguments \nipstogfalse{(see \Cref{prop:redux_unknown}).}
\end{remark}
\subsection{The \drcons{} algorithm and guarantees}
Stating the \drcons{} algorithm is now a matter of putting the pieces together. For known systems, the learner constructs the losses in \Cref{defn:ocoam_control_loss} with $\Ghat = \Gk$, and runs \semions{} on $f_t$, and uses these to perscribe a \drc{} controller in accordance with the above discussion. For unknown systems, one constructs the estimate $\Ghat$ via least squares, and then runs \semions{} on $\fhat_t$; formal pseudocode is given in \Cref{alg:nfc,alg:estimation_static} \nipstogtrue{in \Cref{sec:pseudo_static_feed}}. Our formal guarantees are


\begin{theorem}[Guarantee for Known System]\label{thm:known_control}
Suppose \Cref{asm:loss_reg,asm:bound,asm:stab} holds, and for given $\rhocomp \in (0,1),\constcomp > 0$, let $\Pistar$ be as in \Cref{defn:benchmark}. For simplicity, also assume $\constcomp \ge \constk,\rhocomp \ge \rhok$. Then, for a suitable choice of parameters, \drcons (\Cref{alg:nfc}) achieves the bound
\iftoggle{nips}
{
$\ControlReg_T(\Alg;\Pistar) \le \log^4(1+T) \cdot \frac{ \constcomp^5 (1+\|K\|_{\op})^3 }{(1-\rhocomp)^{5}}  \cdot \dimu \dimy \radnat^2 \cdot \frac{L^2}{\alpha}$. 
}
{
\begin{align*}
\ControlReg_T(\Alg;\Pistar) \le \log^4(1+T) \cdot \frac{ \constcomp^5 (1+\|K\|_{\op})^3 }{(1-\rhocomp)^{5}}  \cdot \dimu \dimy \radnat^2 \cdot \frac{L^2}{\alpha}
\end{align*}
}
\end{theorem}
\begin{theorem}[Guarantee for Unknown System]\label{thm:unknown_control}
Suppose \Cref{asm:loss_reg,asm:bound,asm:stab} holds, and for given $\rhocomp \in (0,1),\constcomp > 0$, let $\Pistar$ be as in \Cref{defn:benchmark}. For simplicity, also assume $\constcomp \ge \constk,\rhocomp \ge \rhok$. In addition, assume $\nablatwo \ell_t \preceq L I$ uniformly. Then, for any $\delta \in (0,1/T)$, \drcons{} with an initial estimation phase (\Cref{alg:estimation_static}) for an appropriate choice of parameters has the following regret with probability $1-\delta$: 
\iftoggle{nips}
{
 $ \ControlReg_T(\Alg; \Pistar) \lesssim  \sqrt{T}\log^3(1+T)\log\frac{1}{\delta} \cdot \frac{\constcomp^8(1+\opnorm{K})^5}{(1-\rhocomp)^{10}} \cdot \dimy(\dimu + \dimy)\radnat^5\cdot \frac{L^2 }{\alpha}$.
}
{
  \begin{align*}
\ControlReg_T(\Alg; \Pistar) &\lesssim  \sqrt{T}\log^3(1+T)\log\frac{1}{\delta} \cdot \frac{\constcomp^8(1+\opnorm{K})^5}{(1-\rhocomp)^{10}} \cdot \dimy(\dimu + \dimy)\radnat^5\cdot \frac{L^2 }{\alpha}
\end{align*}
}

\end{theorem}
Together, these bounds match the optimal regret bounds for known and unknown control, up to logarithmic factors \citep{agarwal2019logarithmic,simchowitz2020naive}. The above theorems  are proven \Cref{sec:control_proofs_ful}, which also gives complete statements which specify the parameter choices \Cref{thm:drconsstat_known_granular,thm:drcons_unknown_specific}. In addition, \Cref{app:general}  generalizates the algorithm by  replacing \emph{static} $K$ in the \drc{} algorithm with a dynamic nominal controller $\pinot$, for which analogous guarantees are stated in  \Cref{sec:control_proofs_ful}. Importantly, \Cref{app:invertibility} verifies that convolution-invertibility holds:

\begin{restatable}{lemma}{lemkapK}\label{lem:kap_bound_K} For $\kappa$ as in \Cref{defn:input_recov}, we have $\kappa(\Gk) \ge \frac{1}{4}\min\{1,\|K\|_{\op}^{\shortminus 2}\}$.
\end{restatable}

\nipstogfalse{

\begin{algorithm}[h]
    \textbf{parameters}: \\
    \quad Newton parameters $\eta,\lambda >0$\\
    \quad \drc{} parameters radius $\radM > 0$, \drc{}  length $m \ge 1$,  memory length $h \ge 0$\\
    \quad closed-loop Markov operator estimate $\Ghat$. \algcomment{~~\% if known system, set $\Ghat \gets \Gk$}\\
    \textbf{initialize:} \\
    \quad Constraint set $\calM\gets \Mdfc(h,\radM)$ (\Cref{eq:Mdfc})\\
    \quad \semions{} subroutine $\calA \leftarrow \semions(\eta,\lambda,\embed(\calM))$ (\Cref{alg:semions}) \label{line:ons_instatiate}\\
    \quad initial values $\ykhat_{0},\ykhat_{\shortminus 1},\dots,\ykhat_{\shortminus (m+h)} \gets 0 $ 
    \For{$t=1,2,\dots$:} 
    {
  \textbf{recieve} $\yalg_t$ from environment, iterate $\matz_t$ from $\calA$, and set \drc{} parameter $\matM_t \leftarrow \embedM^{\shortminus 1}[\matz_{t}]$.\\
      Construct estimate $\vkhat_t = (\ykhat_t,\ukhat_t)$ via \Cref{eq:recover_hat}\\
   \textbf{play} input $\ualg_t \leftarrow K\yalg_t + \uex_t(\matM_t \midykhat{t})$.  \label{line:total_input}\\
    \textbf{suffer} loss $\cost_t(\yalg_t,\ualg_t)$, and observe $\cost_t(\cdot)$ \label{line:suffer}.\\
  \textbf{feed} $\calA$ the pair $(\ell_t,\matHhat_t,\vkhat_t)$, defined in \Cref{eq:unary_unknown}, and update $\calA$.  \label{line:triples}
  }
  \caption{ Disturance Response Control via Online Newton Step ($\drcons$). }
  \label{alg:nfc}
\end{algorithm}

\begin{algorithm}[h]
	       \textbf{Input: }\\
	     {}
    	\quad Newton parameters $\eta,\lambda >0$\\ 
    	 {}\quad \drc{} parameters radius $\radM > 0$, \drc{}  length $m \ge 1$,  memory length $h \ge 0$\\
	     {}\quad Estimation Length $N \ge 0$ \algcomment{\qquad \% $N \propto \sqrt{T}$}\\
	      {}\textbf{Initialize}  $\Ghat^{[0]} = \begin{bmatrix} 0_{\dimu \times \dimy} \\ I_{\dimu} \end{bmatrix}$, and $\Ghat^{[i]} = 0$ for $i > h$.\\
	    \For{t = $1,2,\dots,N$}
      {
	        {} \textbf{receive} $\yalg_t$\\
	         {} \textbf{play} $\ualg_t  = \uexalg_t + K\yalg_t$, where $\uexalg_t \iidsim \mathcal{N}(0,I_{\dimu})$. \\
	   }
	      {} \textbf{estimate}  $\Ghat^{[1:h]} = (\Ghat^{[i]})_{i \in [h]} \leftarrow \argmin_{G^{[1:h]}} \sum_{t=h+1}^N\|\valg_t - \sum_{i=1}^h G^{[i]}\uexalg_{t-i}\|_2^2.$\\
	      {}\textbf{run} \Cref{alg:nfc} for times $t = N+1,N+2,\dots,T$, using $\Ghat$ as the Markov parameter estimate, and parameters $m,h,\lambda,\eta$. \label{line:ls_step}
	    \caption{Full \drcons{} for Unknown System, with estimation\label{alg:estimation_static}}
\end{algorithm}
}

\nipstogfalse{
  \subsection{Discussion of Results\label{app:comparison_discussion}}

In this work, we demonstrate that fast rates for online control, and in particular, the optimal $\sqrt{T}$ regret rate  \cite{simchowitz2020naive} for the online LQR setting, are achievable with non-stochastic noise.  Interestingly, simultaneous work by \citet{lale2020logarithmic} shows that the presence of observation noise implies that the optimal regret for purely stochastic LQG is in fact \emph{polylogarithmic}.  At first this seems puzzling because, on face, LQG appears to be a strict generalization of LQR. However, $\poly(\log T)$ regret occurs when LQG has a strictly non-degenerate stochastic observation noise $\mate_t$, which is \emph{not} the case in LQR. This faster rate is achievable because the noise on the observation provides continuous exploration, allowing the learner to continue to learn with dynamics while simultanously exploiting near-optimal policies. Alternatively, this observation noise can be understood as making the baseline comparator \emph{easier} (i.e. $\min_{\pi\in \Pi} K_T(\pi)$ is larger),  because the underlying control problem is more difficult.

Since we are not guaranteed this observation noise in purely non-stochastic control (indeed, there may be no observation noise at all), $\sqrt{T}$ is still the optimal rate in our setting. Thus, our regret guarantees contribute to the following surprising characterization of regret (with respect to linear dyanic policies) in linear control:
\begin{itemize}
	\item For known system dynamics, non-stochastic control is just easy as stochastic (\Cref{thm:known_control}). There is no substantial price to pay for past mistakes, even under potentially unpredictable, non-stochastic disturbances.
	\item For unknown system dynamics, stochastic process noise confers little advantage over adversarial noise; both have quadratic sensitivity to error (\Cref{thm:unknown_control}).
	\item However, there \emph{is} an advantage to having non-degenerate observation noise. But this is due to continual \emph{exploration} induced by stochastic noise, and \emph{not} because stochastic reduces sensitivity to error.
\end{itemize}
As mentioned in the introduction, competing with \emph{arbitrary} policies (e.g. the optimal control law given the noise) requires regret which is linear in $T$ \citep{li2019online}. Understanding the optimal competive ratio, or further assumptions which allow sublinear regret with respect to the optimal control law, remain an interesting direction for future work.

}


\section{Proof of Logarithmic Memory Regret (\Cref{thm:semions_memory}) \label{sec:known}}
This section proves \Cref{thm:semions_memory}. We begin by bounding the standard (no-memory) regret in \Cref{ssec:online_semi_newton}, and then turn to agressing the contribution of memory in \Cref{ssec:ons_memory}. All ommitted proofs, as well as the proof of \Cref{prop:covariance_lb}, are given in \Cref{app:unknown_proofs} in numerical order.
\subsection{Bounding the (unary) \oco{} Regret \label{ssec:online_semi_newton}}
As a warmup, we establish a bound on the no-memory regret for \semions. Throughout, recall the parameters from \Cref{defn:pol_reg_pars}, which we assume to be finite.
\begin{proposition}\label{thm:semions} Suppose the the losses satisfy \Cref{asm:loss_reg}, and $\kappa_h := \kappa(G) > 0$. Then, for $\eta \ge \frac{1}{\alpha}$,  $\semions(\lambda,\eta,\calC)$ fed pairs $(f_t,\matH_t)$ satisfies the following:
\begin{align*}
\UnaReg_T := \sum_{t=1}^T f_t(\matz_t) - \min_{z \in \calC} \sum_{t=1}^T f_t(z) \le  \frac{ \eta  d \Leff^2 }{2}\log\left(1 + \frac{T\radH^2}{\lambda}\right) + \frac{\lambda\diamz^2}{2\eta}.
\end{align*}
\end{proposition}
This section proves the above proposition, and all ommited proofs in the proofs in this section are deferred to \Cref{app:proof:ssec:online_semi_newton}.
	First, let us establish two simple structural properties of $f_t$:
	\begin{lemma}\label{lem:ft_facts} For all $z \in \calC$
	\begin{enumerate}
	\item $\nablatwo f_t(z) \succeq \alpha \matH_t^\top \matH_t$
	\item There exists a function $g_t(z) \in \R^{\dimv}$ such that $\nabla f_t(z) = \matH_t^\top g_t(z) $, and  $\|g_t(z)\| \le \Leff $. In particular, $\nabla f_t(z)\nabla f_t(z)^\top \preceq \Leff^2 \matH_t^\top \matH_t$.
	\end{enumerate}
	\end{lemma}
	\begin{proof} Point (1): By the chain rule and the fact that $\nablatwo (z \mapsto \matv_t + \matH z) = 0$, we have $\nablatwo f(z) = \matH_t^\top \nablatwo \cost( \matv_t + \matH_t z) \matH_t $. Since $\cost_t$ is strongly convex, $\nablatwo \cost( \matv_t + \matH_t z) \succeq \alpha I$. Point (2): Again invoking the chain rule, $\nabla f_t(z) = \matH_t^\top g_t(z) $, where $g_t(z) = \nabla \cost_t(\matv_t + \matH_t z)$. Since $\cost_t$ is $L$-subquadratic, $\|g_t(z)\| \le L \max\{1,\|\matv_t + \matH_t z\|_2\} \le L\max\{1, \radv + \radG \max_{t, z\in \calC}\|\matY_t z\|_2\} = L \max\{1, \radv + \radG \radyc\} = \Leff$.  
	\end{proof}
	
	Next, we establish a simple quadratic lower bound, which mirrors the basic inequality in analysis of standard \ons:
	\begin{lemma}[Quadratic Lower Bound]\label{lem:quad_lb} For all $z_1,z_2 \in \calC$, we have \begin{align*}
	f_t(z_1) \ge f_t(z_2) + \nabla f_t(z_2) + \frac{\alpha}{2}\|\matH_t(z_1 - z_2)\|_2^2.
	\end{align*}
	\end{lemma}
	\begin{proof}[Proof of \Cref{lem:quad_lb}] By Taylor's theorem, there exists a $z_3$ on the segment joining $z_1$ and $z_2$ for which $f_t(z_1) \ge f_t(z_2) + \nabla f_t(z_2) + \frac{1}{2}\|(z_1 - z_2)\|_{\nablatwo f_t(z_3)}^2.$ By \Cref{lem:ft_facts}, $\nablatwo f_t(z_3) \succeq \alpha \matH_t\matH_t^\top$.
	\end{proof}

	\begin{remark}\label{rem:not_use_Hessian} Observe that \Cref{lem:quad_lb} uses the fact that $\nablatwo f_t(z) \succeq \alpha \matH_t\matH_t^\top$ \emph{globally}. \Cref{lem:quad_lb} may be \emph{false} if instead one replaces $\matH_t^\top\matH_t$ in the definition with $\nablatwo f_t(z_t)$, because the latter may be very large at a given point. This is why we use $\matH^\top \matH_t$ in the definition of $\Lambda_t$, as opposed to the full-Hessian. This is no longer an issue if one assume that $\nablatwo f_t(z) \preceq \beta I$ globally, in which case one pays for the conditioning $\beta/\alpha$. 
	\end{remark}
	\begin{remark}[Comparision to Cannonical Online Newton]\label{rem:comparison_to_cannonical_ons}
	Let us compare the above to the cannonical Online Newton Step algorithm \cite{hazan2007logarithmic}. This algorithm applies to exp-concave functions, which satisfy the bound $\nablatwo f \succeq \alpha \nabla f (\nabla f)^\top$ globally. For these functions, the analogue of \Cref{lem:quad_lb}, with $f_t(z_1) \ge f_t(z_2) + \nabla f_t(z_2) + \frac{\alpha}{2}\|\nabla f_t(z_2) (z_1 - z_2)\|_2^2$ \emph{does in fact} hold, abeit due to a somewhat trickier argument \cite[Lemma 4.3]{hazan2019introduction}. This enables the algorithm to use the preconditioner $\Lambda_t = \lambda I + \sum_{s=1}^t \nabla f (\nabla f)^\top $. Note however that this yields a smaller pre-conditioner $\Lambda_t$, for which \Cref{prop:covariance_lb} may fail.
	\end{remark}
	As a consequence, we obtain intermediate regret bound for \semions, which mirrors the standard analysis of online Newton step  (e.g.~\citet[Chapter 4]{hazan2019introduction}).
	\begin{lemma}[Online Semi-Newton Step Regret]\label{lem:semions_reg} Suppose that $\eta \ge \frac{1}{\alpha}$. Then, 
	\begin{align*}
	\sum_{t=1}^T f_t(\matz_t) - \inf_{z \in \calC}\sum_{t=1}^T f_t(z) \le \frac{\lambda \diamz^2}{2\eta} +  \frac{\eta}{2}\sum_{t=1}^{T} \nabla_t^\top \Lambda_t^{-1}  \nabla_t,
	\end{align*}
	\end{lemma}

	Lastly, we recall a standard log-det potential lemma. To facillitate reuse, the lemma is stated for a slightly more general sequence of matrices $\Lamtil_t $:
	\begin{lemma}[Log-det potential]\label{lem:log_potential} Suppose that $\Lamtil_t \succeq c\sum_{t=1}^T \matH_t^\top \matH_t + \lambda_0$. Then,
	\begin{align*}
	\sum_{t=1}^T \trace(\matH_t \Lamtil_t^{-1}  \matH_t^\top) \le \frac{d}{c} \log\left(1 + \frac{c T \radH^2}{\lambda_0}\right)
	\end{align*}
	\end{lemma}
	\begin{proof} \newcommand{\Lamcheck}{\check{\Lambda}}

	Define $\Lamcheck_t = \sum_{t=1}^T \matH_t^\top \matH_t + \frac{\lambda_0}{c}$. Then,  $\sum_{t=1}^T \trace(\matH_t \Lamtil_t^{-1}  \matH_t^\top)  \le \frac{1}{c}\sum_{t=1}^T \trace(\matH_t \Lamcheck_t^{-1}  \matH_t^\top)$. The result now follows from the standard log-det potential lemma (see e.g.  \citet[Proof of Theorem 4.4]{hazan2019introduction}).
	\end{proof}
	\begin{proof}[Proof of \Cref{thm:semions}] Begin with the unary bound:
	\begin{align*}
	\UnaReg_T := \sum_{t=1}^T f_t(\matz_t) - \inf_{z \in \calC}\sum_{t=1}^T f_t(z) \le\frac{\lambda \diamz^2}{2\eta} +  \frac{\eta}{2}\sum_{t=1}^{T} \nabla_t^\top \Lambda_t^{-1}  \nabla_t.
	\end{align*}
	From \Cref{lem:ft_facts}, we have $\nabla_t\nabla_t^\top \preceq \Leff^2 \matH_t^\top \matH_t$. Since $\Lambda_t \succ 0$, this implies that  $\nabla_t^\top \Lambda_t^{-1}  \nabla_t = \langle \nabla_t \nabla_t, \Lambda_t^{-1} \rangle \le \Leff^2 \langle \matH_t^\top \matH_t, \Lambda_t^{-1} \rangle = \Leff^2 \trace(\matH_t \Lambda_t^{-1} \matH_t^\top)$. Thus, by \Cref{lem:log_potential}, 
	\begin{align}
	\frac{\eta}{2}\sum_{t=1}^{T} \nabla_t^\top \Lambda_t^{-1}  \nabla_t \le \frac{\eta\Leff^2 }{2}\sum_{t=1}^{T} \trace(\matH_t \Lambda_t^{-1} \matH_t^\top) \le \frac{ d \eta\Leff^2 }{2}\log\left(1 + \frac{T\radH^2}{\lambda}\right).\label{eq:move_cost_first}
	\end{align}
	\end{proof}

\subsection{Memory Regret for Known System\label{ssec:ons_memory}}
	In this section, we adress movement costs, thereby proving \Cref{thm:semions_memory}. In what follows, we make the simplifying assumption that $\matz_s = \matz_1$ for $s \le 1$. We will remove this assumption at the end of the proof. Our goal is to bound:
	\begin{align*}
	\MemReg_T &:= \sum_{t=1}^{T} F_t(\matz_t,\dots,\matz_{t-h}) - \min_{z \in \calC} f_t(z) \\
	&= \underbrace{\sum_{t=1}^{T} F_t(\matz_t,\dots,\matz_{t-h}) - f_t(\matz_t)}_{(\MoveDiff_T)} + 
	\underbrace{\sum_{t=1}^{T} f_t(\matz_t)-\min_{z \in \calC} \sum_{t=1}^T f_t(z)}_{(\UnaReg_T)}.
	\end{align*}
	The second term is bounded by direct application of \Cref{thm:semions}. For the first term,  we begin with the following lemma, which shows that the relevant movement cost is only along the $\matY_{t-i}$ directions:
	\begin{lemma}[Movement Cost]\label{lem:Ft_diff} For all $t \ge 1$, we have
	\begin{align*}
	|F_t(\matz_t,\dots,\matz_{t-h}) - f_t(\matz_t)| \le \Leff \radG \sum_{i=1}^h \|\matY_{t-i}(\matz_t - \matz_{t-i})\|_2.
	\end{align*}
	Therefore, by the triangle inequality,  rearranging summations, and the assumption $\matz_{s} = \matz_1$ for $s \le 1$,
	\begin{align*}
	\MoveDiff_T \le h \Leff \radG \sum_{s=1-h}^T \sum_{i=1}^{h-1} \|\matY_{s}(\matz_{s+i+1} - \matz_{s+i})\|_2 \cdot \I_{1 \le s+i \le t-1}.
	\end{align*}
	\end{lemma}
	Next, let us develop a bound on $\|\matY_{s}(\matz_{t+1}-\matz_t)\|_2$:
	\begin{lemma}\label{lem:Y_bound} Adopt the convention $\Lambda_s = \Lambda_1$ for $s \le 0$.
	Further, consider $s \le t$, with $t \ge 1$ and $s$ possibly negative. Then, $\|\matY_{s}(\matz_{t+1}-\matz_t)\|_2\le  \eta \Leff \trace(\matY_s \Lambda_s^{-1}\matY_s)^{\nicehalf} \trace(\matH_t^\top \Lambda_t^{-1}\matH_t))^{\nicehalf}$. Therefore,
	\begin{align*}
	\MoveDiff_T \le \eta h^2\Leff \radG\cdot \sqrt{\sum_{t=1-h}^{T}\trace(\matY_t \Lambda_t^{-1}\matY_t) }\cdot\sqrt{\sum_{t=1}^{T}\trace(\nabla_t^\top \Lambda_t^{-1}\nabla_t)}~.
	\end{align*}
	\end{lemma} 
	Now, we already bounded the sum of the terms $\trace(\nabla_t^\top \Lambda_t^{-1}\nabla_t)$ in \Cref{eq:move_cost_first}:
	\begin{align}
	\sum_{s=1}^{T} \trace(\nabla_t \Lambda_t^{-1}\nabla_t) \le  d\Leff^2\log(1 + \frac{T\radH^2}{\lambda}). \label{eq:X_move_contrib}
	\end{align}
	The main technical challenge is to reason about the sum $\trace(\matY_t \Lambda_t^{-1}\matY_t)$. We bound this quantity using the following proposition:
	\begin{restatable}{proposition}{propcovlb}\label{prop:covariance_lb} Suppose that $\kappa(G) > 0$, and define $\cpsi[t]:= 1 \vee \frac{t\psiG(h+1)^2}{h\radG^2}$. Then, for any $\matY_{1-h},\matY_{2-h},\dots,\matY_{t}$, the matrices $\matH_s = \sum_{i=0}^{[h]} G^{[i]}\matY_{s-i}$ satisfy
	\begin{align*}
	\sum_{s=1}^t \matH_s^\top \matH_s \succeq \frac{\kapnot(G)}{2}\cdot \left(\sum_{s=1-h}^{t} \matY_{s}^{\top}\matY_{s}\right) - 5h \radH^2 \cpsi[t]I.
	\end{align*}
\end{restatable}
	The above proposition is proved in \Cref{sec:prop:covariance_lb}. Under the assumption of the theorem, we have $\cpsi[t] \le 1$, so $5h \radH^2 \cpsi[t] \le 5h \radH^2$. Thus, for $\lambda = 6h \radH^2$, we have $\Lambda_t \ge \frac{\lambda}{6} I + \kapnot \sum_{s=1-h}^{t} \matY_{s}^{\top}\matY_{s} $. Note that this holds even for $t \le 0$, with the above convention $\Lambda_t = \Lambda_1$ for negative $t$.
	Thus, \Cref{lem:log_potential} and the simplifications $\radY \le \radH$, $\kapnot \le 1$ gives
	\begin{align}
	\sum_{s=1-h}^{T} \trace(\matY_t \Lambda_t^{-1}\matY_t) \le \frac{2d}{\kapnot} \log\left(1 + \frac{ 6 \kapnot T \radY^2}{2\lambda}\right)  \le \frac{2d}{\kapnot}\log \left(1 + \frac{3 \radH^2}{\lambda}\right) \label{eq:Y_bound_potential}.
	\end{align}

	 We can now complete the proof of \Cref{thm:semions_memory}.
	\begin{proof}[Proof of \Cref{thm:semions_memory}]
	Combining \Cref{lem:Y_bound}, \Cref{eq:X_move_contrib,eq:Y_bound_potential}, and finally the unary regret bound from \Cref{thm:semions}
	\begin{align*}
	&\MoveDiff_T+ \UnaReg_T \\
	&\le \UnaReg_T  +\eta h^2\Leff^2 \radG \cdot \sqrt{\sum_{t=1-h}^{T}\trace(\matY_t \Lambda_t^{-1}\matY_t) }\cdot\sqrt{\sum_{t=1}^{T}\trace(\matH_t^\top \Lambda_t^{-1}\matH_t)}.\\
	&\le \UnaReg_T  + \sqrt{\frac{2}{\kapnot}} d\eta h^2\Leff^2 \radG \log(1 + \frac{3T\radH^2}{\lambda}).
	\end{align*}
	Finally, since $\lambda = 6h\radH^2$,  $ \log(1 + \frac{3T\radH^2}{\lambda}) \le \log(1+T)$. Thus, combining with the unary regret bound from \Cref{thm:semions}, 
	\begin{align*}
	\MoveDiff_T + \UnaReg_T \le \frac{\lambda D^2 }{\eta} +  \eta \Leff^2 d \left( \frac{1}{2} + h^2\radG\sqrt{\frac{2}{\kapnot}}\right) \log\left(1 + T\right),
	\end{align*}
	To conclude, we use $\eta = \frac{1}{\alpha}$, so that with $\lambda = 6h\radH^2$, yields $\frac{\lambda D^2 \radH^2}{\eta} = 3\alpha \radH^2 D^2$. Moreover, noting $h^2\radG\sqrt{\frac{1}{\kapnot}} \ge 1$\footnote{$\radG \ge 1$ by \Cref{defn:pol_reg_pars}, and $\kapnot \le 1$ by \Cref{defn:input_recov}}, we arrive at 
	\begin{align}
	\MemReg_T = \MoveDiff_T + \UnaReg_T   \le 3\alpha D^2\radH^2 +  \frac{2dh^2 \Leff^2  \radG}{\alpha \kappa^{\nicehalf}}  \log\left(1 + T\right). \label{eq:move_semi_final}
	\end{align}
	Recall that the above bound follows under the assumption that $\matz_s = \matz_1$ for $s \le 1$. Let us remove this assumption presently. Observe that the iterates $\matz_s$ for $s < 1$ \emph{do not} alter the trajector of future iterates $\matz_t$ for $t \ge 1$; they only appear in the memory regret bound via the with memory loss $F_t(\matz_{t:t-h})$. Thus, introducing
	$\matzch_t := \I(t \ge 1)\matz_t + \I(t < 1) \matz_1$, imposing the above assumption ($\matz_s = \matz_1$ for $s \le 1$) comes at the expense of regret at most
	\begin{align*}
	\sum_{t=1}^T |F_t(\matzch_{t:t-h}) - F(\matz_{t:t-h}) &= \sum_{t=1}^h |F_t(\matzch_{t:t-h}) - F(\matz_{t:t-h})|.
	\end{align*}
	With routine computations and the assumption that $L \ge 1$, each term in the above can be bounded by $\Leff \sum_{i=0}^h G^{[i]}\|\matY_{t-i}\matzch_t - \matz_t)\|_2 \le \Leff \radG \radyc\le \Leff^2$. This contributes a total addition cost of $h \Leff^2$, we which can be absored into the right-most term on \Cref{eq:move_semi_final} at the expense of replacing the constant $2$ with a factor of $3$.
	\end{proof}

\newcommand{\MemRegPlus}{\overline{\MemReg}}
\newcommand{\UnaRegPlus}{\overline{\UnaReg}}
\section{Regret with Quadratic Error Sensitivity (\Cref{thm:semions_unknown}) \label{sec:unknown}}
This section proves \Cref{thm:semions_unknown} and its generalizations. It is organized as follows:
\begin{itemize}
	\item In \Cref{ssec:neg_reg_bounds}, we two bounds which make explicit a certain negative regret term.  \Cref{thm:semions_unknown_clam} gives the generaliztion of \Cref{thm:semions_unknown} in the $\epsG^2 \ge \sqrt{T}$ regime (and allows for slight mis-specification of $\lambda$), and \Cref{thm:unknown_granular} proves a guarantee that degrades as $(T\epsG)^{2/3}$ for small $\epsG$. We prove \Cref{thm:semions_unknown_clam} from \Cref{thm:unknown_granular} in \Cref{sssec:clam_from_gran}.
	\item 	The remainder of the section is dedicated to the proof of \Cref{thm:unknown_granular}. This begins with  \Cref{ssec:unknown_reg_prelim}, which introduces relevant preliminaries. 
	\item \Cref{ssec:unknown_grad_error} provides a careful analysis of initial regret terms, and controlling the contribution of errors introduced by using the $\fhat_t$ sequence rather than $f_t$.
	\item \Cref{ssec:blocking} details our careful ``blocking argument'', which we use to offset the errors the terms $\sum_{t} \|\matY_t (\matz_t - \zst)\|$ from the gradients by a negative terms $\sum_t \|\matX_t (\matz_t - \zst)\|_2^2$ that arise in the regret analysis. 
	\item \Cref{ssec:unknown_concluding} concludes the proof of \Cref{thm:unknown_granular}, bounding first the movement cost and then tuning relevant parameters in the analysis.
\end{itemize}
All ommitted proofs are provided in \Cref{app:unknown_proofs}, organized into subsections and presented in numerical order.
\subsection{Bounds for Unknown Systems with Negative Regret \label{ssec:neg_reg_bounds}}
Here, we provide bounds which explicitly account for an appropriate negative regret term, scaling with $\sum_{t=1}^T \|\matY_{t}(\matz_t - \zst)\|^2$. Specifically, for any fixed comparator $\zst \in \cC$, our goal is to bound
\begin{align}
\MemRegPlus_T(\nu;\zst) &:= \sum_{t=1}^T F_t(\matz_{t:t-h}) - f_t(\zst) + \nu\sum_{t=1}^T \|\matY_{t}(\matz_t - \zst)\|^2 \label{eq:polregplus_def},
\end{align}
which gives a negative regret term by re-arranging $ \nu\sum_{t=1}^T \|\matY_{t}(\matz_t - \zst)\|^2$ to the right-hand side of the above display.
Note that we prove this bound for \emph{any} fixed comparator $\zst$, not just the ``best-in-hindsight'' comparator. Moreover, proving this bound for the best-in-hinsight comparator does not imply the bound for all $\zst \in \calC$, because the terms $\matdel_t$ in the negative-regret term differ as a function of $\zst$. 

To state our bound on $\MemRegPlus_T$, we recall the relevant parameter bounds:
\defnPolPars*

Our main result in this section is as follows. We also allow $\lambda$ to be slightly under-specified. This show's relative insensitivity to the selection of $\lambda$, and is also useful when porting the bound over to the control setting:
\begin{thmmod}{thm:semions_unknown}{a}\label{thm:semions_unknown_clam} Consider the setting of \Cref{thm:semions_unknown}, but where instead  $\lambda \in (\clam,1] \cdot (T\epsG^2 + h \radG^2)$ for $\clam \in (0,1]$. Equivalently, consider the setting of \Cref{thm:unknown_granular} below, but with the additional conditions $\epsG \ge \sqrt{T}$ and $\beta = L$. Then for any $\zst \in \calC$, 
\begin{align*}
\clam\MemRegPlus_T\left(\nustar;\zst\right) \lesssim   \log(1+\frac{T}{\clam})\left(\frac{C_1}{\alpha \kappa^{\nicehalf}} + C_2\right) \left(T\epsG^2 + h^2(\radG^2 + \radY)\right),
\end{align*}
where $C_1 := (1+\radY)\radG(h+d) \Leff^2$, $C_2:= (L^2 \cv^2/\alpha +  \alpha D^2)$, and $\nustar = \tfrac{\alpha \sqrt{\kappa}}{48(1+\radY)}$.
\end{thmmod}
\Cref{thm:semions_unknown} is an immediate conseuqnece of \Cref{thm:semions_unknown_clam}. We prove the above guarantee from a more statement, which allows for $\epsG^2 \le \sqrt{T}$ as well. 
\paragraph{Granular Guarantee for \semions{} with errors}
To state our generic guarantee, we specify the following constants:
\begin{definition}[Constants for Unknown $G$ Regret Analysis] \label{defn:unknown_constants} We define the constants 
We begin by establishing a slight generalization of \Cref{thm:semions_unknown_clam}, accomodating arbitrarily small. To start, define the constants
\begin{align}
\Cmid &:=  (1+ \tfrac{\beta^2}{L^2})(1+\radY) h \Leff^2 + \beta^2 \sqrt{\kappa} \cv^2 +  \alpha^2 \sqrt{\kappa} D^2 \label{eq:Cmid} \\
\Chigh &:= (1 + \radY)\Leff \radG^2 \radyc (h + d)+\alpha D^2 \label{eq:Chigh} \\
\Clow &:= (1+\radY)^2 \radG   h^2 \cdot d\Leff^2 \label{eq:Clow}.\\
\nustar &= \frac{\alpha \sqrt{\kappa}}{48(1+\radY)} \min\left\{4(1+\radY) (T \epsG^4)^{1/3}, 1\right\}  \label{eq:nustar}
\end{align}
Finally, we define a logarithmic factor
\begin{align}
\Lfactor := \log ( 1 + \radH^2 T/\lambda), \quad \text{with } \Lfactor \le \log(1 + T) \text{ for } \lambda \ge \radH^2\label{eq:Lfactor}.
\end{align} 
\end{definition}
Our more granular result is the following:
\begin{theorem}[Granular Regret Guarantee for \semions{} on an unknown system]\label{thm:unknown_granular} Consider running \semions{} on the empirical loss sequence ($\fhat_t,\matHhat_t$). Suppose that
\begin{itemize}
	\item The losses $\ell_t$ are $L$-subquadratic and $\alpha$-strongly convex for $L \ge 1 \vee \alpha$ (\Cref{asm:loss_reg}), and are $\beta$ smooth ($\nablatwo \ell_t \preceq  \beta I$)
	\item Suppose that $\|\Ghat -\Gst\|_{\loneop} \le \epsG$, $\Ghat^{[i]} = 0$ for $i > h$, and $\max_{t \ge 1}\|\matv_t - \matvhat_t\|_2 \le \cv \epsG$ for some constant $\cv \ge 0$.
	\item The step size is $\eta = 3/\alpha$, and $\lambda$ lies in $\lambda \in [\clam,1]\left(T\epsG^2 + (T \epsG)^{2/3} + h \radG^2\right)$ for some $\clam \in (0,1]$.
	\item All relevant quantities are bounded as in \Cref{defn:pol_reg_pars}
\end{itemize}
Then, the memory regret on the true loss sequence $(f_t,\matH_t)$ is bounded by 
 \begin{align*}
\clam\MemRegPlus_T\left(\nustar;\zst\right) &\lesssim  \Chigh (T\epsG)^{2/3}\Lfactor +  \frac{ \Cmid T\epsG^2}{\alpha \sqrt{\kappa}}  +\frac{\Clow\Lfactor}{\alpha \sqrt{\kappa}}  + \alpha h\radG^2\diamz^2.
\end{align*}
\end{theorem}

Observe that, when $\epsG^2 \ge \sqrt{T}$, the dominating term is $T \epsG^2$. However, for $\epsilon \le \sqrt{T}$, the term $(T\epsG)^{2/3}$ dominates. 

\subsubsection{Proof of \Cref{thm:semions_unknown_clam} from from \Cref{thm:unknown_granular} \label{sssec:clam_from_gran}}
\Cref{thm:semions_unknown_clam} follows from the granular \Cref{thm:unknown_granular} as a consequence of the following tedious simplifications. Recall that \Cref{thm:semions_unknown} adds the assumptions that $\epsG^2 \ge \sqrt{T}$, and $\beta = L$. This enables the following simplifications. First, since $(T \epsG^4)^{1/3} \ge 1$we can take $\nustar = \frac{\alpha \sqrt{\kappa}}{48(1+\radY)}$, which is precisely the value of $\nu$ used in the theorem. Second, we have $(T\epsG)^{2/3}/T\epsG^2 = 1/(T \epsG^4)^{1/3} \le 1$. This means that the choice of $\lambda =  \clam(T\epsG^2 + h \radG^2)$ is valid for \Cref{thm:unknown_granular}, up to rescaling $\clam$ by a factor of $2$. Thus, we have
 \begin{align*}
\clam\MemRegPlus_T\left(\tfrac{\alpha \sqrt{\kappa}}{48(1+\radY)};\zst\right) &\lesssim  \Chigh (T\epsG)^{2/3}\Lfactor +  \frac{ \Cmid (T\epsG^2)}{\alpha \sqrt{\kappa}}  +\frac{\Clow\Lfactor}{\alpha \sqrt{\kappa}}  + \alpha h\radG^2\diamz^2\\
 &\lesssim     \frac{\Lfactor}{\alpha \sqrt{\kappa}} \left( (T\epsG^2) (\Chigh \alpha \sqrt{\kappa} + \Cmid ) +   \Clow + \alpha^2 \sqrt{\kappa} h\radG^2\diamz^2\right).
\end{align*}
First, let us simplify $\Chigh \alpha \sqrt{\kappa} + \Cmid$. Using the simplifying condition $\beta = L$, and using $\radG \radyc \le \Leff$ (again, $L \ge 1$), we have
\begin{align*}
\Chigh \alpha \sqrt{\kappa} + \Cmid &\lesssim   (1+\radY) (h \Leff^2 + \Leff \radG^2 \radyc (h + d)) + L^2 \sqrt{\kappa} \cv^2 +  \alpha^2 \sqrt{\kappa} D^2\\
&\lesssim   (1+\radY)\radG(h+d) \Leff^2  + L^2 \sqrt{\kappa} \cv^2 +  \alpha^2 \sqrt{\kappa} D^2.
\end{align*}
Hence,
\begin{align*}
&(\Chigh \alpha \sqrt{\kappa} + \Cmid)T\epsG^2 + \Clow + \alpha^2 \sqrt{\kappa} h\radG^2\diamz^2\\
&\lesssim   (1+\radY)\radG(h+d) \Leff^2(T\epsG + (1+\radY)h^2)  + (L^2 \sqrt{\kappa} \cv^2 +  \alpha^2 \sqrt{\kappa} D^2)(T\epsG^2 + h \radG^2)\\
&\lesssim C_1(T\epsG^2 + (1+\radY)h^2)  + \alpha \sqrt{\kappa} C_2 (T\epsG^2 + h \radG^2)\\
&\lesssim (C_1 \alpha \sqrt{\kappa} C_2 )(T\epsG^2 + (1+\radY)h^2 + h \radG^2)\\
&\lesssim (C_1 \alpha \sqrt{\kappa} C_2 )(T\epsG^2 + h^2 \radY h^2 +  \radG^2)
\end{align*}
for  $C_1 := (1+\radY)\radG(h+d) \Leff^2$ and $C_2:= (L^2 \alpha^{-1} \cv^2 +  \alpha D^2)$. Thus we conclude that
 \begin{align*}
\MemRegPlus_T\left(\tfrac{\alpha \sqrt{\kappa}}{48(1+\radY)};\zst\right)&\lesssim \clam^{-1}\log(1+T)\left(\frac{C_1}{\alpha \kappa^{\nicehalf}} + C_2\right) \left(T\epsG^2 + h^2(\radG^2 + \radY)\right),
\end{align*}
as needed.

\subsection{Preliminaries for Proof of \Cref{thm:unknown_granular} \label{ssec:unknown_reg_prelim}}
\paragraph{Notation: } Let us begin by introducing relevant notation. Set $\nabla_t = \nabla f_t(\matz_t)$ to denote the gradients of the true counterfactual stationary counterfactual costs $f_t$, and let $\nabhat_t := \nabla \fhat_t(\matz_t)$ denote the gradient of their approximations. Analogously, define the matrices
 \begin{align*}
 \Lamhat_t = \lambda I + \sum_{t=1}^T\matHhat_t^\top \matHhat_t , \quad \Lambda_t = \lambda I + \sum_{t=1}^T\matH_t^\top \matH_t
 \end{align*} 
 For $t \le 1$, we will use the conventions $\Lambda_t = \Lambda_1$ and $\Lamhat_t = \Lamhat_1$. Throughout, we fix an \emph{arbitrary} comparator $\zst \in \cC$, and  further introduce the notation 
 \begin{align*}
 \matdel_t := \matz_t - \zst, \quad \err_t = \nabhat_t - \nabla_t
 \end{align*} to denote the difference of $\matz_t$ from the comparator, and difference between gradients, respectively. 

We recall that $\lambda,\eta$ are the algorithm parameters dictating the magnitude of the regularizer in $\Lambda_t$, and step size, respectively. We will also introduce a ``blocking parameter'' $\tau$, whose purposes is described at length in \Cref{ssec:blocking}. For simplicity, most of the proof will focuses on the unary  regret analogue of $\MemRegPlus_T$, defined as follows:
\begin{align}\label{eq:Errbar_reg}
		\UnaRegPlus_T(\nu;\zst) &:= \sum_{t=1}^T f_t(\matz_{t}) - f_t(\zst) + \nu\sum_{t=1}^T \|\matY_{t}\matdel_t\|^2, \quad \matdel_t := \matz_t - \zst,
\end{align}
We extend to memory regret in \Cref{ssec:unknown_concluding}. 
denote a logarithmic factor that will appear throughout.
\paragraph{Reduction $\matz_s = \matz_1$ for $s \le 1$:} As in the proof of \Cref{thm:semions_memory} in \Cref{ssec:ons_memory}, we can assume that $\matz_s = \matz_1$, at the expense of an additional factor of $h\Leff^2$ in the regret. This term is dominated by the factor of $\Clow \Lfactor$ in \Cref{thm:unknown_granular}, and can thus be disregarded in the following argument.


\subsection{Bounding  Regret in Terms of Error \label{ssec:unknown_grad_error}}
We begin with the following basic regret bound, controls the excess regret of using inexact gradients compared to standard bounds from online Newton. 
	\begin{lemma}\label{lem:unknown_regret_decomp} Let $\lambda \ge 1$. Then regret on measured on the $\fpred_t$ sequence is bounded by
		\begin{align*}
		\sum_{t=1}^T f_t(\matz_t) - f_t(\zst) &\le \sum_{t=1}^T \err_t^\top \matdel_t   + \frac{1}{2\eta}\sum_{t=1}^T(\|\matHhat_t\matdel_t\|^2 - \eta \alpha\|\matH_t\matdel_t\|^2) +\Reghat_T,
		\end{align*}
		where $\Reghat_T := \frac{ \eta  d \Leff^2 \Lfactor}{2} + \frac{\lambda\diamz^2}{2\eta}$ arises from the regret bound in \Cref{thm:semions}, and we recall $\Lfactor := \log ( e + T\radH^2)$.
	\end{lemma}
	Next, let us turn to bounding the mismatch arising from the terms $\sum_{t=1}^T(\|\matHhat_t\matdel_t\|^2 - \eta \alpha\|\matH_t\matdel_t\|^2)$: 
	\begin{lemma}\label{lem:Xhat_cancel} For $\eta \ge \frac{3}{\alpha}$, we have 
	 $\|\matHhat_t\matdel_t\|^2 - \eta \alpha\|\matH_t\matdel_t\|^2 \le -\|\matH_t\matdel_t\|^2 + 8\radyc^2 \epsG^2$. Hence, we have the regret bound:
	 \begin{align*}
		\sum_{t=1}^T f_t(\matz_t) - f_t(\zst) &\le \sum_{t=1}^T \err_t^\top \matdel_t   - \frac{1}{2\eta}\sum_{t=1}^T\|\matH_t\matdel_t\|^2 + \frac{4}{\eta}T\radyc^2 \epsG^2 + \Reghat_T.
		\end{align*}
	\end{lemma}

	\subsubsection{Controlling the error contributions} Next, we turn to bounding the contribution of the error in estimating the gradient:
\begin{lemma}\label{lem:grad_err} There exists  $g_{1,t} $ and $g_{2,t}$ with $\|g_{1,t}\|_2 \le \Leff$ and $\|g_{2,t}\| \le\beta\epsG(\cv + 2 \radyc) $ such that
\begin{align*}
\err_t  = (\matHhat_t - \matH_t)^\top g_{1,t} + \matH_t^\top \, g_{2,t}.
\end{align*}
\end{lemma}
By leveraring the specific structure of $\err_t$, we obtain:
\begin{lemma}\label{lem:Errbar_reg} For $\eta \ge \frac{3}{\alpha}$, the  following regret bound holds for all $\zst \in \calC$ and all $\nu > 0$:
\begin{align}\label{eq:Errbar_reg}
		\sum_{t=1}^T f_t(\matz_t) - f_t(\zst) &\le   \frac{1}{4\eta}\sum_{t=1}^T\left( \frac{\nu}{h+1} \sum_{i=0}^h\|\matY_{t-i}\matdel_t\|^2 - \|\matH_t\matdel_t\|^2\right) +  T\epsG^2 \cdot\,\Err(\nu)  + \Reghat_T,
		\end{align}
		where $\Err(\nu) := \left(\frac{\eta (h+1) \Leff^2 }{\nu}  + \eta\beta^2(\cv + 2 \radyc)^2 + \frac{4\radyc}{\eta}\right)$.
\end{lemma}
As a consequence, we have
\begin{align}
\UnaRegPlus_T\left(\frac{\nu}{4\eta};\zst\right) &\le   \frac{1}{4\eta}\sum_{t=1}^T\left(\nu  \|\matY_{t}\matdel_t\| + \frac{\nu}{h+1} \sum_{i=0}^h\|\matY_{t-i}\matdel_t\|^2 - \|\matH_t\matdel_t\|^2\right) +   T\epsG^2 \,\Err(\nu)  + \Reghat_T, \label{eq:regplus_consequence}
\end{align}

\subsection{The `blocking argument' \label{ssec:blocking}} A this stage of the proof, the main challenge is to show that for some small constant $\nu$, the terms $\|\matYhat_{t-i}\matdel_t\|^2$ in \Cref{eq:regplus_consequence} are offset by  $\|\matH_t\matdel_t\|^2$ on aggregate. We do this by dividing times into ``blocks'' of size $\tau = \Theta(\sqrt{T})$, centering at the terms $\matdel_t$ at times $t= k_j + 1$, for indices $k_j$ defined below. We define $\jmax := \floor{T/\tau}$ as the number of blocks. We then argue that, within any block
\begin{align}
\sum_{t \text{ in block }j} \|\matH_t\matdel_{t} \|^2 \gtrsim \sum_{i=0}^h\sum_{t \text{ in block }j}   \,\frac{1}{\nu}\|\matYhat_{t-i}\matdel_{t} \|^2 + \bigohst{1}\label{eq:wts_informal_cancel}
\end{align}
for appropriate $\nu$ and block size $\tau$. The reason we should expect an inequality of the above form to holds is that, from adapting \Cref{prop:covariance_lb}, we have the inequality that
\begin{align}
\sum_{t \text{ in block }j}\matH_t\matH^\top \succsim \sum_{t \text{ in block }j}\matY_t\matY_t^\top - \bigohst{1} \cdot I, \label{eq:wts_psd_lb}
\end{align}
However,  \Cref{eq:wts_psd_lb} does not directly imply a bound of the form \Cref{eq:wts_informal_cancel}, beacuse the vectors $\matdel_t$ differ for each $t$.  Instead, we `re-center' the $\matdel_t$ terms in the sum $\matdel_t = \matdel_{k_j+1}$, and at argue
\begin{align}
\sum_{t \text{ in block }j} \|\matH_t\matdel_{k_j+1} \|^2 \approx \sum_{i=0}^h\sum_{t \text{ in block }j}   \,\frac{1}{\nu}\|\matY_{t-i}\matdel_{k_j+1} \|^2 - \bigohst{1} \label{eq:wts_informal_cancel_kj}.
\end{align}
The above bound can be established from an estimate of the form \Cref{eq:wts_psd_lb}.  Summing this up across all $\jmax$ blocks, we see that the negative regret from the terms $\|\matH_t\matdel_{k_j+1} \|^2$ cancels the regret from the terms $\|\matY_{t-i}\matdel_{k_j+1} \|^2$. Accounting for all $\jmax = \Theta(T/\tau)$ blocksgives
\begin{align}
\sum_{j =1}^{\jmax} \sum_{t \text{ in block }j} \|\matH_t\matdel_{k_j+1} \|^2 \approx \sum_{j =1}^{\jmax}\sum_{i=0}^h\sum_{t \text{ in block }j}   \,\frac{1}{\nu}\|\matY_{t-i}\matdel_{k_j+1} \|^2 - \bigohst{T/\tau} \label{eq:wts_informal_cancel_kj}.
\end{align}
incurring an additive factor of $T/\tau$, favoring larger block sizes $\tau$.

But, we must also argue that not too much is lost by approximating the statement \Cref{eq:wts_informal_cancel} with the centered analogue \Cref{eq:wts_informal_cancel_kj}. The cost of recentering will  ultimatels as $\bigohst{\tau}$, so trading off $\tau$ with the bound of yields $\sqrt{T}$ regret in the final bound. 

Interestingly, the cost of recentering is intimately tied to bounding the movement of the iterates $\matz_t$. Thus, we find that the same properties that allow $\semions$ to attain logarithmic regret for the known system case are also indispensible in achieving low sensitivity to error in the unknown system case.

\subsubsection{Formalizing the blocking argument}
Formally, the cost of the above re-centering argument is captured by the following lemma:
 \begin{lemma}[Blocking Argument]\label{lem:unknown_block} Given parameter $\tau \in \N$, and introduce the $k_{j} =  \tau (j-1)$, and $\jmax := \floor{T/\tau}$. Then, with the understanding that $\matz_{s} = 0$ for $s \le 1$, the following holds for all $i \in [h]$, 
 \begin{align*}
 &\sum_{t=1}^T \|\matY_{t-i}\matdel_t\|_2^2 \le 4\tau \radyc + \sum_{j=1}^{\jmax}\sum_{s=1}^{\tau} \|\matY_{k_j + s  - i}\matdel_{k_j+1}\|_2^2   + 4\radyc \sum_{t=1}^T\sum_{s = 0}^{\tau -1} \|\matY_{t  - i}(\matz_{t-s} - \matz_{t - s-1 })\|_2.\\
 &\sum_{t=1}^T \|\matH_{i-h}\matdel_t\|_2^2  \ge  \sum_{j=1}^{\jmax}\sum_{s=1}^{\tau} \|\matH_{k_j + s }\matdel_{k_j+1}\|_2^2   - 4\radyc  \radG \sum_{t=1}^T\sum_{s = 0}^{\tau -1} \|\matH_{t }(\matz_{t-s} - \matz_{t - s-1 })\|_2,
 \end{align*}
 \end{lemma}
 Notice that, while the left-hand side depends on $\matdel_t$, the right hand side is `centered' at $\matdel_{k_j + 1}$ for $j \in [\jmax]$, at the expense of movement penalties on $\matz_{t-s} - \matz_{t-s-1}$. Let us re-write the above bound to give a useful regret decomposition. We introduce bounding terms $\Regmovyi$ and $\Regmovh$ for the movement costs above associated with the centering argument, and $\Rcancel$ associated with the offsetting argument described above. Formally,
 \begin{align*}
 \Regmovyi &:= \sum_{t=1}^T\sum_{s = 0}^{\tau -1} \|\matY_{t  - i}(\matz_{t-s} - \matz_{t - s-1 })\|_2. \\
 \Regmovh &:= \sum_{t=1}^T\sum_{s = 0}^{\tau -1} \|\matH_{t }(\matz_{t-s} - \matz_{t - s-1 })\|_2\\
 \Rcancel &:= \sum_{j=1}^{\jmax}\sum_{s=1}^{\tau} \left(\sum_{i=0}^h\left(\nu\left(\frac{1}{h+1} + \I_{i = 0}\right)\|\matY_{k_j + s  - i}\matdel_{k_j+1}\|_2^2\right)  - \|\matH_{k_j+s}\matdel_{k_j+1}\|_2^2 \right).
 \end{align*}
 Then, from \Cref{lem:unknown_block}, the upper bound on $\UnaRegPlus_T$ in \Cref{eq:regplus_consequence}  can be expressed as 
 \begin{align}
 \UnaRegPlus_T\left(\frac{\nu}{4\eta};\zst\right) \le \frac{1}{4\eta} \Regblock +  T\epsG^2 \,\Err(\nu)  + \Reghat_T, \label{eq:regplus_block}
 \end{align} 
 where we define and bound
 \begin{align}
 \Regblock &:= \sum_{t=1}^T\left(\nu  \|\matY_{t}\matdel_t\| + \frac{\nu}{h+1} \sum_{i=0}^h\|\matY_{t-i}\matdel_t\|^2 - \|\matH_t\matdel_t\|^2\right) \nonumber\\
 &\qquad\le 8\tau \cdot \nu \radyc + 8\nu\radyc\left(\max_{i \in [h]}  \Regmovyi\right) + 4\radyc  \radG \cdot \Regmovh + \Rcancel. \label{eq:main_blocking_eq}
 \end{align}
 Thus, we shall conclude our argument by developing bounds on $\Regmovyi$, $\Regmovh$ and $\Rcancel$.

\paragraph{Movement Costs}
Via \Cref{eq:main_blocking_eq} and the definitions of $\Regmovyi$ and $\Regmovh$, the cost of the re-centering argument is given by a movement costs, which we bound presently. Since the movement of the algorithm are small in the norms induced by the preconditioning matrices $\Lamhat$, our main argument invokes steps of the form
 \begin{align*}
 \|\matH_{t }(\matz_{t-s} - \matz_{t - s-1 })\|_2 \le  \frac{\|\matH_t^\top \Lamhat_{t-s-1}^{-1}\matH_t\|_{\op}^2}{2} + \frac{\|(\matz_{t-s} - \matz_{t - s-1 })\|_{\Lamhat_{t-s-1}}^{2}}{2},
 \end{align*}
 much like the regret analysis in the known system case. Moreover, the contribuitons of the $\|(\matz_{t-s} - \matz_{t - s-1 })\|_{\Lamhat_{t-s-1}}^{2}$ can be bounded via an application of the log-det potential argument, as in \Cref{thm:semions}. 

However, we observe that the conditioning of the relevant movement costs is in terms of the $\Lamhat$ matrix.  To bound terms $\|\matH_t^\top \Lamhat_{t-s-1}^{-1}\matH_t\|_{\op}^2$, we will need to relate the matrices $\Lamhat_{t-s-1}$, constructed based on the estimated sequence $(\matHhat_t)$, and with delays up to $(s+1) = \tau$, to the matrixes $\Lambda_t$, based on $(\matH_t)$ and current time $t$. This is accomplished by the following lemma:
 \begin{lemma}\label{lem:ctau}  For $\clam \in (0,1]$, set $\ccond(\tau) := 2(1+\radY) + 2\clam^{\minhalf}\radY\sqrt{ \frac{\tau \radG^2}{\lambda}}$. Then, for $\lambda \ge \clam T\epsG^2$, we have that for all $t \in [T]$, 
 \begin{align*}
\Lamhat_{t-\tau}^{-1} \preceq \ccond(\tau)^2 \Lambda_t^{-1},
 \end{align*}
  where we adopt the convention $\Lamhat_s = \Lamhat_1$ and $\Lambda_s = \Lambda_1$  for $s \le 1$.
 \end{lemma}
 For our scalings of $\tau$ and $\lambda$, $\ccond$ will be roughly constant in magnitude. With the above lemma in hand, we show that the movement terms from the blocking argument scale proportionally to $\tau$.
 \begin{lemma}\label{lem:blocking_movement_bound} Recall the logarithmic factor $\Lfactor := \log ( e + T\radH^2)$. If $\lambda$ is chosen such that $\lambda \ge \frac{\clam}{h}T\epsG^2 + \clam h\radG^2$, then the movement terms admit the following bounds for $i \in\{0,\dots,h\}$:
 \begin{align*}
 \Regmovyi := \sum_{t=1}^T\sum_{s = 0}^{\tau -1} \|\matY_{t  - i}(\matz_{t-s} - \matz_{t - s-1 })\|_2 &\le   \tau\ccond  \clam^{\minhalf} \cdot d  \Leff  \sqrt{\frac{2(1+10\radY^2)}{\kappa }} \Lfactor. \\
  \Regmovh := \sum_{t=1}^T\sum_{s = 0}^{\tau -1} \|\matH_{t }(\matz_{t-s} - \matz_{t - s-1 })\|_2 &\le    \tau   \ccond \clam^{\minhalf} \cdot  d\Leff \Lfactor
 \end{align*}
 \end{lemma}

\paragraph{Cancellation within blocks} Next, let us argue that the term $\Rcancel$ is small, which leverages cancellation within blocks. As per the proof sketch at the beginning of the section, we show that the terms $\|\matY_{k_j + s  - i}\matdel_{k_j+1}\|_2^2$ offset the terms $\|\matH_{k_j+s}\matdel_{k_j+1}\|_2^2$ up to a $\bigohst{1}$ factor for each $j$, incuring an error scaling as $\jmax \approx T/\tau$ (thereby inducing a trade-off on the parameter $\tau$):
 \begin{lemma}\label{lem:cancelling_neg_reg}  For $\nu \le \frac{\kappa}{4}$, we have 
\begin{align*}
\Rcancel \le \frac{20T }{\tau} \cdot \nu h \radG^2 \radyc^2 + 5 T \epsG^2  \cdot \kappa  \radyc^2.
\end{align*}
\end{lemma}
\subsubsection{Summarizing the blocking argument}

 Grouping all the terms that have emerged thus far, we summarize the current state of our argument in the following lemma:
 \begin{lemma}\label{lem:Regbar_Bound} Assuming $\Leff \ge 1$, $\nu \le \frac{\sqrt{\kappa}}{4(1+\radY)}$, and $\lambda \ge \clam(\frac{1}{h}T\epsG^2 + h\radG^2 + \tau)$, we have that for all $\zst \in \calC$,
\begin{align*}
\clam\UnaRegPlus_T\left(\frac{\nu}{4\eta};\zst\right) &\lesssim \frac{T \epsG^2}{\alpha} \cdot \left(\frac{ h \Leff^2 }{\nu}  + \beta^2 (\cv^2 +  \radyc + \radyc^2 ) \right) + \Reghat_T.\\
&\quad+ \frac{T \nu}{\tau} \cdot \left(\alpha  h \radG^2 \radyc^2\right)  + \tau \cdot \left(\alpha  (1 + \radY) \radyc  \radG^2  \cdot d\Leff \Lfactor \right),
\end{align*}
 \end{lemma}
 Let us take stock of what we have so far. The bound $\UnaRegPlus_T(\nu/4\eta;\zst)$ has four components: 
 \begin{itemize}
 	\item $\Reghat_T $, which accounts for the regret on the $\fhat_t$ sequence.
 	\item A term scaling with $T\epsG^2 $, which accounts for the sensitivity to error. This term also involves the offset $\nu$.
 	\item A term scaling as $\frac{T \nu}{\tau}$, yielding a penalty for the number of blocks in the blocking argument.
 	\item A term scaling as linearly in $\tau$, arising from the movement costs from the recentering argument.
 \end{itemize}
 The final regret bound will follow from carefully trading off the parameters $\nu$ and $\tau$ in the analysis, and from setting $\lambda$ appropriately. Before continuing, we first adress with ``with-memory'' portion of the bound, passing from unary regret to memory regret.

\subsection{Concluding the Bound \label{ssec:unknown_concluding}}
Before concluding the bound, we need to bound the movment cost that appears:
\begin{lemma}[Movement Cost: Unknown System]\label{lem:unknown_move} Under the conditions of \Cref{lem:Regbar_Bound},
\begin{align*}
\MoveDiff_T  := \sum_{t=1}^T F_t(\matz_{t:t-h}) - f_t(\matz_t) \le 9\eta\kappa^{\minhalf}(1+\radY)^2 \radG   h^2 \cdot d\Leff^2\Lfactor
\end{align*}
\end{lemma}
We are now ready to prove our main theorem:
\begin{proof}[Proof of \Cref{thm:unknown_granular}] 
Let us begin by unpacking
\begin{align*}
\Reghat_T + \MoveDiff_T  &\le 9\eta\kappa^{\minhalf}(1+\radY)^2 \radG   h^2 \cdot d\Leff^2\Lfactor + \frac{ \eta  d \Leff^2 \Lfactor}{2} + \frac{\lambda\diamz^2}{2\eta}\\
&\lesssim \frac{1}{\alpha \sqrt{\kappa}}(1+\radY)^2 \radG   h^2 \cdot d\Leff^2\Lfactor  + \alpha \lambda\diamz^2,
\end{align*}
where we use $\eta = \frac{3}{\alpha}$. Thus, from \Cref{lem:Regbar_Bound}, the term $\MemRegPlus_T$ defined in \Cref{eq:polregplus_def} satisfies the following for any $\zst \in \calC$, provided that the conditions of \Cref{lem:Regbar_Bound} hold:
 \begin{align*}
\clam \MemRegPlus_T\left(\frac{\nu}{4\eta};\zst\right) &\le \clam\UnaRegPlus_T\left(\frac{\nu}{4\eta};\zst\right)  +\MoveDiff_T \\
&\lesssim \frac{T \epsG^2}{\alpha} \cdot \left(\frac{ h \Leff^2 }{\nu}  + \beta^2 (\cv^2 +  \radyc + \radyc^2 ) \right) + \frac{1}{\alpha \sqrt{\kappa}}(1+\radY)^2 \radG   h^2 \cdot d\Leff^2\Lfactor  + \alpha \lambda\diamz^2\\
&\quad+ \frac{T \nu}{\tau} \cdot \left(\alpha  h \radG^2 \radyc^2\right)  + \tau \cdot \left(\alpha  (1 + \radY) \radyc  \radG^2  \cdot d\Leff \Lfactor \right),
\end{align*}
where above we use $\clam \le 1$.
Let us now specialize parameters. As per our theorem, we take
\begin{align*}
\lambda = \clam \left(T\epsG^2 + c(T \epsG)^{2/3} + h \radG^2\right),\quad
\tau =  (T \epsG)^{2/3}, \quad \clam \in (0,1)
\end{align*}
which we verify satisfies the condition on $\lambda$ placed by \Cref{lem:Regbar_Bound}. For this choice of parameters, we have
 \begin{align*}
\MemRegPlus_T\left(\frac{\nu}{4\eta};\zst\right) &\lesssim \frac{1}{\alpha \sqrt{\kappa}}(1+\radY)^2 \radG   h^2 \cdot d\Leff^2\Lfactor  + \alpha h \radG^2\diamz^2\\
&+  \frac{T\epsG^2}{\alpha}\cdot \left(\beta^2 (\cv^2 +  \radyc + \radyc^2 ) + \alpha^2 D^2 \right)
\\
&+ \alpha (T\epsG)^{2/3} \cdot \left(D^2 + (1 + \radY) \radyc  \radG^2  \cdot d\Leff \Lfactor\right)\\
&+\frac{T \epsG^2}{\alpha} \cdot \frac{ h \Leff^2 }{\nu}  + \frac{T \nu}{\tau} \cdot \left(\alpha  h \radG^2 \radyc^2\right) .
\end{align*}
Next, let's tune $\nu$. Define $\nu_0 :=  \frac{\sqrt{\kappa}}{4(1+\radY)}$ to denote the upper bound on $\nu$ imposed by \Cref{lem:Regbar_Bound}. Moreover, let $\nu_1$ denote the value of $\nu$ that minimizes the upper bound above, namely
\begin{align*}
\nu_1 = \left(\frac{T \epsG^2}{\alpha} \cdot h \Leff^2\right)^{1/2} \cdot  \left(\frac{T}{\tau} \cdot \alpha  h \radG^2 \radyc^2 \right)^{-1/2}.
\end{align*}
\newcommand{\nubar}{\bar{\nu}}
We set $\nubar = \min\{\nu_0,\nu_1\}$. For this value, we have that
\begin{align*}
\frac{T \epsG^2}{\alpha} \cdot \frac{ h \Leff^2 }{\nubar}  + \frac{T \nubar}{\tau} \cdot \left(\alpha  h \radG^2 \radyc^2\right) &\le \frac{T \epsG^2}{\alpha} \cdot \frac{ h \Leff^2 }{\nu_0} + \frac{T \epsG^2}{\alpha} \cdot \frac{ h \Leff^2 }{\nu_1}  + \frac{T \nu_1}{\tau} \cdot \left(\alpha  h \radG^2 \radyc^2\right)\\
&\le \frac{T \epsG^2}{\alpha} \cdot \frac{ h \Leff^2 }{\nu_0} + 2\sqrt{\frac{T^2 \epsG^2 h^2 \Leff^2 \radG^2 \radyc^2}{\tau} \cdot }\\
&\le \frac{T \epsG^2}{\alpha} \cdot \frac{ h \Leff^2 }{\nu_0} + 2(T\epsG)^{2/3}\sqrt{h} \Leff \radG \radyc\\
&\lesssim \frac{T \epsG^2}{\alpha \sqrt{\kappa}} \cdot (1+\radY) h \Leff^2  + (T\epsG)^{2/3}h \Leff \radG \radyc.
\end{align*}
Combining with the above,
 \begin{align*}
\MemRegPlus_T(\frac{\nubar}{4\alpha};\zst) &\lesssim   \frac{T\epsG^2}{\alpha \sqrt{\kappa}} \cdot \underbrace{\left( (1+\radY) h \Leff^2 +
\beta^2 \sqrt{\kappa} (\cv^2 +  \radyc + \radyc^2 ) + \alpha^2 \sqrt{\kappa} D^2 \right)}_{:=\Cmid'}.
\\
&+  (T\epsG)^{2/3} \Lfactor \cdot \underbrace{\left(h\Leff \radG \radyc + \alpha D^2 + \alpha(1 + \radY) \radyc  \radG^2  \cdot d\Leff \right)}_{:=\Chigh'}\\
&+\frac{\Lfactor}{\alpha \sqrt{\kappa}}\underbrace{(1+\radY)^2 \radG   h^2 \cdot d\Leff^2}_{:=\Clow}  + \alpha \lambda\diamz^2
\end{align*}
where we use $\Chigh',\Cmid'$ as intermediate constants that we simplify as follows. Recalling the 
\begin{align*}
\Cmid' &=  (1+\radY) h \Leff^2 +
\beta^2 \sqrt{\kappa} (\cv^2 +  \radyc + \radyc^2 ) + \alpha^2 \sqrt{\kappa} D^2 \\
&\quad\le (1+ \tfrac{\beta^2}{L^2})(1+\radY) h \Leff^2 +
\beta^2 \sqrt{\kappa} \cv^2 +  \alpha^2 \sqrt{\kappa} D^2 := \Cmid\\
\Chigh' &= h\Leff \radG \radyc + \alpha D^2 + \alpha(1 + \radY) \radyc  \radG^2  \cdot d\Leff \\
&\le (1 + \radY)\Leff \radG^2 \radyc (h + d)+\alpha D^2 := \Chigh
\end{align*}
Note that the constant $\Chigh,\Clow,\Cmid$ coincided with those in \Cref{defn:unknown_constants}. Thus, writing our regret bound compactly, we have
 \begin{align*}
\MemRegPlus_T(\frac{\nubar}{4\alpha};\zst) &\lesssim  \Chigh (T\epsG)^{2/3}\Lfactor +  \frac{ \Cmid (T\epsG^2)}{\alpha \sqrt{\kappa}}  +\frac{\Clow\Lfactor}{\alpha \sqrt{\kappa}}  + \alpha \lambda\diamz^2.
\end{align*}
Finally, let us expose $\nubar$. Recall we set $\nubar = \min\{\nu_0,\nu_1\}$, with $\nu_0 = \frac{\sqrt{\kappa}}{4(1+\radY)}$, and
\begin{align*}
\nu_1 &= \left(\frac{T \epsG^2}{\alpha} \cdot h \Leff^2\right)^{1/2} \cdot  \left(\frac{T}{\tau} \cdot \alpha  h \radG^2 \radyc^2 \right)^{-1/2}.\\
&= \frac{\Leff \epsG  \sqrt{\tau}}{\alpha \radG \radyc}  = \frac{\Leff (T \epsG^4)^{1/3}}{\alpha \radG \radyc},
\end{align*}
finally yielding 
\begin{align*}
\nubar = \min\left\{\frac{\Leff (T \epsG^4)^{1/3}}{\alpha \radG \radyc}, \frac{\sqrt{\kappa}}{4(1+\radY)}\right\},
\end{align*}
To conclude, we paramaterize $\nubar' = \frac{\nubar}{4\eta}$. Since $\eta = \frac{3}{\alpha}$, we take
\begin{align*}
\nubar' &= \frac{\alpha \sqrt{\kappa}}{48(1+\radY)} \min\left\{\frac{4(1+\radY)\Leff (T \epsG^4)^{1/3}}{ \alpha \sqrt{\kappa} \radG \radyc}, 1\right\}\\
&\ge \frac{\alpha \sqrt{\kappa}}{48(1+\radY)} \min\left\{\frac{4(1+\radY)\Leff (T \epsG^4)^{1/3}}{  \radG \radyc}, 1\right\}\\
&\ge \frac{\alpha \sqrt{\kappa}}{48(1+\radY)} \min\left\{4(1+\radY) (T \epsG^4)^{1/3}, 1\right\} := \nustar
\end{align*}
where in the last line we use $L \ge 1$ to bound $\Leff \ge \radG\radyc$.  Thus, taking $\nustar$ to be the above lower bound on $\nubar'$ concludes. 
\end{proof}

	\section{Conclusion \label{app:conclusion}}


In this work, we demonstrate that fast rates for online control, and in particular, the optimal $\sqrt{T}$ regret rate  \cite{simchowitz2020naive} for the online LQR setting, are achievable with non-stochastic noise. 

\paragraph{Future Work}
It is an interesting direction for future research to determine if non-degenerate observation noise can be used to attain {polylogarithmic} regret for unknown systems in the \emph{semi-stochastic} regime considered by \citet{simchowitz2020improper}. This regime interpolates between purely stochastic non-degenerate noise, and arbitrary adversarial noise considered in this setting. 

Furthermore, it may be possible that $\sqrt{T}$ regret for unknown systems is attainable even \emph{without} strongly convex cost function; currently, the state of the art in this setting is $T^{2/3}$ \cite{simchowitz2020improper,hazan2019nonstochastic}.

Finally, we hope future work will take up a more ambitious direction of inquiry, investigating whether these techniques can be applied beyond linear time invariant systems with bound noise. Such directions understanding slowly-varying dynamics, robustness to non-linearities, and model-predictive control.  

\paragraph{Open Question: System Stability and Fast Rates} Lastly, an open question that remains is the extent to which stability of the dynamics affects the extent to which stochastic control is easier than non-stochastic. For example, the guarantees in \citet{lale2020logarithmic} assume that the dynamics of the system are internally stable, which presumbaly simplifies the system identification procedure. On the other hand, our work assumes only that our system can be stabilized by a static feedback controller, which holds without loss of generality for fully observed systems. 

As discussed in \Cref{app:general}, there are many partially observed systems which cannot be stabilized even by static feedback, but can be stabilized by more general linear control laws. For such systems, our guarantees do extend, but under the opaque technical assumption on the dynamics induced by this more general stabilizing controller have the invertibility property of \Cref{defn:input_recov}. Recall that for the simple case of static feedback, this invertible property is proven to hold in \Cref{lem:kap_bound_K}.

On the other hand, \citet{simchowitz2020improper} show that for semi-stochastic disturbances (disturbances with a non-degenerate stochastic component), one can still achieve fast rates  for \emph{any} any linear stabilizing scheme.\footnote{Intuitively, this is because with (semi-stochastic) noise, one can replace the infinite-horizon invertibility condition $\kappa(G)$ of \Cref{defn:input_recov} with a finite-horizon analogue, $\kappa_{m,h}(G)$. It is shown that this analogue decays at most polynomially in $m,h$, even though $\kappa(G)$ may be zero. This translates into a polynomial dependence on $m,h$ in the final bound, which contributes only logarithmic factors for the typical choice $m,h = \BigOh{\log T}$.} This seems to suggest that for controller parametrizations based on more powerful stabilizing controllers, stochasticity may in fact be beneficial.   It is an interesting direction for future work to understand whether these more general stabilizing controllers admit fast regret rates for non-stochastic control.

\section*{Acknowledgements}
MS is generously supported by an Open Philanthropy AI Fellowship. MS also thanks Dylan Foster and Elad Hazan for their helpful discussions.

\clearpage

\iftoggle{nips}
{
\section*{Broader Impact}
Though this paper is primarily theoretical in nature, we believe that the non-stochastic control setting is an important one. Historically, one of the greatest strengths of control theory is its ability to provide robust, mathematical guarantees on performance quality. As control theory merges with recent developments in reinforcement learning,  we see novel applications in domains with little room for error: control algorithms in automated transportation, server cooling, and industrial robotics can wreak havoc when gone awry. These tasks may range from easy-to-model to wildly unpredictable, and purely stochastic models may not suffice to capture the full extent of the  uncertainty in the task. On the other hand, traditional techniques from robust control may be overly conservative, and deem certain tasks infeasible from the outset. 

While far from perfect, we believe that the non-stochastic control model inches us closer towards robustness to modeling assumptions,  without succumbing to excessive pessimism. As such, we find it important to understand what, if any, challenges this more accomodating model poses to data-driven control. We hope that our central theoretical contribution - demonstrating  that the uncertainty in the  noise model is in fact not a significant barrier to achieving near optimal performance - may encourage practioners not to abandon considerations of robustness for fear of sacrificing performance. But there is still a long road ahead, and we recognize that non-stochastic control does not capture many important senses of robustness in the decades-old control literature.  We also recognize that there are, and will continue to be, instances when performance \emph{must} be sacrificed for robustness, and hope our work will contribute a small but helpful part in a broader dialogue about the tensions between safety and performance in data-driven control. 


}{}

\clearpage
\bibliographystyle{plainnat}
\bibliography{main}
\clearpage

\tableofcontents
\addappheadtotoc
\appendix
\newpage
\section{Organization of the Appendix and Notation\label{app:organization}}
The appendix is organized as follows:
\begin{itemize}
	\nipstogtrue{
	\item \Cref{sec:further_discussion} provides further discussion, describing how our work serves to characterize the relative difficulty of adversarial noise in online control settings when compared to stochastic. 
	}
	\item \Cref{ssec:app_classical} provides an in-depth comparison with the classic LQR and LQG settings\iftoggle{nips}{, together with an in-depth discussion in \Cref{app:comparison_discussion} about the extent to which stochasticity affects the optimal regret rates in online control.}{.}
	\item \Cref{app:general} provides the full statement of the algorithm \drcons{} algorithm for the known and unknown settings, and describes the more general \drconsdyn{} algorithm for use with a non-static internal controller. 
	\item \Cref{sec:control_proofs_ful} provides full statements and proofs of our main regret bounds for the control setting, \Cref{thm:known_control,thm:unknown_control}. In particular, we provide the full analogues with the full parameter settings required for the regret bounds, \Cref{thm:unknown_granular,thm:drconsdyn_unknown}. We also provide generalizations of our \drconsdyn{} algorithm , \Cref{thm:drconsdyn_known,thm:drconsdyn_unknown}.
	\iftoggle{nips}
	{
	\item \Cref{sec:known} gives the full proof of the logarithmic regret bound for \semions{}, \Cref{thm:semions_memory}, and \Cref{app:known_proofs} provides the omitted proofs.
	\item \Cref{sec:unknown} gives the full proof of the quadratic error sensitivity of \semions{}, \Cref{thm:semions_unknown}, and \Cref{app:unknown_proofs} provides the omitted proofs.
	}
	{
	\item \Cref{app:known_proofs} provides the omitted proofs from  \Cref{sec:known} , regarding the regret of logarithmic regret for \semions{}.
	\item \Cref{app:unknown_proofs} provides the omitted proofs from  \Cref{sec:known}, regarding the quadratic error sensitivity of \semions{}.
	}
	\item 
\Cref{sec:lower_bounds} gives the proof of \Cref{thm:Regmu_lb} , and then demonstrates the standard online Newton step matches the tradeoff (\Cref{thm:ons_move})
\end{itemize}

\paragraph{Notation:} We use $a = \BigOh{b}$ and $a \lesssim b$ interchangably to denote that $a \le C b$, where $C$ is a universal constant independent of problem parameters. We also use $a \vee b$ to denote $\max\{a,b\}$, and $a \wedge b$ to denote $\min\{a,b\}$. Notation relevant to the control problem is reviewed where-necessary in \Cref{ssec:app_classical,app:general}. In what follows, we review notation relevant to the generic analyses of \semions{}.

In \semions{}, we have the with-memory loss functions
\begin{align*}
F_t(z_t,\dots,z_{t-h}) := \ell_t(\matv_t + \sum_{i=0}^h G^{[i]} \matY_{t-i} z_{t-i}),
\end{align*}
and their unary specializations
\begin{align*}
f_t(z) := F_t(z,\dots,z) = \ell_t(\matv_t + \matH_t z), \quad \matH_t := \sum_{i=0}^h G^{[i]} \matY_{t-i}.
\end{align*}
Here the losses $\ell_t,\matv_t,\matY_t$ change at each round, and $G = (G^{[i]})_{i \ge 0}$ is regarded as part of an infinite-length Markov operator which is fixed throughout. 

For unknown systems, we are use approximate losses, where $\matvhat_t \approx \matv_t$, $\Ghat \approx G$,
\begin{align*}
\fhat_t(z) := \Fhat_t(z,\dots,z) = \ell_t(\matvhat_t + \matHhat_t z), \quad \matHhat_t := \sum_{i=0}^h \Ghat^{[i]} \matY_{t-i}.
\end{align*}
Throughout, we use bold $\matz_t$ to refer to the iterates of the algorithm.

\part{Appendices for Control}

\newcommand{\Kst}{K_{\star}}
\newcommand{\piKst}{\pi^{\Kst}}

\section{Past Work and Classical Settings \label{ssec:app_classical}}

In this section, we describe in detail how our non-stochastic control setting compares with other control settings considered in the literature. At the end of the section, we conclude with a more thorough discussion of the separations (and lack thereof) between stochastic and non-stochastic control. Recall that our linear system is described by the dynamic equations
\begin{align}
\matx_{t+1} = \Ast \matx_t + \Bst \matu_t + \matw_t, \quad 
 \maty_t = \Cst \matx_t + \mate_t, \label{eq:LDS_system_app}
\end{align}
Of special interest are the \emph{fully observed} settings, where $\maty_t = \matx_t$. We may also imagine an intermediate,  \emph{full-rank observation} setting, where $\dimy = \dimx$, and $\sigma_{\min}(\Cst) > 0$. Note that this latter setting allows for observation noise $\mate_t$, while the former does not. Finally, in full generality $\Cst \in \R^{\dimy \dimx}$ may have rank $\rank(\Cst) < \dimx$, and thus states cannot in general be recovered from observations.



\subsection{Online LQR \label{ssec:app_lqr}}
The linear quadratic regularity, or LQR, corresponds to the setting where the state is fully observed $\matx_t = \maty_t$, and the noise $\matw_t $ is selected from a mean-zero, light-tailed stochastic process - typically i.i.d. Gaussian. Crucially, the noise $\matw_t$ is assumed to have some non-degenerate covariance: e.g., $\matw_t \iidsim \calN(0,\Sigma)$ for some $\Sigma \succ 0 $. One then considers quadratic cost functions which do not vary with time:
\begin{align*}
\ell_t(x,u) = \ell(x,u) = x^\top R x + u^\top Q u,
\end{align*}
where $R$ and $Q$ are positive definite matrices. In particular, $\ell(x,u)$ is a strong-convex function, and thus the LQR setting is subsumed by our present work.

For the above setting, the optimal control policy (in the limit as $T \to \infty$) is described by a static feedback law $\matu_t = K_{\star} \matx_t$, where $K_{\star}$ solves the Discrete Algebraic Riccati Euqation, or DARE; we denote the corresponding control policy $\pi^{\Kst}$. Note that this is in fact the optimal \emph{unrestricted} control policy (say, over any policy which executes inputs as functions of present and past observations), despite having the simple static feedback form. 

Results for online LQR consider a regret benchmark typically considered performance with respect to this benchmark (see e.g. \cite{abbasi2011regret,dean2018regret,mania2019certainty,cohen2019learning})
\begin{align*}
\overline{R}_T(\Alg) := J_T(\Alg) - T \lim_{n \to \infty}\frac{1}{n}\Exp_{\matw}[J_n(\piKst)]
\end{align*}
where the righthand term is the \emph{infinite horizon} average cost induced by placing the optimal control law $\Kst$. One can show (e.g. \cite{simchowitz2020naive})  $\Exp_{\matw}[J_n(\piKst)]$ is increasing in $n$. Thus, by Jensen's inequality, it holds that for any $\Pi \subset \Pildc$ containing $\pi^{\Kst}$, 
\begin{align*}
\Exp_{\matw}[\overline{R}_T(\Alg)] &= \Exp_\matw [J_T(\Alg)] - T \lim_{n \to \infty}\frac{1}{n}\Exp_{\matw}[J_n(\piKst)] \\
&\le J_T(\Alg) - \Exp_{\matw}[J_T(\piKst)] \\
&= \Exp_\matw [J_T(\Alg)]- \inf_{\pi \in \Pi} \Exp_{\matw}[J_T(\pi)] \\
&\le \Exp_\matw [J_T(\Alg)] -  \Exp_{\matw}\inf_{\pi \in \Pi}J_T(\pi) \\
&\le \Exp_{\matw}[J_T(\Alg) -  \inf_{\pi \in \Pi}J_T(\pi) ] := \Exp_{\matw}[\Regret_T(\Alg;\Pi)],
\end{align*}
where $\Regret_T$ is our non-stochastic benchmark.
Hence, we find that, in expectation, the standard benchmark for online LQR is weaker than ours. Nevertheless, the two benchmark typically concide up to lower order terms due to martingale concentration. Observe however a key conceptual difference: the LQR regret $\overline{R}_T$ can be defined with an \emph{a prior} benchmark, because the dynamics are stochastic. On the other hand, the non-stochastic benchmark is defined \emph{a posteriori}, after because the noises are selected by an adversary.

\subsection{Online LQG \label{ssec:app_lqg}}
In the LQG, or linear quadratic gaussian control, one typically assumes a partially observed dynamical system, inheriting the full generality of \Cref{eq:LDS_system_app}. Again, the cost function is typically taken to be quadratic function of input and output:
\begin{align*}
\ell_t(y,u) = \ell(y,u) = y^\top R y + y^\top Q y,
\end{align*}
Again, $R,Q$ are assumed to be positive defined, and thus our assumption that $\ell_t$ are strongly convex subsumes the LQG setting.  Typically, online LQG assumes that both the process noise $\matw_t$  \emph{and the observation noise $\mate_t$} are not only mean zero and stochastic, but also well conditioned. For example, $\matw_t \iidsim \calN(0,\Sigma_w)$ and $\mate_t \iidsim \calN(0,\Sigma_e)$, where $\Sigma_w,\Sigma_e \succ 0$. 

Whereas the unconstrained optimal policy in LQR is an \emph{static} feedback law, the optimal LQG policy is \emph{dynamic} linear controller of the form considered in this work. This is true even if $\Cst = I$ but there is non-zero process noise $\mate_t$; that is, $\maty_t = \matx_t + \mate_t$.


\iftoggle{nips}
{
	\section{Pseudocode, and Dynamic Feedback Generalization\label{app:general}}
	\subsection{Full Pseudocode for Static Feedback Parametrization \label{sec:pseudo_static_feed}}

}
{
	\section{Full Algorithm with Dynamic Feedback\label{app:general}}

}

\subsection{Stabilizing with dynamic feedback \label{app:dynam_stabilize_feedback}}
	In general, a partially observed system can not be able to be stabilized by static feedback. To circumvent this, we describe stabilizing the system with an \emph{dynamic feedback controller}, a parameterization we refer to as $\drcdyn$. The following exposition mirrors \citet{simchowitz2020improper}, but is abridged considerably. Specificially, we assume that our algorithm maintains an internal state $\salg_t$, which evolves according to the dynamical equations
	\begin{align}
	\salg_{t+1} = \Apinot \salg_t + \Bpinot \yalg_t + \Bpinotu \uex_t, \label{eq:salg_t}
	\end{align}
	and selects inputs as a combination of an exogenous input $\uex_t$, and an endogenous input determined by the system:
	\begin{align}
	\ualg_t = \uexalg_t + (\Cpinot \salg_t + \Dpinot \yalg_t). \label{eq:ualg_dynamic}
	\end{align}
	 Lastly, the algorithmic prescribes an \emph{control output}, denoted by $\mateta_t$, given by 
	\begin{align*}
	\etalg_{t+1} = \Cpinoteta \salg_t + \Dpinoteta \yalg_t \in \R^{\dimeta},
	\end{align*}
	which we use to parameterize the controller. In the special case of static feedback, we take $\Cpinoteta = 0$ and $\Dpinoteta = I$, so that $\etalg_t = \yalg_t$. We assume that $\pinot$ is \emph{stabilizing}, meaning that, if we have $\max_t\|\mate_t\|,\|\matw_t\|,\|\uexalg_t\| < \infty$ are bounded, then with  $\max_t \|\ualg_t\|,\|\yalg_t\|,\|\etalg_t\| < \infty$. As a consequence of the Youla parametrization \citep{youla1976modern}, one can always construct a controller $\pinot$ which has this property for sufficiently non-pathological systems.

	Analogous to the sequence $\yk_t,\uk_t$, we consider a sequence that arises under no exogenous inputs:
	\begin{definition} We define the `Nature' sequence $\ynat_t,\unat_t,\etanat_t$ as the sequence obtained by executing the stabilizing policy $\pinot$ in the absence of $\uex_t = 0$; we see $\vnat_t = (\ynat_t,\unat_t) \in \R^{\dimy + \dimu}$. Each such sequence is determined uniquely by the disturbances $\matw_t,\mate_t$.
	\end{definition}
	Moreover, the `Nature' sequences can be related to the sequences visited by the algorithm via linear Markov operators
	\begin{definition} We define the linear Markov operators $\Gexyu,\Gexeta$ as the operators for which
	\begin{align*}
	\etalg_t = \etanat_t + \sum_{i=1}^{t} \Gexeta^{[t-i]} \uex_{i}, \quad \valg_t = \vnat_t + \sum_{i=1}^{t} \Gexyu^{[t-i]} \ualg_i.
	\end{align*}
	We note that $\Gexeta^{[0]} = 0_{\dimeta \times \dimu}$ by construction. 
	\end{definition}
	Finally, we describe our controller parametrization:
	\begin{definition}[\drc{} with dynamic stabilizing controller] Generalizing  \Cref{eq:Mdfc}, let  $ \Mdfc(m,\radM)$ denote $M \in \transferclass{\dimu}{\dimeta}$ for which $\|M\|_{\loneop} \le \radM$, and $M^{[i]} = 0$ for all $i \ge m$. Given estimates $\etanathat_{t-m+1},\dots,\etanathat_{t}$, we select
	\begin{align*}
	\uex_t( M \midhateta) := \sum_{i=0}^{m-1}M^{[i]}\etanathat_{t-1}
	\end{align*}
	\end{definition}

	We recover the static feedback setting in the following example:
	\begin{example}[Static Feedback]
	To recover the special case of static feedback, we make the following substitutions
	\begin{itemize}
		\item We set $\salg_t = 0$ for all $t$, $\Cpinot= 0$ and $\Dpinot = K$.
		\item We set $\Cpinoteta = 0$ and $\Dpinoteta = I$, so that $\ualg_t = \uexalg_t + K\yalg_t$
		\item We set we set $\Cpinoteta = 0$ and $\Dpinoteta = I$, so that $\etalg_t = \yalg_t$ for all $t$. 
		\item The quantities $\ynat_t$ and $\etanat_t$ both correspond to $\yk_t$, and $\unat_t = \uk_t$, the operator $\Gexyu$ becomes the Markov operator $\Gk$, and $\Gexeta$ becomes the top $\dimy \times \dimu$ block of $\Gk$, capturing the response from $\uex_t \to \maty_t$. 
		\item Thus, $\uex_t( M \midhateta)$ corresponds to $\uex_t( M \mid \ykhat_{1:t})$.
	\end{itemize} 
	\end{example}

\subsection{Full Algorithm under Dynamic Feedback}
	Let us now turn to the specific of the main algorithm with dynamic feedback, \drconsdyn. Throughout the algorithm, we maintain an internal state updated according to the nominal controller $\pinot$ via \Cref{eq:salg_t}. Moreover, all inputs are selected as $\ualg_t = \uexalg_t + (\Cpinot \salg_t + \Dpinot \yalg_t)$ in accordance with \Cref{eq:ualg_dynamic}.

	Next, we specify how we recover $\valg_t$ and $\etalg_t$. Given estimates $\Ghatexyu,\Ghatexeta$, we parallel \Cref{eq:recover_hat} in defining
	\begin{align}
\vnathat_t &:= \begin{bmatrix}{\ynathat_t}\\\unathat_t\end{bmatrix} =  \begin{bmatrix} \yalg_t \\
\Cpinot \salg_t + \Dpinot \yalg_t
\end{bmatrix} - \sum_{i=1}^{t-1} \Ghatexyu^{[i]} \uexalg_{t-i}, \nonumber\\ 
\etanathat_t &:= \etalg_t - \sum_{i=1}^{t-1} \Ghatexeta^{[i]} \uexalg_{t-i} \label{eq:recover_hat_general}.
\end{align}
As in the static feedback case, the above exactly $\vnat_t,\etanat_t$ for exact estimates $\Ghatexyu = \Gexyu$ and $\Ghatexeta = \Gexeta$. We then contruct optimization losses as follows, mirroring \Cref{eq:unary_unknown}:
\begin{align}
\fhat_t(z) := \loss_t(\vkhat_t + \matHhat_t z), \text{ where } \matHhat_t := \sum_{i=0}^h \Ghatexyu^{[i]}\matY_{t-i}, ~~ \text{and } \matY_{s} = \embedeta[\etanathat_{s:s-m}] \label{eq:dynamic_unary_unknown},
\end{align}
where $\embedeta$ is an embedding map analogues to $\embedy$.

With these estimates and definitions, \Cref{alg:nfc_dyn,alg:estimation_dynamic} provides the pseudocode generalizing \Cref{alg:nfc,alg:estimation_static} to our setting. The main differences are
\begin{itemize}
	\item Using $\etanathat_t$ for the controller parameterization, rather than $\yk_t$.
	\item Mainting the internal state $\salg_t$
	\item Estimating two sets of Markov parameters, $\Ghatexeta$ and $\Ghatexyu$.
\end{itemize}

\iftoggle{alt}{\begin{algorithm2e}[h]}{\begin{algorithm}[h]}
	      \textbf{parameters}: Newton parameters $\eta,\lambda$, radius $\radM$, \drc{} length $m$, memory $h$,  closed-loop Markov operator estimate $\Ghatexeta,\Ghatexyu$, initial internal state $\salg_1$\\
	     {} \textbf{initialize:}\\
	      {}~~~~~constraint set $\calM\gets \Mdfc(h,\radM)$ (\Cref{eq:Mdfc}), with $\calC \gets \embedM(\calM)$.\\ 
	     {}~~~~ optimization subroutine $\calA \leftarrow \semions(\eta,\lambda,\calC)$ (\Cref{alg:semions}), with iterates $\matz_k$  \label{line:ons_instatiate}\\
	     {}~~~~ initial values $\etanathat_{0},\etanathat_{\shortminus 1},\dots,\etanathat_{\shortminus (m+h)} \leftarrow 0$ \\
	    \For{$t=1,2,\dots$:} 
	    {
	     {} \textbf{recieve} $\yalg_t$ from environment \\
	      {} Construct estimate $\vnathat_t = (\ynathat_t,\unathat_t)$ and $\etanathat_t$ via \Cref{eq:recover_hat_general}\\
	     {} Recieve iterate $\matz_t$ from $\calA$, and back out \drc{} parameter $\matM_t \leftarrow \embedM^{\shortminus 1}[\matz_{t}]$.\\
	     {} \textbf{play} input $\ualg_t \leftarrow \Dpinot\yalg_t + \Cpinot \salg_t + \uex_t(\matM_t \midhateta[t])$.  \label{line:total_input_dyn}.\\
	     {} \textbf{suffer} loss $\ell_t(\yalg_t,\ualg_t)$, and observe $\ell_t(\cdot)$ \\
	     {} \textbf{feed} $\calA$ the pair $(\fhat_t,\matHhat_t)$, defined in \Cref{eq:dynamic_unary_unknown}, and update $\calA$  \\
	     {} \textbf{update} internal state $\salg_{t+1}$ according to  \Cref{eq:salg_t}.}
	  \caption { $\drconsdyn$ from Markov Parameter Estimates}
	  \label{alg:nfc_dyn}
\iftoggle{alt}{\end{algorithm2e}}{\end{algorithm}}

\iftoggle{alt}{\begin{algorithm2e}[h]}{\begin{algorithm}[h]}
	      \textbf{Input: } Number of samples $N$, system length $h$, \drc{} length $m$, learning parameters $\eta,\lambda$.\\
	     {} \textbf{Initialize}  $\Ghatexyu^{[0]} = \begin{bmatrix} 0_{\dimu \times \dimy} \\ I_{\dimu} \end{bmatrix}$, and $\Ghatexyu^{[i]} = 0$ for $i > h$, and  $\Ghatexeta^{[i]} = 0$ for $i = 0$ and for $i > h$, $\salg_1 = 0$
	    \For{t = $1,2,\dots,N$}
	    {
	       {} \textbf{draw} $\uexalg_t \sim \mathcal{N}(0,I_{\dimu})$\\
	       {} \textbf{receive} $\valg_t = (\yalg_t,\ualg_t)$ and $\etalg_t$.\\
	       {} \textbf{play} $\ualg_t  = \uexalg_t +(\Cpinot \salg_t + \Dpinot\yalg_t)$\\
	       {} \textbf{update} internal state $\salg_{t+1}$ according to  \Cref{eq:salg_t}.
	    }
	     {} \textbf{estimate} $\Ghat^{[1:h]}$ via
	    \begin{align*}
	    &\Ghatexyu^{[1:h]} \leftarrow \argmin_{G^{[1:h]}} \sum_{t=h+1}^N\|\valg_t - \sum_{i=1}^h G^{[i]}\uexalg_{t-i}\|_2^2\\
	     &\Ghatexeta^{[1:h]} \leftarrow \argmin_{G^{[1:h]}} \sum_{t=h+1}^N\|\etalg_t - \sum_{i=1}^h G^{[i]}\uexalg_{t-i}\|_2^2
	    \end{align*}
	     {} \textbf{run} \Cref{alg:nfc_dyn} for times $t = N+1,N+2,\dots,T$, using $\Ghatexeta,\Ghatexyu$ as the Markov parameter estimates, and parameters $m,h,\lambda,\eta$, and state $\salg_{t+1}$.
	    \caption{Full \drconsdyn{} for Unknown System (with estimation) \label{alg:estimation_dynamic}}
\iftoggle{alt}{\end{algorithm2e}}{\end{algorithm}}

 \newcommand{\eventest}{\mathcal{E}^{\mathrm{est}}}
\newcommand{\Fbar}{\bar{F}}
\newcommand{\fbar}{\bar{f}}
\newcommand{\MemRegHat}{\widehat{\MemReg}}
\newcommand{\uexpipinot}{\matu^{\mathrm{ex},\pinot \to \pi}}
\newcommand{\Gcheckexyu}{\check{G}_{\mathrm{ex} \to v}}
\newcommand{\exyu}{_{\mathrm{ex}\to v}}
\newcommand{\bigohconst}[1]{\mathcal{O}_{\mathrm{cnst}(#1)}}

\section{Full Control Regret Bounds and Proofs \label{sec:control_proofs_ful} }
This section states and proves our main results for the control setting. We state and prove \Cref{thm:drconsdyn_known,thm:drconsdyn_unknown} for the general, dynamic-internal controllers described in \Cref{app:general}. We then derive the regret bounds \Cref{thm:known_control,thm:unknown_control} in the main text as consequences of the above theorems. In addition, we state variations of the main-text bounds which make explicit the parameter settings which attain the desired regret (\Cref{thm:drconsstat_known_granular,thm:drcons_unknown_specific}). The section is organized as follows:
 \begin{itemize}
 \item \Cref{ssec:prelim_asm} gives the requisite assumptions and conditions for the general setup of \Cref{app:general}, which replaces the static $K$controller with dynamics internal controller.
 \item \Cref{eq:complete_reg_bounds} states the general regret guarantees \Cref{thm:drconsdyn_known,thm:drconsdyn_unknown} for the dynamic-internal-controller setup. It also states \Cref{thm:drconsstat_known_granular,thm:drcons_unknown_specific} - the complete regret bounds for static feedback with parameter settings made explicit. The static regret bounds are derived in \Cref{sssec:spec_to_stat}.
 \item \Cref{app:invertibility} proves the bound on the invertibility modulus $\kappa(\Gk)$, \Cref{lem:kap_bound_K}. It also provides discussion regarding the invertibility modulus in the dynamically-stabilized setting (see \Cref{rem:general_guarantees_inver_mod}.
 \item \Cref{ssec:control_known_proof} proves the dynamically-stabilized setting guarantee for the known system, \Cref{thm:drconsdyn_known}. The proof combines the regret decomposition from \citet{simchowitz2020improper} with our policy regret bound, \Cref{thm:semions_memory}. 
 \item \Cref{ssec:control_proof_unknown} proves the dynamically-stabilized setting guarantee for the unknown system,  \Cref{thm:drconsdyn_known}. Again, we combine the existing regret decompositions with the policy regret bound \Cref{thm:semions_unknown}. 
\end{itemize}
The arguments that follow essentially reuse lemmas from \cite{simchowitz2020improper} to port over our policy regret bounds for \semions{} to the control setting. We state formal reductions for the known and unknown system settings in \Cref{prop:redux_known,prop:redux_unknown}, which may be useful in future works applying the \drc{} parameterization.

The only significant technical difference from \cite{simchowitz2020improper} is in the analysis of the unknown system, where we use an intermediate step in their handling of one of the approximation errors. This yields an offset in the $\matY_t$-geometry (see \Cref{prop:redux_unknown}), which is explained further in \Cref{ssec:control_proof_unknown}.

\paragraph{Asymptotic Notation:} Throughout, we will use $\bigohconst{b}$ to denote a quantity $a$ which is at most $C b$, where $C$ is a universal constant independent of problem parameters. Equivalently,  $a = \bigohconst{b}$ if and only if $a \lesssim b$. We use both notations interchangably, and $\bigohconst{\cdot}$ affords convenience.

\subsection{Preliminaries and Assumptions for Dynamic Feedback \label{ssec:prelim_asm}} 
	While the main theorems in the main body of the main text assume explicity geometric decay, the results in this result will be established with a more abstract, yet theoretically more streamlined construction called a \emph{decay function}:
	\begin{definition}[Decay Function] For a Markov operator $G = (G^{[i]})_{i \ge 0}$, we define the decay function as $\psi_{G}(n) := \sum_{i \ge n}\|G^{[i]}\|_{\op}$. We say that $G$ is \emph{stable} if $\psi_{G}(0) < \infty$, which implies that $\lim_{n \to \infty} \psi_G(n) = 0$. In general, we say that $\psi$ is a proper, stable decay function if $\psi(n)$ is non-negative, non-increasing, and $\psi(0) < \infty$.
	\end{definition}

	\begin{asmmod}{asm:stab}{b}[Stability]\label{asm:stab_dyn} We assume that  $\Rpinot := \max\{\|\Gexyu\|_{\loneop},\|\Gexeta\|_{\loneop}\} < \infty$. We further assume that the \emph{decay function} of $\Gexyu$ and $\Gexeta$ are upper bounded by a proper, stable decay function $\psinot$. Note that, when the static analogue \Cref{asm:stab_dyn} holds, we can take 
	\begin{align*}
	\Rpinot = \frac{\constk}{1-\rhok}, \quad\psinot(n) = \Rpinot \rhok^n.
	\end{align*}
	\end{asmmod}
	For any stabilizing $\pinot$, \Cref{asm:stab_dyn} always holds, and in fact $\psinot$ will have geometric decay. In the special case of static feedback $K$, \Cref{asm:stab} implies that
	\begin{align}
	\psi_K(n) \le \frac{c_K\rho_K^n}{1-\rho_K}. \label{eq:psinot}
	\end{align}
	Again, since $\pinot$ is stabilizing, we also may also assume that the iterates $\maty_t^K,\mate_t^K$ are bounded for all $t$:
	\begin{asmmod}{asm:bound}{b}[Bounded Nature's-iterates]\label{asm:bound_dyn} We assume that $(\matw_t,\mate_t)$ are bounded such that, for all $t \ge 1$, $\|\vnat\|,\|\etanat\|\le \radnat$. This is equivalent to  \Cref{asm:bound} in when $\pinot$ corresponds to static feedback $K$.

	\end{asmmod}

	\subsubsection{Policy Benchmarks \label{sec:policy_benchmark}}

	\begin{definition}[Linear Dynamic Controller]\label{defn:LDC} An LDC is specified by a linear dynamical system $(\Api,\Bpi,\Cpi,\Dpi)$, with internal state $\sopenpi_t \in \R^{\dimpi}$,  equipped with the internal dynamical equations $\sopenpi_{t+1} = \Api \sopenpi_t + \Bpi \yopenpi_t \quad \text{and} \quad \uopenpi_t := \Cpi \sopenpi_t + \Dpi \yopenpi_t$. We let $\Pildc$ denote the set of all LDC's $\pi$.  These policies include static fedback laws  $\uopenpi_t = K \yopenpi_t$, but \emph{are considerably more general due to the internal state}. The \emph{closed loop} iterates $(\matypi_t,\matupi_t,\matxpi_t,\matspi_t)$ denotes the unique sequence consistent with  \Cref{eq:LDS_system}, the above internal dynamics, and the equalities $\uopenpi_t = \matu_t$, $\yopenpi_t = \maty_t$. The sequence $(\yk_t,\uk_t)$ is a special case with $\Dpi = K$ and $\Cpi =0$.
	\end{definition}

	\paragraph{Dynamic Policy Benchmark}

	 Lastly, let us quantitatively define our policy benmark, from \cite{simchowitz2020improper}.

	\begin{defnmod}{defn:benchmark}{b}[Policy Benchmark] \label{def:dynamic_benchmark}
	We define a $\pinot \to \pi$ as a Markov operator $\Gpipinot$ such that the inputs $\uexpipinot_t := \sum_{i=1}^t \Gpipinot^{[t-i]} \etanat_i$ satisfies the following for all $t$:
	\begin{align*}
	\begin{bmatrix}
	\matypi_t \\ \matupi_t
	\end{bmatrix} = \begin{bmatrix}
	\naty_t \\ \natu_t
	\end{bmatrix}  + \sum_{i=1}^t \Gexyu^{[t-i]} \uexpipinot_{i}.
	\end{align*}
	where $(\matypi_t,\matupi_t)$ is the sequence obtained by executed LDC $\pi$. We define the comparator class
	\begin{align*}
	\Pistar := \Pi_{\mathrm{stab},\pinot}(\radcomp,\psicomp), \text{ where } \Pi_{\mathrm{stab},\pinot}(R,\psi) := \{\pi \in \Pildc : \|\Gpipinot\|_{\loneop} \le R, \psi_{\Gpipinot}(n) \le \psi(n), \forall n\}.
	\end{align*}
	\end{defnmod}

	Exact expressions for conversion operators are detailed in \citet[Appendix C]{simchowitz2020improper}.

	\paragraph{Static Policy Benchmark}

	\begin{definition}[Static Feedback Operator]\label{defn:gpicl} Let $\Gpicl$ denote the Markov operator $\Gpicl^{[i]} = \Dpicl \I_{i=0} + \Cpicl \Apicl^{i-1}\Bpicl \I_{i > 0}$, where we define
	\begin{align*}
	\Apicl &:= \begin{bmatrix}
	\Ast + \Bst \Dpi \Cst & \Bst \Cpi \\
	\Bpi \Cst & \Api
	\end{bmatrix}, \quad \Bpicl = \begin{bmatrix} \Bst \Dpi & -\Bst  \\
	\Bpi & 0 \end{bmatrix}\\
	\Cpicl &:= \begin{bmatrix} (\Dpi - \Dpinot) \Cst & \Cpi  \end{bmatrix}, \Dpicl = \begin{bmatrix} \Dpi & 0\end{bmatrix}
	\end{align*}
	\end{definition}

	 To specialize to the static-feedback setting described in the main text of the paper, we develop the following concrete expression:
	\begin{lemma}[Conversion operators for static feedback] \label{lem:static_conversion}
	Consider the special case of the above, where $\pi_0$ is corresponds to static feedback with matrix $K$. Then, the following is a $K \to \pi$ conversion operator. 
	\begin{align*}
	G_{K \to \pi}^{[i]} &= \Dpi \I_{i = 0} +  \I_{ i > 0}\Cpicl \Apicl^{i-1}\Bpicl \begin{bmatrix} I \\ K
	\end{bmatrix},
	\end{align*}
	Next,  fix $c_{\star} > 0,\rhocomp \in (0,1)$, and recall the set $ \Pi_{\mathrm{stab}}(c_{\star},\rho_{\star}) := \{\pi: \forall n, \|\Gpicl^{[n]}\|_{\op} \le \constcomp \rhocomp^n\}$. Then defining
	\begin{align}
	\psi_{\star}(n) := \frac{(1 + \|K\|_{\op})c_{\star} \rho_{\star}^n}{1- \rho_{\star}}, \quad  R_{\star} := \frac{(1 + \|K\|_{\op})c_{\star}}{1- \rho_{\star}} \label{eq:psi_bound_static}.
	\end{align}
	we have that $\pi \in \Pistar$, where  $\Pistar = \Pi_{\pinot,\mathrm{stab}}(R_{\star},\psi_{\star})$ as defined in \Cref{def:dynamic_benchmark}. Lastly, in the special case where the target policy $\pi$ corresponds to another static feedback law $\matu_t = K_{\pi} \maty_t$, then
	\begin{align}
	G_{K \to \pi}^{[i]} = \I_{i = 0} K_{\pi} + (K_{\pi} - K) \Cst (\Ast +  \Bst K \Cst)^{i-1} \Bst(K_{\pi} - K) \label{eq:conversion_static_feedback_to_static}
	\end{align}
	\end{lemma}
	\begin{proof} The first and third statements are a special case of \citet[Proposition 1]{simchowitz2020improper},  taking $\Dpinot = K$, and $\Apinot,\Bpinot,\Cpinot$ identically zero. For the second statement follows from the fact that $\|
	G_{K \to \pi}^{[i]}\|_{\op} \le (1+\|K\|_{\op})\|\Gpicl^{[i]}\|_{\op}$. 
	\end{proof}

\subsection{Complete Statement of Regret Bounds for control setting \label{eq:complete_reg_bounds}}

Here, we state our main regret bounds for both general dynamical internal controllers (\Cref{thm:drconsdyn_known,thm:drconsdyn_unknown}), and specialization for static controllers, \Cref{thm:drconsstat_known_granular,thm:drcons_unknown_specific}. The main theorems in the text \Cref{thm:known_control,thm:unknown_control} are special cases of the latter. Proofs of specialization to  static controllers are provided in \Cref{sssec:spec_to_stat} below.

\begin{assumption}[Invertibility Modulus]\label{asm:invertibility_modulus} For the setting setting, where the system is stabilized by a possibility non-static nominal controller $\pinot$, we assuch that the Markov operator $\Gexyu$ satisfies $\kappa(\Gexyu) > 0$.
\end{assumption}
\begin{remark}[Conditions under which \Cref{asm:invertibility_modulus} holds] From \Cref{lem:kap_bound_K}, we note that \Cref{asm:invertibility_modulus} holds whenever $\pinot$ corresponds to stabilizing the system with a static controller. In general, it is more opaque when \Cref{asm:invertibility_modulus} assumption holds. We discuss this in more detail in the \Cref{app:invertibility}.
\end{remark}
With our general setting and notation in place, we are ready to state our general bound. Throughout, we consider a comparator class
\begin{align*}
\Pistar &:= \Pi_{\mathrm{stab},\pinot}(\radcomp,\psicomp), \text{ where } \\
&\Pi_{\mathrm{stab},\pinot}(R,\psi) := \{\pi \in \Pildc : \|\Gpipinot\|_{\loneop} \le R, \psi_{\Gpipinot}(n) \le \psi(n), \forall n\},
\end{align*}
as defined in \Cref{def:dynamic_benchmark}. 
\begin{thmmod}{thm:known_control}{b}[Main Regret Guarantee of \drconsdyn: Known System]\label{thm:drconsdyn_known} Suppose that \ref{asm:loss_reg},\ref{asm:stab_dyn},\ref{asm:bound_dyn}, \ref{asm:invertibility_modulus} hold. Moreover, choose $\lambda = 6h\radnat^2 \Rpinot^2$, $\eta = 1/\alpha$, and suppose that $m,h$ are selected so that that $\psipinot(h+1) \le \Rpinot/T$, $\psicomp(m) \le c\radcomp/T$, and $\radM \ge \radcomp$. Then, the \drconsdyn{} algorithm (\Cref{alg:nfc_dyn})  enjoys the following regret bound: 
\begin{align*}
\ControlReg_T(\Alg;\Pistar)  \lesssim (\alpha \sqrt{\kappa})^{-1}
	mh^2  \dimu \dimeta\Rpinot^3 \radnat^2 \radM^2 L^2\log\left(1 + T\right),
\end{align*} 
The above guarantee is also inherited by \drcons{} (\Cref{alg:nfc}) as a special case.
\end{thmmod}
The above theorem is proven in \Cref{ssec:control_known_proof}. For static stabilizing controllers, we obtain the following specialization. 

\begin{thmmod}{thm:known_control}{a}[Main Regret Guarantee of \drcons: Known System, with Explicit Parameters]\label{thm:drconsstat_known_granular}   Suppose \Cref{asm:loss_reg,asm:bound,asm:stab} holds, and for given $\rhocomp \in (0,1),\constcomp > 0$, let $\Pistar$ be as in \Cref{defn:benchmark}. Select parameters  
\begin{itemize}
	\item $h = \ceil{\frac{\log T}{1 - \rhok}}$
	\item $m = \ceil{\frac{\log T}{1 - \rhocomp}}$
	\item $\radM = \radcomp = (1+\opnorm{K})\frac{\constcomp}{1-\rhocomp}$
	\item $\eta = 1/\alpha$, and $\lambda = 6h\radnat^2 \constk^2 (1-\rhok)^2$
\end{itemize}
Then,\begin{align*}
\ControlReg_T(\Alg;\Pistar)  \lesssim \frac{ \constk^3 \constcomp^2(1+\|K\|_{\op})^3 }{(1-\rhok)^{5}(1-\rhocomp)^{3}}  \cdot \dimu \dimy \radnat^2 \cdot \frac{L^2}{\alpha}\log^4(1+T)
\end{align*} 
\end{thmmod}
For unknown systems, the following guarantees $\BigOhTil{\sqrt{T}}$ regret:
\begin{thmmod}{thm:unknown_control}{b}[Main Regret Guarantee of \drconsdyn: Unknown System]\label{thm:drconsdyn_unknown} Suppose that Assumptions \ref{asm:loss_reg},\ref{asm:stab_dyn},\ref{asm:bound_dyn}, \ref{asm:invertibility_modulus} hold, and that $\ell_t$ are $L$-smooth \emph{(}$\nablatwo \ell_t \preceq L$\emph{)}. Lastly, fix $\delta \in (0,1/T)$. Then, when the unknown-system variant of $\drconsdyn$ with estimation (\Cref{alg:estimation_dynamic}) is run with the following choice of parameters 
\begin{itemize}
	\item $\lambda = \radnat^2\log(1/\delta)\sqrt{T} + h\Rpinot^2$ and $\eta = 3/\alpha$
	\item $N = h^2\sqrt{T}\max\{\dimeta,\dimy + \dimu\}$
	\item $\sqrt{T} \ge 4\cdot 1764 h^2 \radM^2 \Rpinot^2 + c_0 h^2 \dimu^2$, where $c_0$ is a universal constant arising from conditioning of the least squraes problem\footnote{Empirically, one can just verify whether the LS problem is well conditioned}.
	\item $m \ge \mcomp + 2h$ and $\radM \ge 2\radcomp$.
	\item $\psipinot(h+1) \le \Rpinot/T$, $\psicomp(m) \le \radcomp/T$
\end{itemize}
Then, the following regret bound holds with probability $1 - \delta$:
\begin{align*}
\ControlReg_T(\Alg; \Pistar) &\lesssim  \log(1+T)\frac{(\dimeta + \dimy)(\dimy + \dimu) mh L^2\Rpinot^4 \radnat^5\radM^4\sqrt{T}\log(1/\delta)}{\alpha \kappa^{\nicehalf}}.
\end{align*}
The same guarantee also holds for the static analgoue  \Cref{alg:estimation_static}).
\end{thmmod}

The following specializes to static control:
\begin{thmmod}{thm:unknown_control}{a}[Main Regret Guarantee of \drcons: Unknown System, with Explicit Parameters]\label{thm:drcons_unknown_specific} Suppose that \ref{asm:loss_reg},\ref{asm:stab_dyn},\ref{asm:bound_dyn}, \ref{asm:invertibility_modulus} hold, and that $\ell_t$ are $L$-smooth \emph{(}$\nablatwo \ell_t \preceq L$\emph{)}. For simplicity, further select comparator parameters $\rhocomp \ge \rhok$, $\constcomp \ge \constk$. Finally, fix $\delta \in (0,1/T)$. Then, when the unknown-system variant of $\drconsdyn$ with estimation (\Cref{alg:estimation_dynamic}) is run with the following choice of parameters 
\begin{itemize}
	\item $h = \ceil{(1-\rhocomp)^{-1}\log T}$, $m = 3h$, $\radM = 2\frac{(1+\opnorm{K})\constcomp}{1-\rhocomp}$.
	\item $\lambda = \radnat^2\log(1/\delta)\sqrt{T} + h\constk^2/(1-\rhok)^2$ and $\eta = 3/\alpha$
	\item $N = h^2 \sqrt{T}(\dimy+\dimu)$
	\item $\sqrt{T} \ge c \log^2 T ( (1-\rhocomp)^{-6}\constcomp^4(1+\opnorm{K})^2 + (1-\rhocomp)^{-2} \dimu^2)$ for some universal constant $c$ (satisfied for $T = \BigOhTil{1}$).
\end{itemize}
Then, the following regret bound holds with probability $1 - \delta$:
\begin{align*}
\ControlReg_T(\Alg; \Pistar) &\lesssim  \sqrt{T} \cdot\frac{\constk^4\constcomp^4(1+\opnorm{K})^5}{(1- \rhok)^{4}(1-\rhocomp)^{6}} \cdot \frac{L^2\radnat^5 }{\alpha}\cdot \log^3(1+T)\log(1/\delta)
\end{align*}
The same guarantee also holds for the static analgoue  \Cref{alg:estimation_static}).
\end{thmmod}

\subsubsection{Specializing Dynamic Stabilizing Controller to Static  \label{sssec:spec_to_stat}}
	
	\begin{proof}[ Proof of \Cref{thm:known_control,thm:drconsstat_known_granular} ] For the static case, as noted in \Cref{asm:stab_dyn,asm:bound_dyn},  \Cref{asm:bound} implies  \Cref{asm:bound_dyn}, and \Cref{asm:stab} implies \Cref{asm:stab_dyn} with
	\begin{align*}
	\Rpinot = \frac{\constk}{1-\rhok}, \quad\psinot(n) = \Rpinot \rhok^n.
	\end{align*}
	Moreover, recall that our benchmark is $\pi \in \Pi_{\mathrm{stab}}(c_{\star},\rho_{\star})$, as defined in \Cref{defn:benchmark}. from \Cref{lem:static_conversion}, this benchmark is subsumed by the benchmark $\Pistar$ for the choice of $\psicomp,\radcomp$, as in \Cref{eq:psi_bound_static}:
	\begin{align*}
	 R_{\star} := \frac{(1 + \|K\|_{\op})c_{\star}}{1- \rho_{\star}} , \quad \psi_{\star}(n) \le R_{\star} \rho_{\star}^n. 
	\end{align*}
	Let us now use the following technical claim:
	\begin{fact}\label{fact:rho_claim} Let $\rho \in (0,1)$. Then $\rho^n \le 1/T$ for $n \ge \frac{\log T}{1-\rho}$
	\end{fact}
	\begin{proof}[Proof of \Cref{fact:rho_claim}] We have $\rho^n \le 1/T$ for $n \ge \log(T)/\log(1/\rho)$. But $\log(1/\rho) \le \frac{1}{\rho} - 1 = \frac{1 - \rho}{\rho} $, so it suffices to select $n \ge \log(T)(\rho/1 - \rho) \ge \log(T)/(1-\rho)$.
	\end{proof}
	Thus, our conditions $\psipinot(h+1) \le \Rpinot/T$, $\psicomp(m) \le c\radcomp/T$, and $\radM \ge \radcomp$ hold as soon as 
	\begin{align*}
	h \ge \frac{\log T}{1 - \rhok}, \quad  \ge \frac{\log T}{1 - \rhocomp}.
	\end{align*}
	Thus, setting $h = \ceil{\frac{\log T}{1 - \rhok}}$, $m = \ceil{\frac{\log T}{1 - \rhocomp}}$, and $\radM = \radcomp = (1+\opnorm{K})\frac{\constcomp}{1-\rhocomp}$, and $\kappa(\Gk) \ge \frac{1}{4}\min\{1,\|K\|_{\op}^{\shortminus 2}\} \gtrsim (1+\|K\|_{\op})^{-2}$, we obtain
	\begin{align*}
	\ControlReg_T(\Alg;\Pistar)  \lesssim \frac{ \constk^3 \constcomp^2(1+\|K\|_{\op})^3 }{(1-\rhok)^{5}(1-\rhocomp)^{3}}  \cdot \dimu \dimy \radnat^2 \cdot \frac{L^2}{\alpha}\log^4(1+T)
	\end{align*} 
	This requires the step size choice of $\eta = 1/\alpha$ and $\lambda = 6h\radnat^2 \constk^2 (1-\rhok)^2$. 
	\end{proof}

	\begin{proof}[\Cref{thm:drcons_unknown_specific}] For static feedback, we have $\dimeta = \dimy$. Thus,  $(\dimeta + \dimy)(\dimy + \dimu) = \dimy (\dimy + \dimu)$. Next, we have $\Rpinot^4\radM^4 = (1- \rhok)^{-4}\constk^4 \cdot (1+\opnorm{K})^4(1- \rhocomp)^{-4}\constcomp^4$, and $h \le m \lesssim (1-\rhocomp)^{-1}\log (1+T)$. This gives
	\begin{align*}
	(\dimeta + \dimy)(\dimy + \dimu)  mh \Rpinot^4\radM^4 \lesssim \dimy(\dimy +\dimu)\frac{\constk^4\constcomp^4(1+\opnorm{K})^4}{(1- \rhok)^{4}(1-\rhocomp)^{6}}\log^2 (1+T).
	\end{align*}
	Using $1/\sqrt{\kappa} \lesssim (1+\opnorm{K})$, we then get
	\begin{align*}
	\ControlReg_T(\Alg; \Pistar) &\lesssim  \sqrt{T} \cdot\frac{\constk^4\constcomp^4(1+\opnorm{K})^5}{(1- \rhok)^{4}(1-\rhocomp)^{6}} \cdot \frac{\radnat^5 L^2}{\alpha}\cdot \log^3(1+T)\log(1/\delta)
	\end{align*}
	The correctness of the various parameter settings can e checked analogously.
	\end{proof}

\subsection{Invertibility-Modulus and Proof of \Cref{lem:kap_bound_K} \label{app:invertibility}}
In this section, we bound the condition-modulus $\kappa(\Gk)$ defined in \Cref{defn:input_recov}, and generalize the notion to \drcdyn{} parametrizations. To begin, we recall our desired bound:
\lemkapK*

For general \drcdyn{} parameters, the Z-transform yields a clean lower bound for the condition-modulus of $\Gcheckexyu$ from \Cref{defn:input_recov}:

\begin{proposition}\label{prop:cond_mod_general} 
Define the Z-transform $\Gcheckexyu := \C \to \C^{(\dimy + \dimu) \times \dimu}$ as the function
\begin{align*}
\Gcheckexyu(z) = \sum_{i =0}^{\infty} \Gexyu^{[i]} z^{-i}
\end{align*}
Then, we have the lower bound: 
\begin{align*}
\kappa(\Gexyu) \ge \min_{z \in \Torus} \sigma_{\min}(\Gcheckexyu(z))^2,
\end{align*}
where $ \kappa(\Gexyu)$ is the condition-modulus of $\Gexyu$, as defined in \Cref{defn:input_recov}. In particular, if $\Gexyu$ takes the form
\begin{align*} 
\Gexyu^{[i]} = \I_{i= 0}D\exyu + \I_{i > 0} C\exyu A\exyu^{i-1} B\exyu,
\end{align*}
then
\begin{align*}
\kappa_{\pinot} \ge \min_{z \in \Torus} \sigma_{\min}(D\exyu +  C\exyu (zI - A\exyu)^{-1} B\exyu)^2.
\end{align*}
\end{proposition}
\begin{proof}[Proof of \Cref{prop:cond_mod_general}] 
Part 2 applies the well-known formula that the Z-transform of an LTI system with operator $G^{[i]} =  D \I_{i = 0} + CA^{i-1} B \I_{i > 0}$, which can be computed via
\begin{align*}
\check{G}(z) &= D + C\left(\sum_{i \ge 1} A^{i-1} z^{-i} \right) B\\
&= D + C\left(z^{-1}\sum_{i \ge 0} (A/z)^{i} \right) B\\
&= D + C\left(z^{-1}(I  - A/z)^{-1} \right) B\\
&= D + C\left(zI  - A)\right)^{-1}  B,
\end{align*}
where we use formal identity identity $\sum_{i \ge 0} X^i = (I - X)^{-1}$.

\newcommand{\Ucheck}{\check{U}} 

Let us turn to the first part of the proof. We adopt the argument from \cite[Appendix F]{simchowitz2020improper}.  Fix $u_0,u_1,\dots$ with $\sum{n=0}^{\infty}\|u_n\|^2 = 1$, and define a Markov-shaped vector $U = (U^{[i]})$, with $U^{[i]}$, and its Z-transform $\Ucheck(z) := \sum_{i = 0}^n U^{[i]}z^{-i}$.  We have that
\begin{align*}
\sum_{n \ge 0}\left\|\sum_{i=0}^n G^{[i]}u_{n-i}\right\|_2^2 = \sum_{n \ge 0}\|(G * U)^{[n]} \|^2 
\end{align*}
where $*$ denotes the convolution operator. By Parseval's identity, we have that
\begin{align*}
\sum_{n \ge 0}\left\|(G*U)^{[n]}\right\|_2^2 = \frac{1}{2\pi}\int_{0}^{2\pi} \|\widecheck{(G*U)}(e^{\iota \theta})\|_2^2\rmd\theta,
\end{align*}
where $\widecheck{(G*U)}(z) = \sum_{i \ge 0} (G*U)^{[i]} z^{-i}$ is the Z-transform of $G*U$. Because convolutions become multiplications under the Z-transformation, we have that for the Z-transform of $U$, 
\begin{align*}
\frac{1}{2\pi}\int_{0}^{2\pi} \|\widecheck{(G*U)}(e^{\iota \theta})\|_2^2\rmd\theta = \frac{1}{2\pi}\int_{0}^{2\pi} \|\check{G}(e^{\iota \theta}) \check{U}(e^{\iota \theta})\|_2^2\rmd\theta.
\end{align*} 
This establishes the first equality of the claim. For the inequality, we have
\begin{align*}
\frac{1}{2\pi}\int_{0}^{2\pi} \|\check{G}(e^{\iota \theta}) \check{U}(e^{\iota \theta})\|_2^2\rmd\theta &\ge \frac{1}{2\pi}\int_{0}^{2\pi} \sigma_{\min}(\check{G}(e^{\iota \theta}))^2 \|\check{U}(e^{\iota \theta})\|_2^2\rmd\theta\\
&\ge \min_{z \in \Torus}\sigma_{\min}(z)^2  \cdot \frac{1}{2\pi}\int_{0}^{2\pi} \|\check{U}(e^{\iota \theta})\|_2^2\rmd\theta.
\end{align*}
To conclude, we note that by Parsevals identity, $\frac{1}{2\pi}\int_{0}^{2\pi} \|\check{U}(e^{\iota \theta})\|_2^2\rmd\theta. = \sum_{n \ge 0}\|U^{[n]}\| = \sum_{n \ge 0}\|u_n\|^2 = 1$, giving $\sum_{n \ge 0}\left\|\sum_{i=0}^n G^{[i]}u_{n-i}\right\|_2^2  = \frac{1}{2\pi}\int_{0}^{2\pi} \|\check{G}(e^{\iota \theta}) \check{U}(e^{\iota \theta})\|_2^2\rmd\theta \ge \min_{z \in \Torus}\sigma_{\min}(z)^2$, as needed.

\end{proof}

We now turn to giving an explicit lower for the static-feedback stabilized setting:
\newcommand{\Achk}{\check{A}}
\begin{proof}[Proof of \Cref{lem:kap_bound_K}]
For the special case of static feedback, we recall from \Cref{eq:superposition_id} that
\begin{align*}
\Gexyu^{[i]} = \GK^{[i]} =  \I_{i = 0}\begin{bmatrix} 0 \\ I \end{bmatrix} + \I_{i > 0 } \begin{bmatrix} \Cst \\
K\Cst \\ \end{bmatrix}  (\Ast + \Bst K \Cst)^{i-1} \Bst, ~i \ge 1.
\end{align*}
Thus, defining $\Achk(z) := (zI - \Ast + \Bst K \Cst)^{-1}$, we have from \Cref{prop:cond_mod_general} that 
\begin{align*}
\Gcheckexyu(z) = \begin{bmatrix} \Cst \Achk(z) \Bst \\
			I + K \Cst \Achk(z) \Bst
\end{bmatrix},
\end{align*}
where the above holds for all $z \in \Torus$ since $K$ is stabilizing. We now invoke a simple linear algebraic fact:
\begin{claim}[Lemma F.2 in \cite{simchowitz2020improper}]
			Consider a matrix of the form 
			\begin{align*}
			W = \begin{bmatrix}
			YZ\\
			 I + X Z
			 \end{bmatrix}  \in \R^{(d_1 + d) \times d},
			 \end{align*} with $Y \in \R^{d_1 \times d_1}$, $X,Z^\top \in \R^{d \times d_1}$. Then, $\sigma_{\min}(W) \ge \frac{1}{2}\min\{1, \frac{\sigma_{\min}(Y)}{\|X\|_{\op}}\}$.  
			\end{claim}
		Applying the above claim with $Y = I$, $X = K$, and $W = \Cst \Achk(z) \Bst$, we conclude that $\sigma_{\min}(\Gcheckexyu(z)) \ge \frac{1}{2} \min\{1, \|K\|_{\op}^{-1}\}$ for all $z \in \C$. Thus, by \Cref{prop:cond_mod_general}, $\kappa(\Gk)  \ge (\frac{1}{2} \min\{1, \|K\|_{\op}^{-1}\})^2 =  \frac{1}{4} \min\{1, \|K\|_{\op}^{-2}\}$, as needed.
\end{proof}

\begin{remark}[Generic Bounds on Invertibility]\label{rem:general_guarantees_inver_mod} In general, we do not have a generic lower bound on the invertibility modulus which is verifiably no-negative for all choices of stabilizing controllers. For one, it is not clear that our lower bound in \Cref{prop:cond_mod_general} is sharp, in part because we are working with real operators. However, there are certain conditions (e.g. Youla parametrization, where $\Ast$ has no eigenvalues $z \in \Torus$, \citet[F.2.3]{simchowitz2020improper}) where we have $\min_{z \in \Torus}\sigma_{\min}(\Gcheckexyu(z))^2$ is strictly positive. 
\end{remark}

\subsection{Control Proofs for Known System \label{ssec:control_known_proof}}
We focus on the dynamic version of our algorithm, \drconsdyn{}, with stabilizing controller $\pinot$. For known Markov operator, this algorithm specializes to \drcons{} in the case of static feedback.   The following theorem reduces to bounding the policy regret:
\begin{proposition}[Reduction to policy regret for known dynamics]\label{prop:redux_known} Consider the \drconsdyn{} algorithm (\Cref{alg:nfc_dyn}) initialized with the exact Markov operators $\Ghatexyu = \Gexyu,\Ghatexeta = \Gexeta$, and iterates $\matM_t$ produced by an arbitrary black-box optimization procedure $\calA$. Further, suppose that $\psicomp(m) \le c\radcomp/T,\psipinot(h+1) \le c\psipinot(h+1)/\Rpinot$ for some $c > 0$. Then,
\begin{align*}
\ControlReg_T(\Alg) \le \MemReg_T(\Alg) + 12Lc \radM^2\Rpinot^2\radnat^2.
\end{align*}
where, for the $F_t,f_t$ losses in \Cref{defn:loss_analysis}, we define
\begin{align*}
\MemReg_T(\Alg) := \sum_{t=1}^T F_t(\matz_{t:t-h} \midnateta[t]) - \inf_{z\in \calM_{\embed}} \sum_{t=1}^T f_t(z \midnateta[t])
\end{align*}
The same is true for \Cref{alg:nfc} (for static feedback).
\end{proposition}
\begin{remark}\label{rem:slack_par_know}. In the above, we allow a slack parameter $c$ on the choice of $m,h$. This means that our main theorems can be generalized slightly to accomodate when $m,h$ are chosen larger-than-needed.
\end{remark}

Next, we bound the relevant parameters required:
\begin{lemma}[Parameter Bounds]\label{lem:parameter_bounds_known} Assume $\radnat,\radM \ge 1$. The following bounds hold
\begin{enumerate}
	\item[(a)] We have $D = \max\{\|z-z'\|:z,z' \in \calM_{\embed}\} \le 2\sqrt{m}\radM$.
	\item[(b)] We have $\radY := \max_{t} \|\matY_t\|_{\op} = \max_{t}\|\embedeta(\etanat_{1:t})\|_{\op} \le \radnat$.
	\item[(c)] We have $\radyc = \max_{t}\max_{z \in \calC}\|\matY_t z\| \le \radM \radnat$.
	\item[(d)] For $G = \Gexyu$, we have $\radG = \|\Gexyu\|_{\loneop} \le \Rpinot$,  $\psiG \le \psipinot$, and $\radH \le \Rpinot \radnat$
	\item[(e)] We have $\radv \le \radnat$, and $\Leff \le 2L\Rpinot \radM \radnat$.
\end{enumerate}
Moreover, $d = m \dimu\dimeta $
\end{lemma}

We are now ready to prove our general regret bound for the known system case, encompassing 

\begin{proof}[Proof of \Cref{thm:drconsdyn_known}]From~\Cref{thm:semions_memory}, we have the bound:
\begin{align*}
\MemReg_T(\Alg) = \sum_{t=1}^T F_t(\matz_{t:t-h}) - \min_{z \in \calC} \sum_{t=1}^T f_t(z) \le 
	3 \alpha h D^2 \radH^2 + \frac{3dh^2 \Leff^2  \radG}{\alpha \kappa^{\nicehalf}}  \log\left(1 + T\right),
\end{align*}
Let us now specific the above constants using \Cref{lem:parameter_bounds_known}. From this lemma, we have that $ \alpha h D^2 \radH^2 = \alpha h m \Rpinot^2 \radnat^2\radM^2$. Moreover, $dh^2 \Leff^2  \radG = 4mh^2 \dimu \dimeta L^2\Rpinot^3 \radM^2 \radnat^2$. Thus, with $\lambda := 6h\radnat^2 \Rpinot^2$ and $\eta = 1/\alpha$, we get
\begin{align*}
\MemReg_T(\Alg)  &\lesssim 
	mh^2\Rpinot^2 \radnat^2 \radM^2 ( \alpha + (\alpha \sqrt{\kappa})^{-1} L^2\Rpinot \dimu \dimeta \log\left(1 + T\right))\\
	&\lesssim (\alpha \sqrt{\kappa})^{-1}
	mh^2 \dimu \dimeta  \Rpinot^3 \radnat^2 \radM^2 L^2\log\left(1 + T\right)),
\end{align*} 
where we used that  $L^2/\alpha\sqrt{\kappa} \ge L^2/\alpha \ge \alpha$ by the assumption $\alpha \le L$.
Combining with \Cref{prop:redux_known}  and again using $L \le L^2/\alpha\sqrt{\kappa}$ ensures that the total control regret $\ControlReg_T$ suffers an additional constant $L$ in the bound, yielding at most
\begin{align*}
\ControlReg_T(\Alg)  \lesssim (\alpha \sqrt{\kappa})^{-1}
	mh^2  \dimu \dimeta\Rpinot^3 \radnat^2 \radM^2 L^2\log\left(1 + T\right),
\end{align*} 
as needed.

\end{proof}

\subsubsection{Proof of \Cref{prop:redux_known}}
	
	We follow the regret decomposition from \cite{simchowitz2020improper}, noting that our assumptions on the dynamics, magnitude bounds, and costs $c_t$ all align. To facilitate reuse of the technical material from \cite{simchowitz2020improper}, we introduce the following loss notation in the $M$-domain:
	\begin{defnmod}{defn:ocoam_control_loss}{b}[Losses for the analysis]\label{defn:loss_analysis}
	Generalizing \Cref{defn:ocoam_control_loss}, we introduce the $z$-space losses,
	\begin{align*}
	F_t( z_{t:t-h} \midhateta[t]) &:= \loss_t(\vnat_t + \sum_{i=0}^h \Gexyu^{[i]} \matY_{t-i} z_{t-i} ), \text{ where } \matY_s = \embedeta(\etanathat_{1:s}),
	\end{align*}
	with unary specialization $F_t(z_{t:t-h} \midhateta[t]) := f_t(z,\dots,z \midhateta[t] ).$
	and their analogues in $M$-space
	\begin{align*}
	\Fbar_t( M_{t:t-h} \midhateta[t]) &:= \loss_t(\vnat_t + \sum_{i=0}^h \Gexyu^{[i]} \uex_t(M  \midhateta[t] )),
	\end{align*}
	and unary specialization $\fbar_t(M\midhateta[t]) := \Fbar_t(M,\dots,M\midhateta[t])$. Observe that, for $\matz_s = \embed(\matM_{s})$ for $s \in [T]$, and $z = \embed(M)$, then 
	\begin{align}
	F_t( \matz_{t:t-h} \midhateta[t])  = \Fbar_t( \matM_{t:t-h} \midhateta[t]), \quad \text{ and } \quad  f_t( z \midhateta[t])  = \fbar_t( M \midhateta[t]) \label{eq:M_to_z_conversion}.
	\end{align}
	\end{defnmod}

	Moving forward, let $(\maty^M,\matu^M)$ denote the sequence produced by selecting input $\uex_t(M \midnateta[t])$ at each $i$. We then have
	\begin{align*}
	&\ControlReg_T(\Alg; \Pistar) \\
	&= \sum_{t=1}^T \cost_t(\yalg_t,\ualg_t) - \inf_{\pi \in \Pistar} \sum_{t=1}^T\cost_t(\matypi_t,\matupi_t)\\
	&\le \underbrace{\sum_{t=1}^T \left|\cost_t(\yalg_t,\ualg_t) - \Fbar_t(\matM_{t:t-h} \midnateta[t])\right|}_{(i.a)} +   \underbrace{\sum_{t=1}^T \Fbar_t(\matM_{t:t-h} \midnateta[t]) - \inf_{M\in \calM} \sum_{t=1}^T \fbar_t(M \midnateta[t])}_{(ii)}\\
	&+  \underbrace{\max_{M \in \calM} \sum_{t=1}^T\left| \fbar_t(M \midnateta[t]) - \cost_t(\maty^M_t,\matu^M_t) \right|}_{(i.b)} + \underbrace{\left|\inf_{M \in \calM} \sum_{t=1}^T\cost_t(\maty^M_t,\matu^M_t) -  \inf_{\pi \in \Pistar} \sum_{t=1}^T\cost_t(\matypi_t,\matupi_t)\right|}_{(iii)}.
	\end{align*} 
	Let's proceed term by term. From \citet[Lemma 5.3]{simchowitz2020improper} (replacing their notation $R_{G_{\star}},\psi_{G_{\star}}$ with our notation $\Rpinot,\psipinot$),
	\begin{align}
	(i.a) + (i.b) \le 4LT\Rpinot\radM^2\radnat^2\psipinot(h+1). \label{eq:truncation}
	\end{align}
	Secondly, from \Cref{eq:M_to_z_conversion}, we have
	\begin{align}
	(ii) = \sum_{t=1}^T F_t(\matz_{t:t-h} \midnateta[t]) - \inf_{z\in \calM_{\embed}} \sum_{t=1}^T f_t(z \midnateta[t]) := \MemReg_T(\Alg). \label{eq:recover_policy_regret}
	\end{align}
	Finally, from \citet[Theorem 1b]{simchowitz2020improper}, we have that for $\radM \ge \radcomp$,
	\begin{align}
	(iii) \le 2LT\radcomp\Rpinot^2\radnat^2 \,\psi(m) \label{eq:policy_approx}
	\end{align}
	Thus, we obtain
	\begin{align*}
	\ControlReg_T(\Alg; \Pistar) &\le (i.a) + (i.b) + (ii) + (iii) \\
	&\le \MemReg_T(\Alg) + 4LT \radM^2\Rpinot^2\radnat^2\left( \frac{\psicomp(m)}{\radcomp} + \frac{2\psipinot(h+1)}{\Rpinot}\right),
	\end{align*}
	Finally, bound $\psicomp(m) \le c\radcomp/T$ and $2\psipinot(h+1) \le c\Rpinot/T$ concludes. \qed
\subsubsection{Proof of \Cref{lem:parameter_bounds_known}}
	We go term by term:
	\begin{enumerate}
		\item[(a)] We have $D \le 2\max \{\|z\| : z  \in \calM_{\embed}\}$. For $z = \embed(M)$, have that $\|z\| = \|M\|_{\fro} \le \sqrt{m}\|M\|_{\loneop} \le \sqrt{m}\radM$ by \citet[Lemma D.1]{simchowitz2020improper} 
		\item[(b)] Each matrix $\matY_t$ can be represented as a block diagonal, with blocks as rows corresponding to $\etanat_{s}$ for $s \in \{t,t-1,\dots,t-m+1\}$. This matrix has operator norm as most $\max\{\|\etanat_{s}\| : s \in \{t,t-1,\dots,t-m+1\}\} \le \radnat$.
		\item[(c)] We have that $\matY_t z = \uex_t(M \mid \etanat_{1:t}) \le \sum_{i=0}^{m-1}\|M^{[i]}\|_{\op} \|\etanat_{t-i}\|_{\op} \le \radM \radnat$ by Holder's inequality. 
		\item[(d)] These bounds followly directly from our definitions.
		\item[(e)] We have $\radv \le \radnat$ by assumption, and $\Leff := 2L\Rpinot \radM \radnat$ follows from the definition $\Leff = L\max\{\radv + \radG \radyc\}$, and the assumption s $\radM,\radnat \ge 1$, and $\Rpinot \ge 1$ by definition ($\Rpinot = \|\Gexyu\|_{\loneop}$, and $\Gexyu^{[i]} = \begin{bmatrix} 0 \\ I \end{bmatrix}$). 
	\end{enumerate}
	\qed

\subsection{Unknown Systen \label{ssec:control_proof_unknown} }
We begin by stating guarantees for the estimation procedures \Cref{alg:estimation_static} and \Cref{alg:estimation_dynamic}, which follow directly past work:

 \begin{lemma}[ Theorem 6b in \citet{simchowitz2020improper}]\label{lem:ls_estimation_guarantees}
  Let $\delta \in (e^{-T},T^{-1})$, $N ,\dimu \le T$, and $\psiGst(h+1) \le \frac{1}{\sqrt{N}}$. Define $\dmax = \max\{\dimy + \dimu,\dimeta\}$, and set
\begin{align*}
 \epsG(N,\delta) &=  \frac{h^2 \radnat}{\sqrt{N}} C_{\delta}, \quad\text{ where } C_{\delta} := 14\sqrt{\dimu + \dmax + \log \tfrac{1}{\delta}}, \quad \text{and }  \Restu := 3\sqrt{\dimu + \log (1/\delta)}.
\end{align*}
and suppose that $N \ge  h^4 C_{\delta}^2 \Restu^2 \radM^2\Rpinot^2 + c_0 h^2 \dimu^2$ for an appropriately large $c_0$, which can be satisfied by taking
\begin{align*}
N \ge  1764 (\dmax + \dimu +\log  (1/\delta))^2h^4 \radM^2 \Rpinot^2 + c_0 h^2 \dimu^2.
\end{align*}
Then  with probability $1 - \delta - N^{-\log^2 N}$, Algorithm~\ref{alg:estimation_dynamic} satisfies the following bounds
\begin{enumerate}
  \item $\epsG \le 1/\max\{\Restu,\radM\Rpinot\}$.
  \item For all $t \in [N]$, $\|\matu_t\| \le \Restu := 3\sqrt{\dimu + \log (1/\delta)}$
  \item For  estimation error is bounded as
  \begin{align*}
  \|\Ghatexeta - \Gexeta\|_{\loneop} &\le \|\Ghatexeta^{[0:h]} - \Gexeta^{[0:h]}\|_{\loneop} + \Restu\psiGst(h+1) \le \epsG\\
   \|\Ghatexyu - \Gexyu\|_{\loneop} &\le \|\Ghatexyu^{[1:h]} - \Gexyu^{[1:h]}\|_{\loneop} + \Restu\psiGst(h+1) \le \epsG.
  \end{align*}
   Moreover, Algorithm~\ref{alg:estimation_static} also satisfies the above for $\Ghatexyu = \Ghat$ and $\Gexyu = \Gk$.
\end{enumerate}
  \end{lemma}

  The above bounds are in turn a consequence of \citet{simchowitz2019learning}. 
  We denote the event of \Cref{lem:ls_estimation_guarantees} as $\eventest$, and the following exposition assumpt it holds. 

  Next, we state a blackbox reduction to the \drc{} online controller framework. This reduction crucially uses the fact that we have \emph{over-parameterized} the set $\calM$. Specifically, over comparator set is
  \begin{align*}
  \calM_{\star} := \Mdrc(\mcomp, \radcomp),
  \end{align*}
  whereas the algorithm uses the over-parametrized set
  \begin{align}
  \calM := \Mdrc(m,\radM), \text{ with } \radM \ge 2\radcomp \text{ and } m \ge 2\mcomp  + h. \label{eq:comparator_relations}
  \end{align}
  By over-parametrizing the controller set as above, we obtain the following guarantee:
\begin{proposition}[Reduction to policy regret for known dynamics]\label{prop:redux_unknown}. Suppose that \Cref{eq:comparator_relations} holds, and that $\psipinot(h+1) \le c \Rpinot/T$ and $\psicomp(m) \le c\radcomp/T$ for some $c > 1$, and that $N \ge m + h$. Consider the \drconsdyn{} algorithm with estimation (\Cref{alg:estimation_dynamic}) initialized with the exact Markov operators $\Ghatexyu = \Gexyu,\Ghatexeta = \Gexeta$, and iterates $\matM_t$ produced by an arbitrary black-box optimization procedure $\calA$. 
\begin{align*}
\ControlReg_T(\Alg; \Pistar) &\le \MemRegHat_T(\zst) +\, \nu \sum_{t=N+m+2h+1}^T  \left\|\matY_t(\matz_t - \zst)\right\|_{2}^2 \\
&\quad+ \bigohconst{L\Rpinot^3 (N+cm)}   \left( \dimu + \log(1/\delta) + \radM^4 \radnat^2 \right) \\
&\quad+ \bigohconst{L\radM^3 \Rpinot^2 \radnat^2 T \epsG^2}\left(1+  \frac{Lm  \Rpinot^2}{\nu}\right)
\nonumber
\end{align*}
where $\bigohconst{1}$ hides a \emph{universal} numerical constants. Here, for the $F_t,f_t$ losses in \Cref{defn:loss_analysis}, we define the term:
\begin{align*}
\MemRegHat_{T}(\Alg;\zst) := \sum_{t=N+m+2h+1}^T F_t(\matz_{t:t-h} \midhateta[t] ) - \inf_{z\in \calM_{\embed}} \sum_{t=N+m+2h+1}^T f_t(z \midhateta[t]).
\end{align*}
Moreover, the same guarantee is also true of \Cref{alg:estimation_static}. 
\end{proposition}
Again, we allow a slack parameter $c$ to allow for over-specifying $m,h$, demonstrating low sensitivity to imperfectly tuned algorithm parameters. Next, we translate the parameter bounds from the control setting to the ones required for the policy regret analysis of \semions{}:
\begin{lemma}[Parameter Bounds for Unknown Setting]\label{lem:par_bounds:unknown} Assume $\radnat \ge 1$, and that $\cdot$. Then, for $t_0 := N+m+h+1$, the following hold
\begin{enumerate}
	\item[(a)] We have $D = \max\{\|z-z'\|:z,z' \in \calM_{\embed}\} \le \sqrt{m}\radM$.
	\item[(b)] We have $\radY := \max_{t \ge t_0} \|\matY_t\|_{\op} \le 2\radnat$.
	\item[(c)] We have $\radyc = \max_{t \ge t_0}\max_{z \in \calC}\|\matY_t z\| \le 2\radM \radnat$.
	\item[(d)] For $G = \Gexyu$, we have $\radG = |\Ghatexyu\|_{\loneop}  \vee \|\Gexyu\|_{\loneop} \le 2\Rpinot$,  $\psiG \le \psipinot$, and $\radH \le 2\Rpinot \radnat$
	\item[(e)] We have $\radv := \max_{t \ge t_0} \|\vk_t\| \vee \|\vkhat_t\| \le 2\radnat$, and $\Leff := 8L\Rpinot \radM \radnat$.
	\item[(f)] We can take $\cv $ to be $3\radM \radnat$.
\end{enumerate}
Moreover, $d = \dimeta\dimy m$
\end{lemma}
Finally, we are in place to prove our main theorem:
\begin{proof}[Proof of \Cref{lem:par_bounds:unknown}] The bounds follow analogously to those in \Cref{lem:parameter_bounds_known}, with the modification that, for $t \ge N+h$, we have $\|\etanathat_t\| \le 2\radnat$ (by \citet[Lemma 6.1]{simchowitz2020improper}), and that $\|\Ghatexyu\|_{\loneop} \le 2\Rpinot$ under $\eventest$. Moreover, we can take the constant $\cv$ which bounds $\|\vnathat_t - \vnat_t\|_2 \le \cv \epsG$ to be $3\radM\radnat$ by \citet[Lemma 6.4b]{simchowitz2020improper}.
\end{proof}

\begin{proof}[Proof of \Cref{thm:drconsdyn_unknown}]
Let us prove the bound for the dynamic-controller variant \Cref{alg:estimation_dynamic}; the static-controller variant works similarly. Recall that we assume the following
\begin{itemize}
	\item $\lambda = \radnat^2\log(1/\delta)\sqrt{T} + h\Rpinot^2$, $\eta = 3/\alpha$
	\item $N = h^2\sqrt{T}\dmax$
	\item $\sqrt{T} \ge 4\cdot 1764 h^2 \radM^2 \Rpinot^2 + c_0 h^2 \dimu^2$
	\item $m \ge \mcomp + 2h$, $\radM \ge 2\radcomp$
	\item $\psipinot(h+1) \le \Rpinot/T$, $\psicomp(m) \le \radcomp/T$.
\end{itemize}
Let $\epsG$ be an upper bound on the estimation error, which we will set to be greater than $\sqrt{T}$. By taking $\lambda  \in [\clam,1] (T\epsG^2 + h\radH^2)$, and applying \Cref{thm:semions_unknown_clam}, we can bound
\begin{align*}
\MemRegHat_T(\zst) +\, \nu \sum_{t=N+m+2h+1}^T  \left\|\matY_t(\matz_t - \zst)\right\|_{2}^2 
\lesssim \\  \clam^{-1}\log(1+\frac{T}{\clam})\left(\frac{C_1}{\alpha \kappa^{\nicehalf}} + C_2\right) \left(T\epsG^2 + h^2(\radG^2 + \radY)\right),
\end{align*}
where $C_1 := (1+\radY)\radG(h+d) \Leff^2$, $C_2:= (L^2 \cv^2/\alpha +  \alpha D^2)$, and $\nustar = \tfrac{\alpha \sqrt{\kappa}}{48(1+\radY)}$ are constants which we must bound presently. Since $d =  \dimeta\dimy m \ge h$, $L \ge \alpha$,and $\kappa \le 1$
\begin{align*}
C_1 &\lesssim \dimeta\dimy m \radnat \Rpinot \Leff^2 \lesssim \dimeta\dimy m  L^2\Rpinot^3 \radnat^3\radM^2 \\
C_2 &\lesssim L^2/\alpha  \radnat^2\radM^2 + m \radM^2 \lesssim L^2/\alpha (m \radnat^2 \radM^2) \le \frac{L^2}{\alpha \sqrt{\kappa}}(m \radnat^2 \radM^2).
\end{align*}
Thus, we can bound
\begin{align*}
\left(\frac{C_1}{\alpha \kappa^{\nicehalf}} + C_2\right) \lesssim \frac{\dimeta\dimy m  L^2\Rpinot^3 \radnat^3\radM^2}{\alpha \kappa^{\nicehalf}}.
\end{align*}
Thus, from \Cref{prop:redux_unknown} with $\nu = \nustar$, taking $c = 1$, and bounding $\radG \lesssim \Rpinot$,  $\radY \lesssim \radnat$ from \Cref{lem:par_bounds:unknown}
\begin{align*}
\ControlReg_T(\Alg; \Pistar) &\lesssim  \clam^{-1}\log(1+\frac{T}{\clam})\frac{\dimeta\dimy m  L^2\Rpinot^3 \radnat^3\radM^2}{\alpha \kappa^{\nicehalf}} \left(T\epsG^2 + h^2(\Rpinot^2 + \radnat)\right),. \\
&\quad+ L\Rpinot^3 (N+m)  \left(\dimu + \log(1/\delta) + \radM^4 \radnat^2 \right) + L\radM^3 \Rpinot^2 \radnat^2 T \epsG^2\left(1+  \frac{Lm  \Rpinot^2}{\nustar}\right)
\nonumber
\end{align*}
Using the above bounds we have $\nustar = \tfrac{\alpha \sqrt{\kappa}}{48(1+\radY)} \gtrsim \alpha \sqrt{\kappa}/\radnat$. Thus, for $L \ge \alpha$ and $\kappa \le 1$, the term $\frac{Lm  \Rpinot^2}{\nustar}$ dominates $1$, and we have
\begin{align*}
L\radM^3 \Rpinot^2 \radnat^2 T \epsG^2\left(1+  \frac{Lm  \Rpinot^2}{\nustar}\right) \lesssim \frac{L^2 \radM^3 \Rpinot^4 \radnat^3 m}{\alpha \sqrt{\kappa}}
\end{align*}
Moreover, using $N \ge m$ by assumption and aggregating terms and simplifying
\begin{align*}
\ControlReg_T(\Alg; \Pistar) &\lesssim  \clam^{-1}\log(1+\frac{T}{\clam})\frac{\dimeta\dimy  mL^2\Rpinot^4 \radnat^3\radM^3}{\alpha \kappa^{\nicehalf}} \left(T\epsG^2 + h^2(\Rpinot^2 + \radY)\right),. \\
&\quad+ L\Rpinot^3 N  \left( \dimu + \log(1/\delta) + c\radM^4 \radnat^2 \right).
\end{align*}
Next, recall $\dmax := \max\{\dimu + \dimy, \dimeta\}$, let us take $N = h^2\sqrt{T}\dmax $. From   \Cref{lem:ls_estimation_guarantees}, this yields $\epsG^2 = \frac{h^4 \radnat^2}{N} C_{\delta}^2 \eqsim \frac{h^4 \radnat^2 (\dmax + \log(1/\delta))}{N} = \radnat^2\log(1/\delta)/\sqrt{T}$ and that $\epsG^2 \ge \sqrt{T}$. This yields
\begin{align*}
\ControlReg_T(\Alg; \Pistar) &\lesssim  \clam^{-1}\log(1+\frac{T}{\clam})\frac{\dimeta\dimy  mL^2\Rpinot^4 \radnat^3\radM^3}{\alpha \kappa^{\nicehalf}} \left(\sqrt{T}\radnat^2\log(1/\delta) + h^2(\Rpinot^2 + \radnat)\right),. \\
&\quad+ L\Rpinot^3 h^2 \sqrt{T} \dmax\left( \dimu +  \log(1/\delta) + \radM^4 \radnat^2 \right)
\end{align*}
Finally, we us bound $L\Rpinot^3 h^2 \sqrt{T} \dmax\left( \dimu +  \log(1/\delta) + \radM^4 \radnat^2 \right) \le L\radM^4 \radnat^2 \Rpinot^3 h^2 \log(1/\delta)\dimu$, and take $L \le L^2/\alpha \le L^2/\alpha\sqrt{\kappa}$. Thus, we can bound the above by 
\begin{align*}
\ControlReg_T(\Alg; \Pistar) &\lesssim  \clam^{-1}\log(1+\frac{T}{\clam})\frac{\dimeta\dimy  (m+h^2)L^2\Rpinot^4 \radnat^3\radM^4}{\alpha \kappa^{\nicehalf}} \left(\sqrt{T}\radnat^2\log(1/\delta) + \Rpinot^2 \right).
\end{align*}
Finally, for $\lambda = \radnat^2\log(1/\delta)\sqrt{T} + h\Rpinot^2$, we can take $\clam \eqsim 1$. Together with $m+h^2 \le mh$ under the present assumption, we conclude
\begin{align*}
\ControlReg_T(\Alg; \Pistar) &\lesssim  \log(1+T)\frac{(\dimeta \dimy + \dmax \dimu)  mh L^2\Rpinot^4 \radnat^3\radM^4}{\alpha \kappa^{\nicehalf}} \left(\sqrt{T}\radnat^2\log(1/\delta) + \Rpinot^2 \right).
\end{align*}
Finally, we require $N \ge  1764 (\dmax + \dimu +\log  (1/\delta))^2h^4 \radM^2 \Rpinot^2 + c_0 h^2 \dimu^2.$, which means for our choice of $N = h^2\sqrt{T}\dmax$ and $\dmax \ge \dimu$, our stipulation that $\sqrt{T} \ge 4\cdot 1764 h^2 \radM^2 \Rpinot^2 + c_0 h^2 \dimu^2$ suffices. This ensures in turn that $\sqrt{T}\radnat^2\log(1/\delta)$ dominates $\Rpinot^2$, allowing us to drop the term from the final bound, ultimately yields
\begin{align*}
\ControlReg_T(\Alg; \Pistar) &\lesssim  \log(1+T)\frac{(\dimeta \dimy + \dmax \dimu) mh L^2\Rpinot^4 \radnat^5\radM^4\sqrt{T}\log(1/\delta)}{\alpha \kappa^{\nicehalf}}.
\end{align*}
Finally, using $\dmax = \max\{\dimeta, \dimy +\dimu\}$, we have $(\dimeta \dimy + \dmax \dimu) \le \dimeta(\dimy + \dimu) + \dimu(\dimy + \dimu) = (\dimeta + \dimy)(\dimy + \dimu)$, concluding the bound.

\end{proof}

\subsubsection{Proof of \Cref{prop:redux_known} \label{sssec:prop_redux_unknown}}
\newcommand{\Mbar}{\overline{M}}

Recall that $\fbar_t,\Fbar_t$ losses from \Cref{defn:loss_analysis}. In a fixed a comparator matrix $\Mbar \in \calM$, where we recall $\calM = \Mdrc(m,\radM)$, where $\radM \ge 2\Rcomp$ and $m \ge 2\mcomp - 1 + h$. $\Mbar$ will be chosen towards the proof in a careful way, and is not necessarily the best-in-hindsight parameter on the $\Mbar$ sequence. Our regret decomposition is as follows:
\begin{align*}
&\ControlReg_T(\Alg; \Pistar) = \sum_{t=1}^T \cost_t(\yalg_t,\ualg_t) - \inf_{\pi \in \Pistar} \sum_{t=1}^T\cost_t(\matypi_t,\matupi_t)\\
&\quad\le \underbrace{\sum_{t=1}^{N+m+2h} \cost_t(\yalg_t,\ualg_t) }_{(i)} + \underbrace{\sum_{t=N+m+2h+1}^T |\cost_t(\yalg_t,\ualg_t) - \Fbar_t(\matM_{t:t-h} \midhateta[t])|}_{(ii.a)} \\
&\quad+   \underbrace{\sum_{t=N+m+2h+1}^T \Fbar_t(M_{t:t-h} \midhateta[t] ) -  \sum_{t=N+m+2h+1}^T \fbar_t(\Mbar \midhateta[t])}_{(iii)}\\
&\quad+   \underbrace{ \sum_{t=N+m+2h+1}^T f_t(\Mbar \midhateta[t]) - \inf_{M' \in \calM_{\star}} \sum_{t=N+m+2h+1}^T \fbar_t(M' \midnateta[t]) }_{(iv)}\\
&\quad+  \underbrace{\max_{M \in \calM_{\star}} |\sum_{t=1}^T \fbar_t(M \midnateta[t]) - \cost_t(\maty^M_t,\matu^M_t) |}_{(ii.b)} + \underbrace{\left|\inf_{M \in \calM_{\star}} \sum_{t=1}^T\cost_t(\maty^M_t,\matu^M_t) -  \inf_{\pi \in \Pistar} \sum_{t=1}^T\cost_t(\matypi_t,\matupi_t)\right|}_{(v)}.
\end{align*} 
Again, let us work term-by-term, starting with the terms which are most similar to the terms that arise in the known system. Together with $\radM \ge \radcomp$, the last two terms can be bounded via \Cref{eq:truncation} and \Cref{eq:policy_approx}
\begin{align*}
(ii.b)+(v) \lesssim LT \radM^2\Rpinot^2\radnat^2\left( \frac{\psicomp(\mcomp)}{\radcomp} + \frac{2\psipinot(h+1)}{\Rpinot}\right).
\end{align*}
Moreover, similar arguments can be used to bound $(ii.a) \lesssim \text{RHS of  \Cref{eq:truncation}}$ (specifically, one replaces the appearance of $\etanat_t$ in the proof \citet[Lemma 5.3]{simchowitz2020improper} with $\etanathat_{t}$, and uses the bound $\|\etanathat_t\| \le 2\radnat$ by \citet[Lemma 6.1]{simchowitz2020improper} ). Thus, we have so far   
\begin{align*}
(ii.a) + (ii.b)+(v) \lesssim LT \radM^2\Rpinot^2\radnat^2\left( \frac{\psicomp(\mcomp)}{\radcomp} + \frac{2\psipinot(h+1)}{\Rpinot}\right).
\end{align*}
Next, analogously to \Cref{eq:recover_policy_regret}, we recognize that
\begin{align*}
(iii) &= \MemRegHat_T(\zst), \text{ for } \zst := \embed(\Mbar).
\end{align*}
Furthermore, from \citet[Lemma 6.3]{simchowitz2020improper} and the definition of the term $\Rubar$ in \citet[Lemma 6.1b]{simchowitz2020improper}, and with $N \ge m + 2h$, we have $(i) \lesssim LN \Rpinot^2(\Restu + \radM \radnat)^2 $. Thus, collecting what we have thus far, we obtain
\begin{align*}
&\ControlReg_T(\Alg; \Pistar) \\
&\le \MemRegHat_T(\zst) + (iv)\\
&\quad+ \bigohconst{1}\cdot L \Rpinot^2 \cdot \left(N (\Restu + \radM \radnat)^2 + T \radM^2\radnat^2\left( \frac{\psicomp(\mcomp)}{\radcomp} + \frac{\psipinot(h+1)}{\Rpinot}\right)\right),
\end{align*}
where $\bigohconst{1}$ supresses a universal constant. It remains to account for the term $(iv)$. In particular, for $\psipinot(h+1) \le c\Rpinot/T$ and $\psicomp(\mcomp) \le c\radcomp /T$, the above simplies to 
\begin{align}
\ControlReg_T(\Alg; \Pistar) &\le \MemRegHat_T(\zst) + (iv)\nonumber\\
&\quad+ \bigohconst{L}\Rpinot^2 \cdot \left((N+c) (\Restu^2 + \radM^2 \radnat^2) \right), \label{eq:first_couple_terms}
\end{align}

\begin{lemma}[Slight Modification of Equation E.6 in \citet{simchowitz2020improper}, altering numerical constants and allowing $c$ dependence] Suppose that $\eventest$ holds, and that $\psipinot(h+1) \le c\Rpinot/T $. Futher, assume $\radM \ge 2\radcomp$ and $m \ge 2\mcomp + h$. Then, there exists an $\Mbar \in \calM$ such that, for all $\nu > 0$, we have
\begin{align}
	\text{\normalfont Term }(iv) &\le  \bigohconst{1}\cdot L\radM^3 \Rpinot^2 \radnat^2 T \epsG^2\left(1+  \frac{Lm  \Rpinot^2}{\nu}\right)  \\
	&+ \bigohconst{c}L\radM^2 \Rpinot^2 \radnat ((\Restu + \radM \radnat) \Rpinot + m)
	\nonumber \\
	&\quad + \nu \sum_{t=N+m+2h+1}^T  \left\|\uex_j( \matM_j   \midhateta[j]) -  \uex_j( \Mbar   \midhateta[j])\right\|_{2}^2 \label{eq:natyhat_hat_nearly_done}.
	\end{align}
\end{lemma}
Absorbing the first $h$ terms in the sum into the term on the first line (using arguments as in \Cref{lem:parameter_bounds_known}, this contributes $\bigohconst{\radM^2 \radnat^2h}\le \bigohconst{\radM^2 \radnat^2m}$ ), and translating back to our $\matY,z$-notation, we have that there exists a $\zst \in \calM_{\embed}$ such that
\begin{align*}
	\text{\normalfont Term }(iv) &\le  \bigohconst{1}\cdot L\radM^3 \Rpinot^2 \radnat^2 T \epsG^2\left(1+  \frac{Lm  \Rpinot^2}{\nu}\right)  \\
	&+ \bigohconst{c}\radM^2 \Rpinot^2 \radnat ((\Restu + \radM \radnat)\Rpinot + m)\\
	&\quad + \nu \sum_{t=N+m+h+1}^T  \left\|\matY_t(\matz_t - \zst)\right\|_{2}^2.
	\end{align*}
	Putting things together with \Cref{eq:first_couple_terms}, we have the bound that for $\psipinot(h+1) \le \Rpinot/T $ and $\psicomp \le \radcomp/T$, we find
	\begin{align*}
&\ControlReg_T(\Alg; \Pistar) \\
&\le \MemRegHat_T(\zst) + \nu \sum_{t=N+m+2h+1}^T  \left\|\matY_t(\matz_t - \zst)\right\|_{2}^2 \\
&\quad+ \bigohconst{1}\cdot L\radM^3 \Rpinot^2 \radnat^2 T \epsG^2\left(1+  \frac{Lm  \Rpinot^2}{\nu}\right)\\
&\quad + \bigohconst{L} \Rpinot^2 \cdot \left((N+c) (\Restu^2 + \radM^2 \radnat^2)  + c\radM^2 \radnat ((\Restu + \radM \radnat)\Rpinot + m) \right)
\end{align*}
Finally, since $N \ge m$, we bound
 \begin{align*}
 &L \Rpinot^2 \cdot \left(N (\Restu^2 + c\radM^2 \radnat^2)  + \radM^2 \radnat ((\Restu + \radM \radnat)\Rpinot + m) \right) \\
 &\quad\le \bigohconst{L} \Rpinot^3  \left( (N+cm) (\Restu^2 + \radM^3 \radnat^2) +cm \Restu \radM^2 \radnat \right)\\
 &\quad\le \bigohconst{L} \Rpinot^3   (N+cm) (\Restu^2 + c\radM^4 \radnat^2) ,
 \end{align*}
 where the last step is by AM-GM. Thus,
\begin{align*}
&\ControlReg_T(\Alg; \Pistar)\\
 &\le \MemRegHat_T(\zst) +\, \nu \sum_{t=N+m+2h+1}^T  \left\|\matY_t(\matz_t - \zst)\right\|_{2}^2 \\
&\quad+ \bigohconst{L\Rpinot^3 (N+c)}   \left( \Restu^2 + \radM^4 \radnat^2 \right) + \bigohconst{ L\radM^3 \Rpinot^2 \radnat^2 T \epsG^2}\left(1+  \frac{Lm  \Rpinot^2}{\nu}\right)
\nonumber,
\end{align*}
which after substituing in $\Restu^2 \lesssim \dimu + \log(1/\delta)$ (\Cref{lem:ls_estimation_guarantees}),
concludes the bound.

\newpage
\part{Appendices for \ocom}


\section{Ommited Proofs from \Cref{sec:known} \label{app:known_proofs}}

\subsection{Proof of \Cref{prop:covariance_lb} \label{sec:prop:covariance_lb}}

\begin{proof}[Proof of \Cref{prop:covariance_lb}]
Let $v \in \R^{\dimu}$, with $\|v\| = 1$, and let $u_s = \matY_s v$ for $s \in \{1-h,2-h,\dots,t\}$, and set $u_s = 0$ for $s \le t - h$ and $s > t$. From \Cref{fact:vector_norm_lb}, which shows that $\|v+w\|_2^2 \ge \frac{1}{2}\|v\|^2 - \|w\|^2$, we have
\begin{align*}
v^\top \sum_{s=1}^t \matH_s^\top \matH_s v &:=  \sum_{s=1}^t \|\matH_s v\|_2^2 = \sum_{s=1}^t \left\|\sum_{i=0}^h G^{[i]} \matY_{s-i} v\right\|_2^2\\
&=  \sum_{s=1}^t \left\|\sum_{i=0}^h G^{[i]} u_{s-i} \right\|_2^2\\
&\ge  \sum_{s=1-h}^{t+h} \left\|\sum_{i=0}^h G^{[i]} u_{s-i} \right\|_2^2 - 2h \radG^2\radY^2.\\
&\overset{\normtext{(\Cref{fact:vector_norm_lb})}}{\ge}  \frac{1}{2}\sum_{s=1-h}^{t+h} \left\|\sum_{i=0}^{\infty} G^{[i]} u_{s-i} \right\|_2^2 - \sum_{s=1-h}^{t+h}\left\|\sum_{i> h}^{\infty} G^{[i]} u_{s-i} \right\|_2^2
 - 2h \radG^2\radY^2\\
 &\ge  \frac{1}{2}\sum_{s=1-h}^{t+h} \left\|\sum_{i=0}^{\infty} G^{[i]} u_{s-i} \right\|_2^2 - \sum_{s=1-h}^{t+h}\psiG(h+1)^2\radY^2
 - 2h \radG^2\radY^2\\
  &=  \frac{1}{2}\sum_{s=1-h}^{t+h} \left\|\sum_{i=0}^{\infty} G^{[i]} u_{s-i} \right\|_2^2 - \underbrace{\left(t\psiG(h+1)^2 + 4h\radG^2\right)\radY^2}_{:=\gamma_{t;h}},
\end{align*}
where we use $\psiG(h+1) \le \psiG(0) = \radG^2$ in the last line. \newcommand{\utild}{\tilde{u}} Moreover, setting $\utild_s = u_{s-h}$,
\begin{align*}
\sum_{s=1-h}^{t+h} \left\|\sum_{i=0}^{\infty} G^{[i]} u_{s-i} \right\|_2^2 &= \sum_{s=1}^{t+2h} \left\|\sum_{i=0}^{\infty} G^{[i]} \utild_{s-i} \right\|_2^2 \\
&\overset{(i)}{=} \sum_{s=1}^{\infty} \left\|\sum_{i=0}^{s} G^{[i]}\utild \right\|_2^2\\
&\overset{(ii)}{\ge} \kappa_0 \sum_{s=1}^{\infty} \|\utild_s\|_2^2\\
&= \kappa_0 \sum_{s=1}^{\infty} \|u_{s-h}\|_2^2
\end{align*}
where $(i)$ uses that we have $\utild_s = 0$ for $s \le 0$ and for $s \ge t+2h$, and $(ii)$ invokes \Cref{defn:input_recov}. Combining the two displays, we have
\begin{align*}
v^\top \sum_{s=1}^t \matH_s^\top \matH_s v  &\ge  \frac{\kappa_0}{2} \sum_{s=1}^{\infty} \|u_{s-h}\|_2^2 - \gamma_{t;h}\\
&\ge  \frac{\kappa_0 }{2}\sum_{s=1}^{t+h} \|\matY_{s-h}v\|_2^2 - \gamma_{t;h}\\
&=  v^\top\left(\frac{\kappa_0}{2} \sum_{s=1-h}^{t} \matY_{s}^{\top}\matY_{s} - \gamma_{t;h} I\right)v,
\end{align*}
where the last line uses $\|v\| = 1$. Finally, defining $\cpsi[t]:= \max\{1, \frac{t\psiG(h+1)^2}{h\radG^2}\}$, we have $\gamma_{t;h} = \radY^2(t\psiG(h+1)^2 + 4h\radG^2) \le \radY^2(h\cpsi[t]\radG^2 + 4h\radG^2) \le  5h \radH^2 \cpsi[t]$, yielding the desired bound.
\end{proof}

\subsection{Proof of \Cref{lem:semions_reg} \label{app:proof:ssec:online_semi_newton}}
Let $\zst \in \argmin_{z \in \calC}\sum_{t=1}^T f_t(z)$. Following the standard analysis of Online Newton Step (e.g. \citet[Chapter 4]{hazan2019introduction} with $\gamma \leftarrow 1/\eta$), one has
\begin{align*}
\sum_{t=1}^T \nabla_t(\matz_t - \zst) &\le \frac{\eta}{2}\sum_{t=1}^{T} \nabla_t^\top \Lambda_t^{-1}  \nabla_t + \frac{1}{2\eta}\sum_{t=1}^T(\matz_t - \zst)^\top(\Lambda_t - \Lambda_{t-1})(\matz_t - \zst) + \frac{1}{2\eta}(\matz_1 - \zst)^\top \Lambda_0 (\matz_1 - \zst)
\end{align*}
The last term is at most $\frac{\lambda}{2\eta} \diamz^2$. Moreover, since $\Lambda_t - \Lambda_{t-1} = \matH_t\matH_t^\top$,
\begin{align*}
\sum_{t=1}^T \nabla_t(\matz_t - \zst) - \frac{1}{2\eta}\|\matH_t(\matz_t - \zst)\|_2^2 \le\lambda \diamz^2 +  \frac{\eta}{2}\sum_{t=1}^{T} \nabla_t^\top \Lambda_t^{-1}  \nabla_t.
\end{align*}
Finally, for $\eta \ge \frac{1}{\alpha}$, we recognize that $\nabla_t(\matz_t - \zst) - \frac{1}{2\eta}\|\matH_t(\matz_t - \zst)\|_2^2 \ge \nabla_t(\matz_t - \zst) - \frac{\alpha}{2}\|\matH_t(\matz_t - \zst)\|_2^2 \ge f_t(\matz_t) - f_t(\zst)$ by \Cref{lem:quad_lb}. Thus,
\begin{align*}
\sum_{t=1}^T f_t(\matz_t)  - f_t(\zst) \le\lambda \diamz^2 +  \frac{\eta}{2}\sum_{t=1}^{T} \nabla_t^\top \Lambda_t^{-1}  \nabla_t,
\end{align*}
as needed. \qed

\subsection{Proof of \Cref{lem:Ft_diff}}

	\newcommand{\matzbar}{\overline{\matz}}
	We  have $F_t(\matz_t,\dots,\matz_{t-h}) - f_t(\matz_t) = F_t(\matz_t,\dots,\matz_{t-h}) - F_t(\matz_t,\dots\matz_t)$. Therefore Taylor's theorem, there exists some $\mu \in [0,1]$ such that, for $\matzbar_{t-i} = \mu \matz_{t-i} + (1-\mu)\matzbar_t$,
	\begin{align*}
	F_t(\matz_t,\dots,\matz_{t-h}) - f_t(\matz_t) &= (\nabla F_t(\matzbar_t,\dots,\matzbar_{t-h}))^\top (0,\matz_{t-1}-\matz_t,\matz_{t-2}-\matz_t,\dots,\matz_{t-h}-\matz_t).
	\end{align*}
	By the Chain Rule, we then have
	\begin{align*}
	|F_t(\matz_t,\dots,\matz_{t-h}) - f_t(\matz_t)| &=\left| \nabla \ell( \matv_t + \sum_{i=0}^h G^{[i]}\matY_{t-i}\matzbar_t)^\top \left(\sum_{i=1}^h G^{[i]}\matY_{t-i}(\matz_{t-i} - \matz_t)\right)\right|\\
	&\le \|\nabla \ell( \matv_t + \sum_{i=0}^h G^{[i]}\matY_{t-i}\matzbar_t)\|_2 \cdot \radG \cdot \max_{i \in \{1,\dots,h\}}\|\matY_{t-i}(\matz_{t-i} - \matz_t)\|_2.
	\end{align*}
	Analogous to the \Cref{lem:ft_facts}, we have $\|\nabla \ell( \matv_t + \sum_{i=0}^h G^{[i]}\matY_{t-i}\matzbar_t)\|_2 \le \Leff$, concluding the first part of the proof. For the second display, we have
	\begin{align*}
	\sum_{t=1}^{T} F_t(\matz_t,\dots,\matz_{t-h}) - f_t(\matz_t) &\le \Leff \radG \sum_{t=1}^T \max_{i\in \{1,\dots,h\}} \|\matY_{t-i}(\matz_t - \matz_{t-i})\|_2\\
	&\le \Leff \radG \sum_{t=1}^T \sum_{i=1}^h\|\matY_{t-i}(\matz_t - \matz_{t-i})\|_2\\
	&\le \Leff \radG \sum_{t=1}^T \sum_{i=1}^h\sum_{j=1}^{i-1}\|\matY_{t-i}(\matz_{t-j+1} - \matz_{t-j})\|_2\\
	&= \Leff \radG \sum_{s=1-h}^T \sum_{i=1}^h \sum_{j=1}^{i-1}\|\matY_{s}(\matz_{s+i-j+1} - \matz_{s+i-j})\|_2\\
	&\le h \Leff \radG \sum_{s=1-h}^T \sum_{i=1}^{h-1} \|\matY_{s}(\matz_{s+i+1} - \matz_{s+i})\|_2 \cdot \I_{s+i+1 \le t}.
	\end{align*}
	Finally, since $\matz_t - \matz_{t-1} =0$ for $t \le 1$, the above indicator $\I_{s+i+1 \le t}$ can be replaced with $\I_{2 \le s+i+1 \le t} = \I_{1 \le s+i \le t-1}$, completing the proof.\qed

\subsection{Proof of \Cref{lem:Y_bound}}
	For $t \le 0$, $\|\matY_s (\matz_{t+1}-\matz_{t})\|_2  = 0$. Otherwise, we have
	\begin{align}
	\|\matY_s (z_{t}-z_{t-1})\|_2 &= \|\matY_s \Lambda_t^{-1/2}\Lambda_t^{1/2}(\matz_{t+1}-\matz_{t})\|_2 \nonumber\\
	&\le \|\matY_s \Lambda_t^{-1/2}\|_{\op} \cdot \|\Lambda_t^{1/2}(\matz_{t+1}-\matz_{t})\|_2 \nonumber\\
	&\overset{(i)}{\le} \|\matY_s \Lambda_t^{-1/2}\|_{\op} \cdot \|\Lambda_t^{1/2}(\widetilde{\matz}_{t+1}-\matz_{t}) \|_2 \nonumber\\
	&= \|\matY_s \Lambda_t^{-1/2}\|_{\op} \|\Lambda_t^{1/2}\cdot \eta \Lambda_{t}^{-1}\nabla_t\|_2 \nonumber\\
	&= \eta \sqrt{\|\matY_s \Lambda_t^{-1/2}\|_{\op}^2 \| \Lambda_{t}^{-1/2}\nabla_t\|_2^2} \label{eq:operator_norm_bound_movement},
	\end{align}
	where $(i)$ follows from the Pythagorean theorem, using that $\matz_{t+1}$ is projected in the $\Lambda_t$-norm. 
	Finally, we can crudely bound $\|\matY_s \Lambda_t^{-1}\matY_s\|_{\op} \le \trace(\matY_s \Lambda_t^{-1}\matY_s)$. Since we consider indices $t \ge s$, we have $\trace(\matY_s \Lambda_t^{-1}\matY_s)\le \trace(\matY_s \Lambda_s^{-1}\matY_s)$, where we have the understanding that $\Lambda_s = \Lambda_1$ for $s \le 0$. Thus, we see that for $t > 0$,
	\begin{align*}
	\|\matY_s (\matz_{t+1}-\matz_{t})\|_2 \le \eta  \trace(\matY_s \Lambda_s^{-1}\matY_s)^{\nicehalf} \trace(\nabla_t^\top \Lambda_t^{-1}\nabla_t))^{\nicehalf}
	\end{align*}
	Thus, from \Cref{lem:Ft_diff} and by Cauchy Schwartz,
	\begin{align*}
		\MoveDiff_T &\le  h \Leff \radG \sum_{i=1}^{h-1}\sum_{s=1-h}^T  \|\matY_{s}(\matz_{s+i+1} - \matz_{s+i})\|_2 \I_{1 \le s+i \le t-1}\\
		&\le \eta h\Leff \radG \cdot \sum_{i=1}^{h-1} \sqrt{\sum_{s=1-h}^{T} \I_{1 \le s+i \le t-1} \cdot \trace(\matY_s \Lambda_s^{-1}\matY_s)}\sqrt{\sum_{s=1-h}^{T} \I_{1 \le s+i \le t-1} \cdot \trace(\nabla_{s+i }^\top \Lambda_{s+i }^{-1}\nabla_{s+i })}\\
		&\le \eta h^2\Leff \radG \cdot \sqrt{\sum_{s=1-h}^{T}  \trace(\matY_s \Lambda_h^{-1}\matY_s)}\sqrt{\sum_{s=1}^{T} \trace(\nabla_t^\top \Lambda_t^{-1}\nabla_t)},
		\end{align*}
		as needed.
	\qed


\section{Ommited Proofs from \Cref{sec:unknown} \label{app:unknown_proofs}}
\subsection{Useful Facts for Analysis}
We begin by listing some useful elementary facts:
\begin{fact}\label{fact:epsX} For all $t \ge 1$ and all $z \in \calC$, we have $\|\matH_t -\matHhat_t\|_\op \le \epsG \radY$ and $\|(\matH_t -\matHhat_t)z\|_\op \le \epsG \radyc$
\end{fact}
\begin{proof} $\|\matH_t -\matHhat_t\|_\op = \|\sum_{i=0}^{h}(\Gst^{[i]}-\Ghat^{[i]})\matY\|_{\op} \le \|\radY \|_{\op}\sum_{i=0}^{h}\|\Gst^{[i]}-\Ghat^{[i]}\|_{\op} \le \epsG \radY$. The second bound is similar.
\end{proof}
\begin{fact}\label{fact:vector_norm_lb} Given two vectors $v,w\in \R^m$, $\|v+w\|_2^2 \ge \frac{1}{2}\|v\|^2 - \|w\|^2$.
		\end{fact}
		\begin{proof} $\|v+w\|_2^2 = \|v\|^2 +\|w\|^2 + 2\langle v, w \rangle \ge \|v\|^2 +\|w\|^2 - 2\|v\|\|w\| \ge \|v\|^2 +\|w\|^2 - \frac{1}{2}\|v\|^2 - 2\|w\|^2 = \frac{\|v\|^2}{2} - \|w\|^2$, as needed.
		\end{proof}

 \begin{fact}\label{fact:square_perturb} $\|a\|^2 \le \|b\|^2 + (\|a\|+\|b\|)\|b-a\|$
 \end{fact}
 \begin{proof} $\|a\|_2^2 = \langle a, a \rangle = \langle b - a, a \rangle + \langle b, a \rangle = \langle b - a , a + \langle b , a - b \rangle + \|b\|^2$. The bound now follows form Cauchy-Schwartz
 \end{proof}
\subsection{Proof of \Cref{lem:unknown_regret_decomp}}
Let $\zst \in \calC$ be an arbitrary comparator point.  Analogus to the proof of Lemma~\ref{lem:semions_reg},
\begin{align}
\sum_{t=1}^T \fhat_t(\matz_t) - \fhat_t(\zst) \le \sum_{t=1}^T \nabhat_t^\top (\matz_t - \zst) - \frac{\alpha}{2}\|\matHhat_t(\matz_t- \zst)\|_{2}^2 \label{eq:reg_one_bound}
\end{align}

One the other hand, the standard inequality obtained from applying $\semions$ to the $(\fhat_t)$-sequence (see, for analogy, page 58 of \cite{hazan2019introduction}), we obtain
\begin{align*}
\nabhat_t^\top (\matz_t - \zst) \le \frac{\eta}{2} \|\nabhat\|_{\Lamhat_t^{-1}}^2 + \frac{2}{\eta}\|\matz_t - \zst\|_{\Lamhat_t}^2 -\frac{2}{\eta}\|\matz_{t+1} - \zst\|_{\Lamhat_{t}}^2 .
\end{align*}
Summing up over $t$ and telescoping
\begin{align}
\sum_{t=1}^T\nabhat_t^\top (\matz_t - \zst) &\le \frac{\eta}{2}  \sum_{t=h+1}^T\|\nabhat\|_{\Lamhat_t^{-1}}^2  + \sum_{t=1}^T\frac{1}{2\eta}\|\matz_t - \zst\|_{\Lamhat_t - \Lamhat_{t-1}}^2 +\frac{1}{2\eta}\|\matz_{h} - \zst\|_{\Lamhat_{h}}^2 \nonumber\\
&= \frac{\eta}{2}  \sum_{t=1}^T\|\nabhat\|_{\Lamhat_t^{-1}}^2  + \frac{1}{2\eta}\sum_{t=1}^T\|\matHhat_t(\matz_t - \zst)\|^2 +\frac{\lambda \diamz^2}{2\eta} \label{eq:reg_two_bound},
\end{align}
where we use $\Lamhat_t - \Lamhat_{t-1} = \matHhat_t^\top \matHhat_t$ and $\Lamhat_0 = \lambda I$. Thus, introducing $\err_t := \nabla \fhat_t(z) - \nabla f(\matz_t)$ and combining \eqref{eq:reg_one_bound} and \eqref{eq:reg_two_bound},
\begin{align*}
\sum_{t=1}^T f_t(\matz_t) - f_t(\zst) &\le \sum_{t=1}^T \err_t^\top (\matz_t - \zst)   + \frac{1}{2\eta}\sum_{t=1}^T(\|\matHhat_t(\matz_t - \zst)\|^2 - \eta \alpha \|\matH_t(\matz_t - \zst)\|^2) \\
&\qquad+ \frac{\eta}{2}  \sum_{t=1}^T\|\nabhat\|_{\Lamhat_t^{-1}{}}^2 +\frac{\lambda \diamz^2}{2\eta}
\end{align*}
Plugging in $\matdel_t = \matz_t - \zst$ concludes the proof, and re-iterating the proof of \Cref{thm:semions} concludes the proof.

\subsection{Proof of \Cref{lem:Xhat_cancel}}
	First, we can bound $\|\matHhat_t\matdel_t\|^2 \le 2\|\matH_t\matdel_t\|^2 + 2\|(\matH_t-\matHhat_t)\matdel_t\|^2$, and 
	\begin{align*}
	\|(\matH_t-\matHhat_t)\matdel_t\| \le \|(\matH_t-\matHhat_t)\matz_t\| + \|(\matH_t-\matHhat_t)\zst\| \le 2\radyc \epsG
	\end{align*} 
	by \Cref{fact:epsX}. Taking $\eta \ge \frac{3}{\alpha}$, we find then that 
	\begin{align*}
	\|\matHhat_t \matdel_t\|^2 - \eta \alpha\|\matH_t\matdel_t\|^2 \le 2\|\matH_t\matdel_t\|^2 + 8\radyc^2 \epsG^2 - 3\|\matH_t\matdel_t\|^2= - \|\matH_t\matdel_t\|^2 + 8\radyc^2 \epsG^2. \end{align*}
	The second statement of the lemma follows by substitution into \Cref{lem:unknown_regret_decomp}.
	\qed

\subsection{Proof of \Cref{lem:grad_err}}
We have the bound
\begin{align*}
\err_t &:= \nabla \fhat_t(z) - \nabla f(\matz_t)\\
 &= \matHhat_t^\top \nabla \loss_t( \matvhat_t + \matHhat_t \matz_t) - \matH_t  \nabla \loss_t( \matvst_t + \matH_t \matz_t)\\
&= (\matHhat_t - \matH_t)^\top \nabla \loss_t( \matvhat_t + \matHhat_t \matz_t) + \matH_t  \left(\nabla \loss_t( \matvhat_t + \matHhat_t \matz_t)  - \nabla \loss_t( \matvst_t + \matH_t \matz_t)\right).
\end{align*}
Defining
\begin{align*}
g_{t,1} &:= \nabla \loss_t( \matvhat_t + \matHhat_t \matz_t)\\
g_{t,2} &:= \left(\nabla \loss_t( \matvhat_t + \matHhat_t \matz_t)  - \nabla \loss_t( \matvst_t + \matH_t \matz_t)\right)
\end{align*}
We have that $\|g_{t,1}\|_2 \le \Leff$ by analogy to \Cref{lem:ft_facts}. Moreover, since $\beta$-smoothness implies that the gradients are $\beta$-Lipschitz, and by invoking \Cref{fact:epsX}, we have
\begin{align*}
\left(\nabla \loss_t( \matvhat_t + \matHhat_t \matz_t)  - \nabla \loss_t( \matvst_t + \matH_t \matz_t)\right) \le \beta \|(\matvhat_t + \matHhat_t \matz_t)  -   (\matvst_t + \matH_t \matz_t)\| 
&\le \beta(\cv \epsG + 2\epsG \radyc).
\end{align*}
\qed

\subsection{Proof of \Cref{lem:Errbar_reg}}

Recall that from \Cref{lem:Xhat_cancel}, we have the bound 
\begin{align}
		\sum_{t=1}^T f_t(\matz_t) - f_t(\zst) &\le \sum_{t=1}^T \err_t^\top \matdel_t   - \frac{1}{2\eta}\sum_{t=1}^T\|\matH_t\matdel_t\|^2 + \frac{4}{\eta}T\radyc^2 \epsG^2 + \Reghat_T. \label{eq:lem_Xhat_cancel_recall}
\end{align}

Let us now bound the sum  $\sum_{t=1}^T \err_t^\top \matdel_t$ via \Cref{lem:grad_err}. The lemma ensures $\err_t  = (\matHhat_t - \matH_t)^\top g_{1,t} + \matH_t^\top \, g_{2,t}.$ where $\|g_{1,t}\|_2 \le \Leff$ and $\|g_{2,t}\| \le\beta\epsG(\cv + 2 \radyc) $.  The contribution of the term including $g_{2,t}$ is easily adressed: 
\begin{align*}(\matH_t^\top g_{2,t})^\top\matdel_t \le \|g_{2,t}\|_2 \|\matH_t\matdel_t\|_2 \le \beta\epsG(\cv + 2 \radyc)\|\matH_t\matdel_t\|_2 \le \eta \beta^2\epsG^2(\cv + 2 \radyc)^2 +  \frac{1}{4\eta}\|\matH_t\matdel_t\|_2,
\end{align*} 
by the AM-GM inequality. Next, we handle the term $(\matHhat_t - \matH_t)^\top g_{1,t}$. First we bound
\begin{align*}((\matHhat_t - \matH_t)^\top g_{1,t})^\top \matdel_t \le \|g_{1,t}\| \|(\matHhat_t - \matH_t)\matdel_t\| \le \Leff  \|(\matHhat_t - \matH_t)\matdel_t\|.
\end{align*} 
Plugging into \Cref{eq:lem_Xhat_cancel_recall} gives
 \begin{align}
		\sum_{t=1}^T f_t(\matz_t) - f_t(\zst) &\le \sum_{t=1}^T \Leff\|(\matHhat_t - \matH_t) \matdel_t\|   - \frac{1}{4\eta}\sum_{t=1}^T\|\matH_t\matdel_t\|^2 \nonumber\\
		&\qquad+  T\left(\eta \beta^2(\cv + 2\epsG \radyc)^2 + \frac{4\radyc^2}{\eta}\right)\epsG^2  + \Reghat_T. \label{eq:lemm_errbar_reg}
		\end{align}
		For arbitrary sequences $\matH_t,\matHhat_t$, there is no obvious way to cancel the terms $\Leff\|(\matHhat_t - \matH_t) \matdel_t\|$ and $-\|\matH_t\matdel_t\|^2$ to achieve a $\bigohst{T\epsG^2}$-error dependence. However, there is additional structure we can leverage. We can observe that 
		\begin{align*}
		\|(\matHhat_t - \matH_t) \matdel_t\|_{2}^2 = \left\|\sum_{i=0}^h (\Ghat^{[i]}-\Gst)^{[i]} \matY_{t-i}\matdel_t\right\|_{2}^2 \le \epsG \max_{i\in [0:h]}\|\matY_{t-i}\matdel_t\|^2.
		\end{align*}
		Hence, by AMG-GM, we have that for any $\nu > 0$, 
		\begin{align*}
		\Leff\|(\matHhat_t - \matH_t) \matdel_t\| \le \nu^{-1} (h+1)\eta \Leff^2 \epsG^2 + \frac{\nu}{4(h+1)\eta}\max_{i\in [0:h]}\|\matY_{t-i}\matdel_t\|^2.
		\end{align*}
		Together with \Cref{eq:lemm_errbar_reg}, the above display implies
		\begin{align*}
		\sum_{t=1}^T f_t(\matz_t) - f_t(\zst) &\le   \frac{1}{4\eta}\sum_{t=1}^T\left( \frac{\nu}{h+1} \sum_{i=0}^h\|\matY_{t-i}\matdel_t\|^2 - \|\matH_t\matdel_t\|^2\right) +  T\epsG^2 \cdot\,\Err(\nu)  + \Reghat_T,
		\end{align*}
		where $\Err(\nu) := \left(\frac{\eta (h+1) \Leff^2 }{\nu}  + \eta\beta^2(\cv + 2 \radyc)^2 + \frac{4\radyc}{\eta}\right)$.
		\qed.

\subsection{Proof of \Cref{lem:unknown_block}}
Fix a block length $\tau \in \N$, and recall the index $k_j = (j-1)\tau$, and $\jmax$ as the largest $j$ such that $\jmax \tau \le T$. We bound
 \begin{align}
 &\sum_{t=1}^T \|\matY_{t-i}\matdel_t\|_2^2  \nonumber\\
 &= \sum_{j=1}^{\jmax}\sum_{s=1}^{\tau} \|\matY_{k_j+s - i}\matdel_{k_j+s}\|_2^2 + \sum_{s  = 1 + \tau(\jmax - 1)}^T\|\matY_{i-h}\matdel_t\|_2^2\nonumber\\
  &\le  4\tau \radyc +  \sum_{j=1}^{\jmax}\sum_{s=1}^{\tau} \|\matY_{k_j+s - i}\matdel_{k_j+s}\|_2^2 \nonumber\\
  &\overset{(i)}{\le}  4\tau \radyc + \sum_{j=1}^{\jmax}\sum_{s=1}^{\tau} \|\matY_{k_j + s  - i}\matdel_{k_j}\|_2^2  + (\|\matY_{k_j + s  - i}\matdel_{k_j+1}\|_2 + \|\matY_{k_j + s  - i}\matdel_{k_j+s}\|_2)\|\matY_{k_j + s  - i}(\matdel_{k_j+s} - \matdel_{k_j+1})\|_2\nonumber\\
  &\overset{(ii)}{\le}  4\tau \radyc + \sum_{j=1}^{\jmax}\sum_{s=1}^{\tau} \|\matY_{k_j + s  - i}\matdel_{k_j+1}\|_2^2   + 4\radyc\sum_{j=1}^{\jmax}\sum_{s=1}^{\tau} \|\matY_{k_j + s  - i}(\matdel_{k_j+s} - \matdel_{k_j+1})\|_2, \label{eq:unknown_block_first_eq}
 \end{align}
Where $(i)$ uses the inequality $\|a\|^2 \le \|b\|^2 + (\|a\|+\|b\|)\|b-a\|$ from \Cref{fact:square_perturb}, and where $(ii)$ uses the $\|\matY_{s}(\matdel_{t})\| \le \|\matY_s \zst \| + \|\matY_s \matz_t\| \le 2\radyc$.  

Next, recalling $\matdel_t := \matz_t - \zst$, we develop
 \begin{align*}
 \sum_{j=1}^{\jmax}\sum_{s=1}^{\tau} \|\matY_{k_j + s  - i}(\matdel_{k_j+s} - \matdel_{k_j+1})\|_2 &=  \sum_{j=1}^{\jmax}\sum_{s=2}^{\tau} \|\matY_{k_j + s  - i}(\matz_{k_j+s} - \matz_{k_j+1})\|_2 \\
 &\le \sum_{j=1}^{\jmax}\sum_{s=2}^{\tau}\sum_{s' = 0}^{s-2} \|\matY_{k_j + s  - i}(\matz_{k_j+s - s'} - \matz_{k_j - s' - 1})\|_2 \\
 &\le \sum_{j=1}^{\jmax}\sum_{s=2}^{\tau}\sum_{s' = 0}^{\tau'-1} \|\matY_{k_j + s  - i}(\matz_{k_j+s - s'} - \matz_{k_j - s' - 1})\|_2 \\
 &\le \sum_{t=1}^T\sum_{s' = 0}^{\tau -1} \|\matY_{t  - i}(\matz_{t-s'} - \matz_{t - s'-1 })\|_2,
 \end{align*}
where above we use the convention $\matz_t = 0$ for $t \le 1$, and that the induces $k_j + s$ range over a subset of $ t \in [T]$. Relabeling $s'$ with $s$, and combining with \Cref{eq:unknown_block_first_eq} this finally yields
\begin{align*}
\sum_{t=1}^T \|\matY_{t-i}\matdel_t\|_2^2  \ge 4\tau \radyc + \sum_{j=1}^{\jmax}\sum_{s=1}^{\tau} \|\matY_{k_j + s  - i}\matdel_{k_j}\|_2^2   + 4\radyc \sum_{t=1}^T\sum_{s = 0}^{\tau -1} \|\matY_{t  - i}(\matz_{t-s} - \matz_{t - s-1 })\|_2.
\end{align*}
Following similar steps (but using \Cref{fact:epsX} to bound $\|\matH_t z\| \le \radG\radyc$), we obtain 
\begin{align*}
\sum_{t=1}^T \|\matH_t \matdel_t\|_2^2 \ge  \sum_{j=1}^{\jmax}\sum_{s=1}^{\tau} \|\matH_{k_j + s }\matdel_{k_j}\|_2^2   - 4\radyc \radG \sum_{t=1}^T\sum_{s = 0}^{\tau -1} \|\matH_{t }(\matz_{t-s} - \matz_{t - s-1 })\|_2,
\end{align*}
\qed
 \subsection{Proof of \Cref{lem:ctau}}
Recall our convention $\Lamhat_s = \Lamhat_1$ and $\Lambda_s = \Lambda_1$  for $s \le 1$. For any $\mu \in (0,1]$, we have the bound
 \begin{align*}
 \Lamhat_{t-\tau} &= \lambda I + \sum_{s=1}^{t-\tau} \matHhat_s^\top \matHhat_s ~ \succeq \lambda I + \mu \sum_{s=1}^{t-\tau} \matHhat_s^\top \matHhat_s\\
 &\succeq (\lambda - \mu \tau \radH^2) I + \mu\sum_{s=1}^{t} \matHhat_s^\top \matHhat_s \\
 &\succeq (\lambda - \mu \tau \radH^2) I + \sum_{s=1}^{t} \frac{\mu}{2}\matH_s^\top \matH_s  - \mu(\matHhat_s - \matH_s)^\top(\matHhat_s - \matH_s),
 \end{align*}
 where the last step follows from \Cref{fact:vector_norm_lb}. We can crudely bound$(\matHhat_s - \matH_s)^\top(\matHhat_s - \matH_s) \preceq \|\matHhat_s - \matH_s \|^2 I \preceq \radY^2 \epsG^2 I$ via \Cref{fact:epsX}, giving 
 \begin{align*}
  \Lamhat_{t-\tau} \succeq (\lambda - \mu \tau \radH^2 - \mu \radY^2 t \epsG^2) I + \frac{\mu}{2} \sum_{s=1}^{t} \matH_s^\top \matH_s.
 \end{align*} 
 Bounding $t \le T$, and taking $\mu = \min\{1, \frac{\lambda}{2(\tau \radH^2 + \radY^2 \epsG^2 T)}\}$, we obtain
 \begin{align*}
 \Lamhat_{t-\tau} \succeq \frac{\lambda}{2}  + \frac{\mu}{2} \sum_{s=1}^{t} \matH_s^\top \matH_s  \succeq \frac{\mu}{2} \Lambda_t
 \end{align*} 
 Thus, for any upper bound $\ccond \ge \sqrt{\frac{2}{\mu}}$
 \begin{align}
 \Lamhat_{t-\tau}^{-1} \preceq \frac{2}{\mu} \Lambda_t^{-1} \preceq \ccond^2 \Lambda_t^{-1}. \label{eq:ctau_proof_line}
 \end{align}
 Finally, we can bound 

 \begin{align*}
 \sqrt{\frac{2}{\mu}} &= \sqrt{\max\{2, \frac{4(\tau \radH^2 + \radY^2 \epsG^2 T)}{\lambda}\}} \\
 &= \sqrt{\max\{2, 4\radY^2\frac{\tau \radG^2 + \epsG^2 T}{\lambda}\}} \\
 &\overset{(i)}{\ge} \sqrt{\max\{2, 4\clam^{-1}\radY^2(1 + \frac{\tau \radG^2}{\lambda})} \\\
&\le 2(1+\radY) + 2\clam^{\minhalf}\radY \sqrt{ \frac{\tau \radG^2}{\lambda}} := \ccond,
\end{align*}
where we use that $\lambda \ge \clam T \epsG^2$ in $(i)$. This verifies that $\ccond$ in the lemma is an upper bound on $\sqrt{2/\mu}$, and the lemma now follows from \Cref{eq:ctau_proof_line}.
\qed

\subsection{Proof of \Cref{lem:blocking_movement_bound}}
Let $\tau \in \N$ denote our blocking parameter. Again, adopt the convention $\Lamhat_s = \Lamhat_1$ and $\Lambda_s = \Lambda_1$ for $s \le 0$, and let $\ccond$ be such from \Cref{lem:unknown_block}, which ensures that, for all $t$,
\begin{align}
\Lamhat_{t-\tau}^{-1} \preceq   \ccond^2 \Lambda_t^{-1}.  \label{eq:ctau_recall_eq}
\end{align}
Then, any for $s \in \{0,\dots,\tau - 1\}$ such that $s \le t - 1$ any $\mu > 0$, we have
 \begin{align*}
\|\matY_{t  - i}(\matz_{t-s} - \matz_{t - s-1 })\|_2 &\le \|\matY_{t  - i}\Lamhat_{t-s-1}^{-\frac{1}{2}}\|_{\op} \|\Lamhat_{t-s-1}^{\frac{1}{2}}(\matz_{t-s} - \matz_{t - s-1 })\|_2 \\
&\le \|\matY_{t  - i}\Lamhat_{t-\tau-i}^{\frac{1}{2}}\|_{\op} \|\Lamhat_{t-s-1}^{\frac{1}{2}}(\matz_{t-s} - \matz_{t - s-1 })\|_2 \\
&\le \|\matY_{t  - i}\Lamhat_{t-\tau-i}^{-\frac{1}{2}}\|_{\op} \|\Lamhat_{t-s-1}^{\frac{1}{2}}\nabhat_{t-s-1}\|_2 \tag*{(Projection Step)}\\
&\le \sqrt{\trace(\matY_{t  - i}\Lamhat_{t-\tau-i}^{-1}\matY_{t  - i})  \cdot \|\nabhat_{t-s-1}\|_{\Lamhat_{t-s-1}}^2}\\
&\le \ccond\sqrt{\trace(\matY_{t  - i}\Lamhat_{t-i}^{-1}\matY_{t  - i})  \cdot \|\nabhat_{t-s-1}\|_{\Lamhat_{t-s-1}}^2} \tag{\Cref{eq:ctau_recall_eq}}.
 \end{align*}
Note that the above expression does not depend on $\tau$. Thus, since $\matz_{t-s} - \matz_{t-s-1} = 0$ for $s > t-1$ (recall here we assume $\matz_i = \matz_1$ for $i \le 1$), an application of Cauchy Schwartz yields
 \begin{align}
 \sum_{t=1}^T\sum_{s = 0}^{\tau -1} \|\matY_{t  - i}(\matz_{t-s} - \matz_{t - s-1 })\|_2 &\le \tau \ccond\left( \sum_{t=s+1}^T\trace(\matY_{t-i}\Lambda_{t-i}^{-1}\matY_{t-i})\right)^{\frac{1}{2}} \left(\sum_{t=s+1}^T\|\nabhat_{t-s}\|_{\Lamhat_{t-s}}^2\right)^{\frac{1}{2}}\nonumber\\
 &\le \tau \ccond\left( \sum_{t=1}^T\trace(\matY_{t-i}\Lambda_{t-i}^{-1}\matY_{t-i})\right)^{\frac{1}{2}} \left(\sum_{t=1}^T\|\nabhat_{t}\|_{\Lamhat_{t}}^2\right)^{\frac{1}{2}}\nonumber\\
 &\le \tau \ccond\left( \sum_{t=1-h}^T\trace(\matY_{t}\Lambda_{t}^{-1}\matY_{t})\right)^{\frac{1}{2}} \left(\sum_{t=1}^T\|\nabhat_{t}\|_{\Lamhat_{t}}^2\right)^{\frac{1}{2}}\label{eq:Y_tau_move_CS}
 \end{align}
 Arguing as in the proof of \Cref{thm:semions}, and using $\lambda \ge h\radG^2 \ge 1$, 
 \begin{align}
 \sum_{t=1}^T\|\nabhat_t\|_{\Lamhat_{t}}^2 \le \Leff^2 \sum_{t=1}^T\trace\left(\matHhat_{t}\Lamhat_{t}^{-1}\matHhat_{t})\right)^{\frac{1}{2}} \le d\Leff^2 \log(1 + \frac{T\radH^2}{\lambda}) \le d\Leff^2 \cdot \Lfactor\label{eq:nabhat_norm_square}.
 \end{align}
We now develop a simple claim, which is a consequence of \Cref{prop:covariance_lb}:
\begin{claim}\label{lem:Y_move_general} Recall $\cpsi[t]:= \max\{1, \frac{t\psiG(h+1)^2}{h\radG^2}\}$, and set $\munot = \min\{1, \frac{\lambda}{10h \radH^2 \cpsi[T]}\}$.  We have
\begin{align*}
  \sum_{t=1-h}^T \trace(\matY_t^\top \Lambda_t^{-1}\matY_t) \le \frac{2d}{\munot \kappa}\Lfactor.
\end{align*}
\end{claim}
\begin{proof}[Proof of \Cref{lem:Y_move_general}]
From \Cref{prop:covariance_lb}, we have the bound
 \begin{align*}
	\sum_{s=1}^t \matH_s^\top \matH_s \succeq \frac{\kappa}{2} \sum_{s=1-h}^{t} \matY_{s}^{\top}\matY_{s} - 5h \radH^2 \cpsi[t]I.
	\end{align*}
	Thus, for any $ \munot = \min\{1, (10h \radH^2 \cpsi[T])^{-1}\} \le 1$,
 \begin{align*}
 \Lambda_t &= \lambda I + \sum_{s=1}^t \matH_t \matH_t^\top \ge \lambda I + \munot \sum_{s=1}^t \matH_t \matH_t^\top\\
	&= \lambda I + \munot \left(\frac{\kappa}{2} \sum_{s=1-h}^t \matY_s \matY_s^\top - 5h \radH^2 \cpsi[T] \right)~\succeq \frac{\lambda}{2} I + \frac{\munot \kappa}{2} \sum_{s=1-h}^t \matY_s \matY_s^\top.
 \end{align*}
 Hence, from the log-det potential bound of \Cref{lem:log_potential}, the bounds $\munot, \kappa \le 1$ and $\radH = \radG \radY$
 \begin{align*}
  \sum_{s=1-h}^T \trace(\matY_s^\top \Lambda_s^{-1}\matY_s) \le \frac{2d}{\munot \kappa}\log(1 + \frac{\munot \kappa T\radY^2}{\lambda}) \le \frac{2d}{\munot \kappa}\log(1 + \frac{ T\radH^2}{\lambda}) = \frac{2d}{\munot \kappa} \Lfactor.
 \end{align*}
\end{proof}
To apply the above, let us simplify our expression for $\munot$. Recall that 
\begin{align*}
\munot = \min\left\{1, \frac{\lambda}{10h \radH^2 \cpsi[T]}\right\}, \quad \cpsi[T]:= \max\left\{1, \frac{T\psiG(h+1)^2}{h\radG^2}\right\} \le (1 + T\epsG^2/h\radG^2),
\end{align*}
where we note that $\epsG = \|\Ghat  - G\|_{\loneop} \ge \sum_{i > h}\|G^{[i]}\|_{\op} \ge \psiG(h+1)$, since $\Ghat^{[i]} = 0$ for $i > h$.  Using the bounds $\radH/\radG = \radY$ and  $\lambda \ge \clam(T\epsG^2 + h\radG^2)$ for $\clam \in (0,1]$, 
\begin{align*}
\munot^{-1} &\le 1 + \frac{10h \radH^2 \cpsi[T]}{\lambda}\\
&\le1 + \frac{10h \radH^2 (1 + T\epsG^2/h\radG^2)}{\lambda} \\\
&= 1 + \frac{10\radY^2(h \radG^2 + \radY^2 T\epsG^2/h )}{\lambda} \le 1 + \clam^{-1} 10 \radY^2 .
\end{align*}
Together with \Cref{lem:Y_move_general}, we obtain
 \begin{align}
  \sum_{t=1-h}^T\trace(\matY_{t}\Lambda_{t}^{-1}\matY_{t}) \le \frac{2d}{\munot \kappa} \Lfactor \le \frac{2d (1+10\radY^2) }{ \kappa}\cdot \Lfactor \label{eq:Y_move_bound_unknown}.
  \end{align}

  Thus, putting together Equations~\eqref{eq:Y_tau_move_CS},~\eqref{eq:nabhat_norm_square}, and~\eqref{eq:Y_move_bound_unknown}, 
   \begin{align*}
 \sum_{t=1}^T\sum_{s = 0}^{\tau -1} \|\matY_{t  - i}(\matz_{t-s} - \matz_{t - s-1 })\|_2 &\le \tau \ccond\clam^{\minhalf}  \cdot \Leff d \sqrt{\frac{2(1+10\radY^2)}{\kappa }}\Lfactor,
 \end{align*}
which is the first inequality of the lemma. For the second inequality, we establish the following analogue of \Cref{eq:Y_tau_move_CS}:
 \begin{align*}
 \sum_{t=1}^T\sum_{s = 0}^{\tau -1} \|\matH_{t}(\matz_{t-s} - \matz_{t - s-1 })\|_2 &\le \tau \ccond  \cdot \left( \sum_{t=1}^T\trace(\matH_{t}\Lambda_{t}^{-1}\matH_{t})\right)^{\frac{1}{2}} \left(\sum_{t=1}^T\|\nabhat_t\|_{\Lamhat_{t}}^2\right)^{\frac{1}{2}}.
 \end{align*}
Again, we bound $\sum_{t=1}^T\|\nabhat_t\|_{\Lamhat_{t}}^2 \le d\Leff^2 \cdot \Lfactor$ as in \Cref{eq:nabhat_norm_square}. Moreover, from \Cref{eq:X_move_contrib}, we can bound $\sum_{t=1}^T\trace(\matH_{t}\Lambda_{t}^{-1}\matH_{t}) \le d\Lfactor$. Thus,
 \begin{align*}
 \sum_{t=1}^T\sum_{s = 0}^{\tau -1} \|\matH_{t}(\matz_{t-s} - \matz_{t - s-1 })\|_2 &\le \tau d \Leff \clam^{\minhalf} \ccond \Lfactor,
 \end{align*}
 which is precisely the second inequality of the lemma.

\qed

\subsection{Proof of \Cref{lem:cancelling_neg_reg}}
We state a slighlty sharper variant of \Cref{prop:covariance_lb}, which considers directions limited to $\matdel \in\calC - \calC$:
\begin{claim}\label{prop:cov_lb_C_directions}  Set $\cpsi[t]:= \max\{1, \frac{t\psiG(h+1)^2}{h\radG^2}\}$. let $\matdel = z - z'$ for some $z,z' \in \calC$. Then, 
\begin{align*}
\matdel^\top \left(\sum_{s=1}^T \matH_t\matH_t\right)\matdel \ge \frac{\kappa}{2}\matdel^\top\left(\sum_{s=1-h}^T \matH_t\matH_t\right)\matdel - 20 h \radyc^2\radG^2 \cpsi[t].
\end{align*}
\end{claim}
\begin{proof} The proof is analogous to \Cref{prop:covariance_lb}, but instead, the remainder term need only account for directiong $z - z'$ for $z,z' \in \calC$. This replaces the factor of $\radY $ one would obtain with a factor of $ \max_{t,t'} \|\matY_t \matdel_{t'}\| \le 2\radyc$, yielding a remainder temr of $20 h \radyc^2\radG^2 \cpsi[t]$ instead of $5 h \rady^2\radG^2 \cpsi[t]$ in the original proposition.
\end{proof}
Let us now turn to the proof of our lemma. From \Cref{prop:cov_lb_C_directions}, we have
\begin{align*}
\sum_{s=1}^{\tau}\|\matH_{k_j} \matdel_{k_j+1}\|_2^2 &= \matdel_{k_j+1}^\top\left(\sum_{s=1}^{\tau}\matH_{k_j+s}^\top \matH_{k_j+s}\right) \matdel_{k_j+1}\\
&\ge \frac{\kappa}{2}\matdel_{k_j+1}^\top\left(\sum_{s=1-h}^{\tau}\matY_{k_j+s}^\top \matY_{k_j+s}\right) \matdel_{k_j+1} - 20h \cpsi[\tau]\radG^2 \radyc^2
\end{align*}
Moreover, for any $i \in [h]$, we have
\begin{align*}
 \sum_{s=1}^{\tau}\sum_{i=0}^h\|\matY_{k_j + s  - i}\matdel_{k_j+1}\|_2^2 &= \matdel_{k_j+1}^\top\left(\sum_{s=1}^\tau \matY_{k_j + s  - i}^\top \matY_{k_j + s - i}\right)\matdel_{k_j + 1} \\
&\le  \matdel_{k_j+1}^\top\left(\sum_{s=1-h}^\tau \matY_{k_j + s}^\top \matY_{k_j + s}\right)\matdel_{k_j+1}.
\end{align*}
 Thus, for $\nu \le \frac{\kappa}{4}$, we have
\begin{align*}
\sum_{s=1}^{\tau}\sum_{i=0}^h\nu(h^{-1} + \I_{i = 0})\|\matY_{k_j + s  - i}\matdel_{k_j}\|_2^2  &\le 2 \nu \matdel_{k_j}^\top\left(\sum_{s=1-h}^\tau \matY_{k_j + s}^\top \matY_{k_j + s}\right)\matdel_{k_j} \\
&\le \frac{\kappa}{2} \matdel_{k_j}^\top\left(\sum_{s=1-h}^\tau \matY_{k_j + s}^\top \matY_{k_j + s}\right)\matdel_{k_j} \\
&\le \sum_{s=1}^{\tau}\|\matH_{k_j} \matdel_{k_j}\|_2^2 + 20h \cpsi[\tau]\radG^2 \radyc^2.
\end{align*}
Hence, rearranging, we have
\begin{align*}
\Rcancel &:= \sum_{j=1}^{\jmax}\sum_{s=1}^{\tau} \left(\sum_{i=0}^h\left(\nu(1 + h\I_{i = 0})\|\matY_{k_j + s  - i}\matdel_{k_j}\|_2^2\right)  - \|\matH_{k_j+s}\matdel_{k_j}\|_2^2 \right) \\
&\quad\le \jmax 20h \cpsi[\tau]\radG^2 \radyc^2\\
&\quad\le \frac{T}{\tau} 20h \cpsi[\tau]\radG^2 \radyc^2.
\end{align*}
Finally, let us simplify the dependence on $\cpsi[\tau]$.  We have 
\begin{align*}
\frac{\cpsi[\tau]}{\tau} = \max\{\tau^{-1}, \frac{\psiG(h+1)^2}{h\radG^2}\} \le \frac{\cpsi[\tau]}{\tau} = \max\{\tau^{-1}, \frac{\epsG^2}{h\radG^2}\} \le \frac{1}{\tau} + \frac{\epsG^2}{h\radG^2}.
\end{align*}
Together with $\nu \le \frac{\kappa}{4}$, this gives
\begin{align*}
\Rcancel \le  \frac{20\nu h }{\tau} T \cpsi[\tau]\radG^2 \radyc^2 &\le \frac{20\nu h }{\tau} T \radG^2 \radyc^2 + 20 \nu T \epsG^2 \radyc^2\\
 &\le \frac{20T }{\tau} \cdot \nu h \radG^2 \radyc^2 + 5 T \epsG^2  \cdot \kappa  \radyc^2.
\end{align*}
\qed

\subsection{Proof of \Cref{lem:Regbar_Bound}}
	 From \Cref{eq:main_blocking_eq}, we bound
\begin{align*}
\UnaRegPlus_T\left(\frac{\nu}{4\eta};\zst\right) \le \frac{1}{4\eta} \Regblock +  T\epsG^2 \,\Err(\nu)  + \Reghat_T,
\end{align*}
where  from \Cref{eq:regplus_block} we have
\begin{align*}
 \Regblock \le 8\tau \cdot \nu \radyc + 8\nu\radyc\left(\max_{i \in [h]}  \Regmovyi\right) + 4\radyc  \Rpinot \cdot \Regmovh + \Rcancel.
 \end{align*}
 Let us develop the above bound on $\Regblock $. From \Cref{lem:blocking_movement_bound}, we have \begin{align*}
   \Regmovyi \le \tau  \ccond \clam^{\minhalf} \cdot d \Leff  \sqrt{\frac{2(1+10\radY^2)}{\kappa }} \Lfactor ,\quad \text{and} \quad \Regmovh  \le \tau   \ccond\clam^{\minhalf} \cdot  d\Leff \Lfactor,
   \end{align*} 
   and from \Cref{lem:cancelling_neg_reg}, we have $ \Rcancel \le \frac{20T }{\tau} \cdot \nu h \radG^2 \radyc^2 + 5 T \epsG^2  \cdot \kappa  \radyc^2.$. Thus, using followed by 
 \begin{align*}
 \Regblock &\le 8\tau \cdot \nu \radyc + 8\nu\radyc\left(\max_{i \in [h]}  \Regmovyi\right) + 4\radyc  \Rpinot \cdot \Regmovh + \Rcancel\\
 &\overset{(i)}{\le} 8\tau \ccond \clam^{\minhalf}\radyc \left(\nu + \nu  d\Leff  \sqrt{\frac{2(1+10\radY^2)}{\kappa }} \Lfactor + d\radG\Leff \Lfactor\right) + \frac{20T }{\tau} \cdot \nu h \radG^2 \radyc^2 + 5 T \epsG^2  \cdot \kappa  \radyc^2\\
 &\overset{(ii)}{\le} 8\tau \ccond \clam^{\minhalf} \radyc d\Leff \Lfactor \left( \nu   \sqrt{\frac{2(1+10\radY^2)}{\kappa }}  + 2\radG\right) + \frac{20T }{\tau} \cdot \nu h \radG^2 \radyc^2 + 5 T \epsG^2  \cdot \kappa  \radyc^2\\
 &\lesssim \tau \ccond \clam^{\minhalf} \radyc d\Leff \Lfactor \left( \nu   \sqrt{\frac{2(1+\radY^2)}{\kappa }}  + \radG\right) + \frac{T }{\tau} \cdot \nu h \radG^2 \radyc^2 +  T \epsG^2  \cdot \kappa  \radyc^2,
 \end{align*}
 where $(i)$ uses the above bounds together  with $\ccond \clam^{\minhalf}\ge 1$ (see \Cref{lem:ctau}) , and $(ii)$ uses $\nu \le 1 \le \Leff$ and $d\radG \Lfactor \ge 1$, and where the last line disposes of constants. Using $\radG \ge 1$, and  the assumption $\nu \le \frac{\sqrt{\kappa}}{4(1+\radY)}$, the above is at most
 \begin{align*}
  \Regblock \lesssim \tau \clam^{\minhalf}\ccond \radyc \radG d\Leff \Lfactor  + \frac{T }{\tau} \cdot \nu h \radG^2 \radyc^2 +  T \epsG^2  \cdot \kappa  \radyc^2,
 \end{align*}
Next, using $\lambda \ge \clam\tau$, we have from \Cref{lem:ctau},
\begin{align*}
\ccond &= 2(1+\radY) + 2\radY\sqrt{ \frac{\tau \radG^2}{\lambda}} \lesssim \clam^{\minhalf}(1 + \radY)\radG.
\end{align*}
Thus, we obtain
\begin{align*}
 \Regblock \lesssim \clam^{\minone}\tau (1 + \radY) \cdot \radyc  \radG^2 \cdot d\Leff \Lfactor + \frac{T }{\tau} \cdot \nu h \radG^2 \radyc^2 +  T \epsG^2  \cdot \kappa  \radyc^2,
 \end{align*}
 Combining with $\eta = \frac{3}{\alpha}$, 
 we have
\begin{align*}
&\UnaRegPlus_T\left(\frac{\nu}{4\eta};\zst\right) \\
&\quad\le \frac{1}{4\eta} \Regblock +  T\epsG^2 \,\Err(\nu)  + \Reghat_T\\
&\quad\lesssim \clam^{\minone} \tau \left(\alpha  (1 + \radY) \radyc  \radG^2  \cdot d\Leff \Lfactor \right) +\frac{T }{\tau} \left(\alpha \nu h \radG^2 \radyc^2\right) +  T \epsG^2  \left(\alpha \kappa  \radyc^2 + \Err(\nu)\right) + \Reghat_T.
\end{align*}
Finally, let us substitute in
\begin{align*}
\Err(\nu) &:= \frac{\eta (h+1) \Leff^2 }{\nu}  + \eta\beta^2(\cv + 2 \radyc)^2 + \frac{4\radyc}{\eta}\\
&\lesssim \frac{ h \Leff^2 }{\alpha\nu}  + \frac{1}{\alpha}\beta^2(\cv^2 +  \radyc^2) + \alpha\radyc.
\end{align*}
Since $\alpha \le \beta$ by necessitiy and $\kappa \le 1$, we have $\alpha \le \frac{\beta^2}{\alpha}$, 
so that
\begin{align*}
\Err(\nu)  + \alpha \kappa  \radyc^2  \lesssim \frac{ h \Leff^2 }{\alpha\nu}  + \frac{\beta^2}{\alpha}(\cv^2 +  \radyc + \radyc^2 ) 
\end{align*}
Altogether, combined with the bound $\clam \le 1$, this yields
\begin{align*}
\clam\UnaRegPlus_T(\frac{\nu}{4\eta};\zst) &\lesssim \frac{T \epsG^2}{\alpha}  \left(\frac{ h \Leff^2 }{\nu}  + \beta^2 (\cv^2 +  \radyc + \radyc^2 ) \right) + \Reghat_T.\\
&\quad+ \frac{T \nu}{\tau} \left(\alpha  h \radG^2 \radyc^2\right)  + \tau \cdot \left(\alpha  (1 + \radY) \radyc  \radG^2  \cdot d\Leff \Lfactor \right),
\end{align*}
as needed.

	\qed.

\subsection{Proof of \Cref{lem:unknown_move}}
Consider $\MoveDiff_T  := \sum_{t=1}^T F_t(\matz_{t:t-h}) - f_t(\matz_t)$. The decomposition \Cref{lem:Ft_diff} holds verbatim, and by appropriately modifying \Cref{lem:Y_bound} to use the fact that the iterates are based on $\matHhat_t$, $\Lamhat_t$, we arive at. 
\begin{align*}
	\MoveDiff_T  \le \eta h^2\Leff^2 \radG\cdot \sqrt{\sum_{t=1-h}^{T}\trace(\matY_t \Lamhat_t^{-1}\matY_t) }\cdot\sqrt{\sum_{t=1}^{T}\trace(\matHhat_t^\top \Lamhat_t^{-1}\matHhat_t)}.
	\end{align*}
As  in \Cref{eq:X_move_contrib}, we  bound
\begin{align*}
\sum_{t=1}^{T}\trace(\matHhat_t^\top \Lamhat_t^{-1}\matHhat_t) \le d\log(1 + \frac{T\radH^2}{\lambda}) \le d\, \Lfactor,
\end{align*}
where we take $\lambda \ge 1$ and use $\Lfactor = \log(1 + T \radH^2/\lambda)$  from \Cref{eq:Lfactor}.
Moreover, applying \Cref{lem:ctau} with $\tau = 0$, we have that $\Lamhat_t^{-1} \preceq 4(1+\radY)^2 \Lambda_t^{-1}$, giving 
\begin{align*}
\sum_{t=1-h}^{T}\trace(\matY_t \Lamhat_t^{-1}\matY_t) &\le 4 (1+\radY)^2 \sum_{t=1-h}^{T}\trace(\matY_t \Lambda_t^{-1}\matY_t) ~\le 4(1+\radY)^2 \frac{2d (1+10\radY^2) }{ \kappa}\cdot \Lfactor 
\end{align*}
where the last inequality uses \Cref{eq:Y_move_bound_unknown}. Thus,
\begin{align*}
\MoveDiff_T &\le 9\eta (1+\radY) h^2 d\Leff^2 \Lfactor\radG\cdot \sqrt{(1+\radY^2)/\kappa}\\
&\le 9\eta\kappa^{\minhalf}(1+\radY)^2 \radG   h^2 \cdot d\Leff^2\Lfactor
\end{align*}
\qed

\newcommand{\Rademacher}{\mathrm{Rademacher}}
\newcommand{\calR}{\mathcal{R}}

\section{Lower and Upper Bounds on Euclidean Movement \label{sec:lower_bounds}}
\subsection{Proof of \Cref{thm:Regmu_lb} \label{sec:proof:regmu_lb}}
Our construction is loosely based of of \cite[Theorem 13]{altschuler2018online}.

Recall the lower bound set up $\calC = [-1,1]$, $f_t(z) = (\matv_t - \epsilon z)^2$, and $\epsilon \le 1$. Let $E$ be an epoch length to be selected, and suppose for simplicity that $k = T/E $ is an integer. Let $T_i := 1 + E\cdot(i-1)$ denote the start of each epoch for $i \ge 1$. Let us define the distribution $\calD$ over $\matv_1,\dots,\matv_T$ via:
\begin{align*}
\matv_{t} &:= \begin{cases} \iidsim \mathrm{Unif}(\{-1,1\})& t = T_i\\
  \matv_{T_i} & t \in \{T_i + 1,\dots,T_{i+1}-1\}
\end{cases}
\end{align*}
Lastly, recall the definition:
\begin{align*}
\Regmu := \sum_{t=1}^T f_t(\matz_{t})   - \inf_{z \in \calC} \sum_{t=1}^T f_t(z) + \mu \sum_{t=1}^T |\matz_{t-1} - \matz_t|
\end{align*}
Our key technical ingredient is the following lemma, which shows that if the regularizer is large enough, the optimal strategy is essentially to select $\matz_t = \matz_{T_i}$ within any given epoch $i$: 
\begin{lemma}\label{lem:sub_sum}  For $\mu \ge 4E\epsilon$, 
\begin{align*}
\sum_{t=T_{i} + 1}^{T_{i+1} - 1} f_t(\matz_t) + \mu |\matz_{t} - \matz_{t- 1}| \ge (E-1)f_t(\matv_{T_i} - \matz_{T_i}).
\end{align*}
\end{lemma}
\begin{proof} 
We can write
\begin{align*}
\sum_{t=T_{i} + 1}^{T_{i+1} - 1} f_t(\matz_t) + \mu |\matz_{t} - \matz_{t- 1}| &= \sum_{t=T_{i} + 1}^{T_{i+1} - 1} f_{T_i}(\matz_{t}) + \mu|\matz_{t} - \matz_{t - 1}|\\
&\ge \sum_{s=1}^{E-1} f_{T_i}(\matz_{t}) + \mu \cdot \max_{t=T_{i} + 1,\dots,T_{i+1}-1}|\matz_{T_i} - \matz_{t}|\\
&\ge \sum_{s=1}^{E-1} \underbrace{f_{T_i}(\matz_{t}) + \frac{\mu}{E - 1}|\matz_{T_i} - \matz_{t}|}_{:= g(\matz_t)},
\end{align*}
where the first inequality uses the triangle inequality, and the second replaces the maximum by the average. Define $\mu_0 = \frac{\mu}{2(E - 1)\epsilon}$, and set $g(z) := f_{T_i}(\matz_{t}) + \frac{\mu}{E - 1}|\matz_{T_i} - \matz_{t}| = (\matv_{T_i} - \epsilon z)^2 + 2\epsilon\mu_0 |\matz_{T_i} - \matz_{t}|$. Then, 
\begin{align*}
\partial g(z) &=  2 \epsilon\left( \epsilon z - \matv_{T_i} + \mu_0 \sigma(z)\right)
\end{align*}
where $\sigma(z) = 1$ if $\matz_{T_i} > z$, $-1$ if $\matz_{T_i} < z$, and is in interval $[-1,1]$ if $z = \matz_{T_i}$. Now, if $\mu_0 \ge 2$, then, $|\epsilon z - \matv_{T_i}| \le \mu_0$, so that the first order optimality conditions are met by selecting $z^\star = \matz_{T_i}$. This yields
\begin{align*}
g(z^\star) = (\matv_{T_i} - \epsilon \matz_{T_i})^2.
\end{align*}
The bound follows.
\end{proof}
By summing within different epochs, the above lemma implies a simple lower bound on $\Regmu$:
\begin{align*}
\Regmu &= \sum_{i=1}^k \sum_{t=T_i}^{T_{i-1}+1} f_t(\matz_t) - f_t(z) + \mu\|\matz_{t}-\matz_{t-1}\|\\
&\overset{(i)}{=} \sum_{i=1}^k f_{T_i}(\matz_{T_i}) - Ef_{T_i}(z) + \mu\|\matz_{T_i}-\matz_{T_i-1}\| + \left(\sum_{t=T_i}^{T_{i-1}+1} f_t(\matz_t)+\mu\|\matz_{t}-\matz_{t-1}\|\right)\\
&\overset{(ii)}{\ge} \sum_{i=1}^k f_{T_i}(\matz_{T_i}) - Ef_{T_i}(z) + \mu\|\matz_{T_i}-\matz_{T_i-1}\| + (E-1)f_{T_i}(\matz_{T_i})\\
&\ge \sum_{i=1}^k f_{T_i}(\matz_{T_i}) - Ef_{T_i}(z) + (E-1)f_{T_i}(\matz_{T_i})\\
&= \sup_{z \in \calC} E\left(\sum_{i=1}^Kf_{T_i}(\matz_{T_i})  - f_{T_i}(z)\right),
\end{align*}
where $(i)$ uses that $f_t = f_{T_i}$ in epoch $i$ and $(ii)$ uses \Cref{lem:sub_sum}.
Crucially, the above quantity is scaled up by a factor of $E$, and the learner is forced to commit to a single iterate per epoch.  Continuing with $f_{T_i}(z) = (\matv_{T_i} - \epsilon z)^2$, 
\begin{align*}
\Regmu &\ge \sup_{z \in \calC} E\left(\sum_{i=1}^k (\matv_{T_i} - \epsilon \matz_{T_i})^2 - (\matv_{T_i} - \epsilon z)^2\right)\\
&= \sup_{z \in \calC} E\left(\sum_{i=1}^k  - 2\epsilon \matv_{T_i}\matz_{T_i} + \underbrace{\epsilon^2\matz_{T_i}^2}_{\ge 0} + 2 \epsilon z \matv_{T_i} -  \epsilon^2 \cdot \underbrace{z^2}_{\le 1}  \right)\\
&\ge \sup_{z \in \calC} E\left(\left(\sum_{i=1}^k  - 2\epsilon \matv_{T_i}\matz_{T_i}  + 2 \epsilon z \matv_{T_i}  \right) - k \epsilon^2\right).
\end{align*}
Taking an expectation, and noting that $\Exp[\matv_{T_i}\matz_{T_i}] = 0$ by construction, we have that
\begin{align*}
\Exp[\Regmu] &\ge E \left(2 \epsilon \Exp \left[ \sup_{z \in \calC} z \sum_{i=1}^k \matv_i \right]  - k \epsilon^2 \right) \\
&= 2\epsilon E \left( \Exp \left|\sum_{i=1}^k \matv_i \right| - \frac{k}{2} \epsilon\right) \\
&\ge 2\epsilon E \left(c\sqrt{k} - \frac{k\epsilon}{2}\right),
\end{align*}
where $c \le 1$ is a universal constant. \footnote{Note the folklore results that the expectation average of $k$ Rademacher random variables scales as $\sqrt{k}$} 
Let us now tune the above bound. Select 
\begin{itemize}
\item $k =  \floor{(8Tc/ \mu)^{2/3}}$ 
\item $\epsilon = \mu/4E$. 
\end{itemize}
We first check that these parameters are valid:
\begin{claim} For a universal constant $c_1$, it holds that if $\mu \le c_1 T$, then $k \ge 1$ and $\epsilon \le 1$.
\end{claim}
\begin{proof}
For $\mu \le 8c T$, $k \ge 1$. Moreover,
\begin{align*}
\epsilon = \frac{\mu}{4E} = \frac{\mu k}{4T} \le (8Tc/ \mu)^{2/3} \frac{\mu}{T} = 4 c^{2/3} (\mu/T).
\end{align*}
Hence, for $\mu \le T/4c^{2/3}$, the above is at most $1$. Setting $c_1 = \min\{8c,1/4c^{2/3}\}$ concludes.
\end{proof}
For the above choices, we have
\begin{align*}
\Exp[\Regmu] &\ge 2\epsilon E \left(c\sqrt{k} - \frac{k\epsilon}{2}\right)\\
&= \frac{\mu}{2}\left(c\sqrt{k} - \frac{k^2 \mu}{8T } \right) \quad = \frac{c\sqrt{k} \mu}{4}\left(2 - \frac{k^{3/2}}{8T c /\mu} \right)\\
&\ge \frac{c\sqrt{k} \mu}{4} \quad \ge \frac{c \mu \floor{(8Tc/ \mu)^{2/3}}^{1/2}}{4}\\
&\ge c_2 \mu(T/\mu)^{1/3} = c_2 (\mu^2 T)^{1/3},
\end{align*}
for some universal constant $c_2$. Moreover, suppose that that $\Exp[\UnaReg_T] \le R$. Then, for $\mu \ge c_1 T$
\begin{align*}
 c_2 (\mu^2 T)^{1/3} \le \Exp[\Regmu] \le R + \mu \Exp[\EucMoveCost_T].
\end{align*}
Rearranging, we have that if $ c_2 (\mu^2 T)^{1/3} \ge 2R$, $\Exp[\EucMoveCost_T] \ge  \frac{c_2}{2} (T/\mu)^{1/3} $. For this to hold, we take $\mu = \sqrt{(2R/c_2)^3/T}$, yielding 
\begin{align}
\Exp[\EucMoveCost_T] \ge  \frac{c_2}{2} (T\cdot(T/(2R/c_2)^3)^{1/3}  = \frac{c_2}{2}(c_2 T/2R)^{1/2} \ge c_3 \sqrt{T/R}.
\end{align}
Finally, we need to ensure that $\mu \le c_1 T$, which hold for $(2R/c_2)^3/T \le c_1^2 T^2$, i.e. for $R \le c_4 T$ for a universal $c_4$. 

\subsection{Matching Tradeoff via \ons \label{sec:ons_tradeoff}}
We now show that \ons{} mathces the tradeoff in \Cref{thm:Regmu_lb} up to logarithmic factors, problem constants and dimension.  To show this, we first check that \ocoam{} losses satisfy the general \ons{} regularity conditions. We say $f$ is $\tau$-exp concave if $\nablatwo f \succeq \tau \cdot \nabla f (\nabla f)^\top$\citep{hazan2019introduction}.  The following is a direct consequence of \Cref{lem:ft_facts}
\begin{lemma}\label{lem:ocoam_exp_conc} Let $f_t$ be an \ocoam{} loss with parameters bounded as in \Cref{defn:pol_reg_pars}, where $\ell$ satisfies \Cref{asm:loss_reg}. Then $f_t$ is $\frac{\alpha}{\Leff^2}$-exp concave, and $\radH \Leff$-Lipschitz on $\calC$.
\end{lemma}
We now show that \ons{} matches the optimal $(\mu^2 T)^{1/3}$ scaling up to dimension and logarithmic factors:
\begin{theorem}\label{thm:ons_move} Consider \ons{} on a sequence family of $G$-Lipschitz, $\tau$-exp concave functions on a convex set $\calC$ of diameter $D$. Let define $R_0 = (GD + \tau^{-1})\cdot d\log T$ be the standard upper bound (up-to-constants) on the regret of \ons{} \citep{hazan2019introduction}. Then, for any $\mu \in \R$, there exists a choice of regularization parameter $\lambda$ such that \ons{} with $\eta = 2\max\{4GD,1/\tau\}$ has:
\begin{align*}
\Regmu &\lesssim (R_0 D^2 \cdot T\mu^2)^{1/3} + R_0.
\end{align*}
\end{theorem}
For the special case of \ocoam{}, the above guarantee can also be satisfied for by \semions (albeit with modififed dependence on problem parameters ). 

Consider the \ons{} algorithm, with updates
\begin{align}
\matztil_{t+1} = \matz_t - \eta \Lambda_t^{-1}\nabla_t, \quad \matz_{t+1} = \argmin_{z \in \calC}\|\matztil_{t+1} - z\|_{\Lambda_t}^2, \quad \Lambda_t := \lambda I + \sum_{s=1}^{t-1}\nabla_t \nabla_t^\top, \quad \nabla_t := \nabla f_t(\matz_t) \label{eq:ons_updates}
\end{align}
Set $\eta = 2\max\{4GD,1/\tau\}$, $\lambda \ge G^2$. From \citet[Section 4.3]{hazan2019introduction}, with the notation change $\eta \gets 1/\gamma $, $\tau \gets \alpha$ $\Lambda_t \gets A_t$, and $\lambda \gets \epsilon$, \ons{} has unary regret bouned by
\begin{align*}
\UnaReg_T &\le \frac{\eta}{2}\sum_{t=1}^T\nabla_t \Lambda_t^{-1}\nabla_t + \frac{D^2 \lambda}{2\eta} \le\frac{ d\eta}{2}\log( 1 + T) + \frac{D^2 \lambda}{2\eta}.
\end{align*}
Moreover, we can bound 
\begin{align*}
\EucMoveCost_T = \sum_{t=1}^T \|\matz_t - \matz_{t-1}\| &\overset{(i)}{\le} \frac{1}{\sqrt{\lambda}}\sum_{t=1}^T \|\Lambda^{1/2}(\matz_t - \matz_{t-1})\| \\
&\overset{(ii)}{\le} \frac{1}{\sqrt{\lambda}}\sum_{t=1}^T \|\Lambda^{1/2}(\matztil_t - \matz_t)\| \quad = \frac{1}{\sqrt{\lambda}}\sum_{t=1}^T \|\Lambda^{-1/2}\nabla_t\| \\
&\overset{(iii)}{\le} \frac{\eta}{\sqrt{\lambda}}\sqrt{T \sum_{t=1}^T \nabla_t^\top \Lambda_t^{-1} \nabla_t }\quad\overset{(iv)}{\le} \frac{\eta}{\sqrt{\lambda}}\sqrt{T d \log(1 + T ) },
\end{align*}
where $(i)$ uses $\Lambda_t \succeq \lambda$, $(ii)$ uses the Pythagorean theorem, $(iii)$ uses Cauchy-Schwartz, and $(iv)$ applies the log-determinant lemma as in \citet[Section 4.3]{hazan2019introduction} with $\lambda \ge G^2$.  Hence, 
\begin{align*}
\Regmu &\le \frac{ d\eta}{2}\log( 1 + T) + \frac{D^2 \lambda}{2\eta} +  \frac{\eta \mu}{\sqrt{\lambda}}\sqrt{T d \log(1 + T) }.
\end{align*}
Set $\lambda_0$ to satisfy $\frac{D^2 \lambda_0}{2\eta} = \frac{\eta \mu}{\sqrt{\lambda_0}}\sqrt{T d \log(1 + T) }$. Then,
\begin{align*}
\frac{D^2 \lambda_0}{2\eta} + \frac{\mu}{\sqrt{\lambda_0}}\sqrt{T d \log(1 + T) } &= \frac{D^2 \lambda_0}{\eta} \\
&= \frac{D^2}{\eta} \cdot \left(\frac{2\eta^2}{D^2}\mu \sqrt{Td \log(1+T)} \right)^{2/3}\\
&= \frac{D^2}{\eta} \cdot \left(\frac{2\eta^4}{D^4}\mu^2 Td \log(1+T) \right)^{1/3}\\
&=  \left(2D^2 \cdot \mu^2T  \cdot \eta d \log(1+T) \right)^{1/3}.
\end{align*}
Setting $\lambda = G^2\vee \lambda_0$ yields
\begin{align*}
\Regmu &\le \frac{ d\eta}{2}\log( 1 + T) + \frac{G^2  D^2 }{2\eta} +  \frac{ d\eta}{2}\log( 1 + T) + \frac{D^2 \lambda_0}{2\eta} +  \frac{\mu}{\sqrt{\lambda_0}}\sqrt{T d \log(1 + T) }\\
&\le \frac{ d\eta}{2}\log( 1 + T) + \frac{G^2  D^2 }{2\eta} + \left(2D^2 \cdot \mu^2T  \cdot \eta d \log(1+T) \right)^{1/3}.
\end{align*}
Subsititing in $\eta = 2\max\{4GD,1/\tau\}$, and defining $R_0 = \max\{GD,1/\tau\} \cdot d \log(1+T)$ gives that the above is at most
\begin{align*}
\Regmu &\lesssim (R_0 D^2 \cdot T\mu^2)^{1/3} + R_0.
\end{align*}

\end{document}